\documentclass[sigconf]{acmart}

\usepackage{hyperref}
\usepackage{cleveref}

\usepackage{microtype}
\usepackage{graphicx}
\usepackage{enumitem}
\usepackage{subfigure}
\usepackage{booktabs}

\usepackage{thm-restate}

\usepackage[algo2e,ruled,vlined,lined,linesnumbered]{algorithm2e}

\usepackage{amsmath}

\usepackage{amssymb}
\usepackage{mathtools}
\usepackage{amsthm}
\usepackage{bbm}

\theoremstyle{plain}
\newtheorem{theorem}{Theorem}[section]

\newtheorem{corollary}{Corollary}[theorem]
\theoremstyle{definition}

\newtheorem{assumption}{Assumption}

\newtheorem{remark}{Remark}[theorem]

\definecolor{babyblueeyes}{rgb}{0.63, 0.79, 0.95}
\definecolor{celadon}{rgb}{0.67, 0.88, 0.69}

\usepackage{tcolorbox}

\usepackage{hyperref}
\usepackage{diagbox}
\usepackage{makecell}

\usepackage[textsize=tiny]{todonotes}

\AtBeginDocument{%
  \providecommand\BibTeX{{%
    \normalfont B\kern-0.5em{\scshape i\kern-0.25em b}\kern-0.8em\TeX}}}

\setcopyright{acmcopyright}
\copyrightyear{2024}
\acmYear{2024}
\acmDOI{XXXXXXX.XXXXXXX}

\acmPrice{15.00}
\acmISBN{978-1-4503-XXXX-X/18/06}

\begin{document}

\title{Client-Centric Federated Adaptive Optimization}

\author{Jianhui Sun}
\affiliation{
\institution{Department of Computer Scienc}
  \institution{University of Virginia}
  \city{Charlottesville}
  \state{Virginia}
  \country{USA}
}
\email{js9gu@virginia.edu}

\author{Xidong Wu}
\affiliation{
\institution{Department of Electrical and Computer Engineering}
  \institution{University of Pittsburgh}
  \city{Pittsburgh}
  \state{Pennsylvania}
  \country{USA}
}
\email{xidong_wu@outlook.com}

\author{Heng Huang}
\affiliation{
\institution{Department of Computer Science}
  \institution{University of Maryland}
  \city{College Park}
  \state{Maryland}
  \country{USA}
}
\email{henghuanghh@gmail.com}

\author{Aidong Zhang}
\affiliation{
\institution{Department of Computer Scienc}
  \institution{University of Virginia}
  \city{Charlottesville}
  \state{Virginia}
  \country{USA}
}
\email{aidong@virginia.edu}

\begin{abstract}
Federated Learning (FL) is a distributed learning paradigm where clients collaboratively train a model while keeping their own data private. With an increasing scale of clients and models, FL encounters two key challenges, client drift due to high degree of statistical/system heterogeneity, and lack of adaptivity. However, most existing FL research is based on unrealistic assumptions that virtually ignore system heterogeneity. In this paper, we propose Client-Centric Federated Adaptive Optimization, which is a class of novel federated adaptive optimization approaches. We enable several features in this framework such as arbitrary client participation, asynchronous server aggregation, and heterogeneous local computing, which are ubiquitous in real-world FL systems but are missed in most existing works. We provide a rigorous convergence analysis of our proposed framework for general nonconvex objectives, which is shown to converge with the best known rate. Extensive experiments show that our approaches consistently outperform the baseline by a large margin across benchmarks.
\end{abstract}


\maketitle
\section{Introduction}
\label{sec:intro}

Federated learning (FL) is a distributed learning setting where many clients collaboratively train a machine learning model under the coordination of a central server, while keeping the training data decentralized and private \cite{li2021survey}. In cross-device FL setting, clients are usually an enormous number of edge devices, which own data that are critical to train a large-scale model on server but are not allowed to share the private data due to privacy concerns or regulatory requirements \cite{european_commission_regulation_2016,California_Consumer_Privacy_Act_CCPA}. 

Federated Averaging (FedAvg) \cite{McMahan2017FedAvg} solves this problem by taking a ``computation then aggregation'' approach, i.e., having each client train locally with its own data, and the server aggregates the local models every once in a while. FedAvg and its variants \cite{Wang20FedNova,karimireddy2020scaffold,Li20FedProx,reddi2020adaptive} enjoy communication efficiency as well as an appealing ``linear speedup'', i.e., the convergence accelerates with increasing number of clients and local steps, even with only a small fraction of clients participating in each round \cite{li2020federated,yang2021achieving}, and have thus become the most popular FL algorithms.

In spite of the empirical success, FedAvg and its variants face the following two key challenges that may severely destabilize the training and deteriorate its performances,

\begin{itemize}[leftmargin=*]
    \item \textbf{High degrees of heterogeneity}. The \textit{statistical heterogeneity} (i.e., the local data distributions of clients are non \textit{i.i.d.}), and \textit{system heterogeneity} (i.e., the completely different levels of system characteristics such as battery level, computational/memory capacity, network connection), are both ubiquitous in large-scale edge networks. Such high degrees of heterogeneity may result in \textit{client drift}, where local client models move away from globally optimal models \cite{karimireddy2020scaffold,reddi2020adaptive,liu2024robust}. FedAvg converges slowly or even diverges in presence of client drift.
    \item \textbf{Lack of adaptivity}. Despite its simplicity, there is a disadvantage that FedAvg scales its gradient uniformly and applies an identical learning rate to all features throughout the entire training process, which is similar to its non-federated counterpart Stochastic Gradient Descent (SGD). Such isotropic training leads to inferior performances when the feature space is sparse, or there exists a few dominant feature directions, which is quite common in training large-scale overparameterized deep models such as \cite{He2016DeepResNet,dosovitskiy2021VIT,devlin2018bert,brown2020GPT}.
\end{itemize}

Though various attempts have been made to tackle part of the above challenges, most of them belong to a regime which we refer to as \textbf{Server-Centric FL}. In server-centric FL, the following unrealistic assumptions on server-client coordination have been implicitly made: (1) each client is available upon the server's request, and the server determines the sampling scheme of clients; and (2) all clients conduct a highly synchronized and homogeneous local computing, i.e., all clients synchronize with server at each round and then run an identical number of local epochs. Server-centric FL is an idealized system in which the server orchestrates clients who have minimal independence. 

In spite of the wide adoption of server-centric FL in existing literature \cite{McMahan2017FedAvg,stich2018localSGD,li2020federated,yang2021achieving,Hsu2019MeasuringTE,Wang2020SlowMo,Li20FedProx,bao2022doubly,karimireddy2020scaffold,Wang20FedNova,reddi2020adaptive,wang22adaptive,zhang2020fedpd,acar2021feddyn,huang2024overcoming,wu2024solving,zhou2022federated,wu23fcso,sun2024role,che2023multimodal, mei2024using,wang2022fedkc}, it is nevertheless an over-simplification of real-world FL deployments. First, clients mostly determine whether to participate subject only to their own conditions and may be unavailable due to reasons that are completely unpredictable by the server e.g., low battery or no Wi-Fi connection. It is therefore impossible for server to dictate the participation scheme as in server-centric FL. Second, synchronization and homogeneity of local computing is inefficient and unnecessary as clients all have various levels of computing capability. Enforcement of synchrony results in ``straggler effect'', i.e., slower clients lock the entire training process. 

Most existing works which attempt to mitigate \textit{statistical heterogeneity} \cite{li2020federated,yang2021achieving,Hsu2019MeasuringTE,Li20FedProx, qi2022stochastic, zhou2023every} or add adaptivity \cite{reddi2020adaptive,wang22adaptive} are in the server-centric FL regime, which virtually ignore the ubiquitous \textit{system heterogeneity} in large-scale FL deployments.

In light of the limitations of server-centric FL, we propose an FL regime which we refer to as \textbf{Client-Centric Federated Learning} (CC-FL) to have a more precise characterization of real-world settings. And to further address the two key challenges FL algorithms encounter, i.e., client drift and lack of adaptivity, we propose \textbf{Client-Centric Federated Adaptive Optimization}, which is a general framework that consists of a class of novel federated optimization approaches, such as CC-FedAdam/CC-FedAdagrad/CC-FedAMS, which are formalized in Algorithm \ref{alg:cc_fed_adaptive_algs}.

Specifically, in CC-Federated Adaptive Optimization, we enable several unique features which are missed by most existing works: (1) \textit{arbitrarily heterogeneous local computation}, client participates only when it intends to, and each client can self-determine a time-varying and device-dependent number of local epochs; (2) \textit{asynchronous aggregation}, each participating client can work with an asynchronous view of the global model from an outdated timestamp; and (3) \textit{concurrent server-client optimization}, global optimization on server happens concurrently with local update by client, which avoids stragglers from stopping the entire training process. Each client takes an independent role in deciding the participation, and how much computation it carries out in this regime, which is why we refer to it as Client-Centric FL (CC-FL).

One key ingredient we incorporate is the server-side adaptive optimization, which is extremely helpful in mitigating client drift and adding adaptivity when training large-scale overparameterized models. Another critical advantage is that as the design of adaptive optimizers leverages the average of historical gradient information, the global update no longer relies only on current model update as in FedAvg baseline. Note that historical information would effectively regularize the global update from going wild when there is unexpected client distributional drift. Data-driven model training often leads to poor out-of-distribution performance \cite{zheng2022causally}. 

We then provide the convergence analysis for CC-Federated Adaptive Optimization. Our theory reveals key factors that affect the performances, and shows that our proposed approach obtains the best known convergence rate. We are unaware of any existing analysis that studies federated adaptive optimization with asynchrony and heterogeneous local computing.

Finally, extensive experiments show that any of the proposed CC-FedAdam/CC-FedAdagrad/CC-FedAMS can outperform FedAvg and its variants by a large margin across benchmarks. The improvement is consistent across various levels of statistical/system heterogeneity. We also carry out an exhaustive hyperparameter sensitivity analysis and our approach turns out to be quite easy to tune to obtain a much better performance than FedAvg.

Our main contributions can be summarized as follow,

\begin{itemize}[leftmargin=*]
    \item We propose a class of novel FL optimization approaches, which is effective in mitigating the client drift and lack of adaptivity issues in presence of system heterogeneity.
    \item We show our approach matches a best known convergence rate $\mathcal{O}\left(\sqrt{\frac{1}{mKT}}\right)$ for general nonconvex objectives. To our best knowledge, this is the first convergence analysis on federated adaptive optimization with both statistical and system heterogeneity \footnote{$\mathcal{O}\left(\sqrt{\frac{1}{mKT}}\right)$ is w.r.t total rounds $T$, number of local epochs $K$, and number of responsive clients $m$. Please refer to Section \ref{sec:convergence} for details.}.
    \item Empirical results demonstrate that our approaches consistently outperform widely used baseline by a large margin, across benchmarks and levels of client heterogeneity.
\end{itemize}

\textbf{Organization}. The rest of the paper is organized as follows. In Section \ref{sec:background}, we introduce background that is pertinent to our proposed algorithm. In Section \ref{sec:client-centric}, we first discuss the limitations of existing server-centric FL algorithms and then introduce our proposed \textbf{Client-Centric Federated Adaptive Optimization}. In Section \ref{sec:convergence}, we provide the convergence analysis of our proposed approach, followed by Section \ref{sec:exp}, where we provide experimental results that validate the effectiveness of our proposed algorithm. We discuss the related works in Section \ref{sec:related_work}, and conclude our works in Section \ref{sec:conclusion}. We defer the proof of our convergence analysis and extra related works/experiments to Appendix due to space limit.

\section{Background}
\label{sec:background}

\subsection{Federated Averaging (FedAvg)}
\label{subsec:background_fedavg}

Most Federated Learning problems can be formalized as,
 
\begin{equation}
\begin{gathered}
\min_{x\in\mathbb{R}^d}f(x)\triangleq\frac{1}{n}\sum_{i=1}^n f_i(x), \quad
\text{where} \quad f_i(x)=\mathbb{E}_{\xi\sim \mathcal{D}_i} f_i(x,\xi).
\end{gathered}
\label{fed_min_objective}
\end{equation}
where $n$ is the total number of clients and $x$ is the model parameter with $d$ as its dimension. Each client $i$ is associated with a local data distribution $\mathcal{D}_i$ and a local objective function $f_i(x)=\mathbb{E}_{\xi\sim \mathcal{D}_i} f_i(x,\xi)$. The global objective function is the averaged objective among all clients. We consider the general \textit{non i.i.d.} setting, i.e., $\mathcal{D}_i$ can be completely different from $\mathcal{D}_j$ when $i\neq j$.

Federated Averaging (FedAvg) \cite{McMahan2017FedAvg} is the first and probably most popular algorithm to optimize \eqref{fed_min_objective}, formalized in Algorithm \ref{fed_adaptive_alg}. Suppose the total number of communication rounds is $T$, at each round $t\in\{1,\dots,T\}$, FedAvg is composed of \textbf{client-level optimization} and \textbf{server-side optimization}. Specifically, at round $t$, server randomly samples a subset of clients $\mathcal{S}_t$ and sends the global model $x_t$ to each participating client,

\begin{itemize}[leftmargin=*]
    \item \textbf{Client-level Optimization}. Each participating client $i\in\mathcal{S}_t$ initializes the global model at $x_{t,0}^i\gets x_t$, where $x_{t,k}^i$ denotes the $i$-th client's local model at $k$-th local step. Each client $i$ would then conduct $K$ steps of local SGD, and updates $x_{t,k+1}^i=x_{t,k}^i-\eta_l g_{t,k}^i$, where $\eta_l$ is the local learning rate. Client then computes the model difference $\Delta_t^i=x_t-x_{t,K}^i$.
    \item \textbf{Server-side Optimization}. Server collects the model differences from all participating clients $\{\Delta_t^i\}_{i\in\mathcal{S}_t}$, and aggregates by averaging, i.e., $\Delta_t=\frac{1}{ | \mathcal{S}_t  |}\sum_{i\in \mathcal{S}_t}\Delta_t^i$. FedAvg algorithm then simply updates the global model by $x_{t+1}=x_t-\eta \Delta_t$, where $\eta$ is the global learning rate.
\end{itemize}

Note that in vanilla FedAvg \cite{McMahan2017FedAvg}, $\eta=1$ throughout the entire training process, which is equivalent to directly averaging the local model $x_{t,K}^i$.  If we regard averaged model difference $\Delta_t$ as a ``pseudo-gradient'', another interpretation of FedAvg's server-side optimization $x_{t+1}=x_t-\eta \Delta_t$ is a one-step SGD with constant learning rate 1.

Interpreting FedAvg as a one-step SGD with constant learning rate 1 would immediately inspire us to enable a more flexible optimizer other than SGD, e.g. \cite{yang2021achieving} studies FedAvg with two-sided learning rates (i.e., FedAvg with $\eta$ that may not be 1), and shows a proper selection of $\eta$ enables FedAvg to achieve the best known convergence rate for general nonconvex objective functions. FedAvg with Momentum (FedAvgM) \cite{Hsu2019MeasuringTE} is proposed to augment momentum, which is formalized in Algorithm \ref{fed_adaptive_alg}. \cite{Hsu2019MeasuringTE,Wang2020SlowMo} empirically show the convincing performances of FedAvgM, especially when the clients have highly heterogeneous data distributions.

\subsection{Federated Adaptive Optimization}
\label{subsec:background_fed_adaptive}

The key limitation of SGD and SGD momentum is that they apply an identical learning rate to all features uniformly. When the feature space is sparse or heterogeneous, which is ubiquitous when training large-scale models, the lack of adaptivity may lead to limited training speed and inferior learning performances. In light of such limitations, adaptive optimization approaches propose to scale the gradient by square roots of some form of the average of the squared values of past gradients, to apply an adaptive per-feature learning rate. Examples such as AdaGrad \citep{duchi11adaptive} and Adam \citep{Kingma2015AdamAM}, have long become the default optimizer in many learning tasks, due to fast convergence speed \cite{Wilson2017Generalization}. 

\begin{algorithm2e}
\SetAlgoVlined
\KwIn{\\
Number of clients $n$, objective function $f(x)=\frac{1}{n}\sum_{i=1}^n f_i(x)$\; 
Number of communication rounds $T$, \textbf{local} learning rate $\eta_l$, \textbf{local} updating number $K$\;
Initialization $x_0$, \textbf{global} learning rate $\eta$, and $\beta$, $\gamma$, $\epsilon$;}
\SetAlgoLined
\For{$t\in\{1,...,T\}$}
{   
    Randomly sample a subset $\mathcal{S}_t$ of clients\;
    
    Server send $x_t$ to subset $\mathcal{S}_t$ of clients\;
    
    \For{each client $i\in \mathcal{S}_t$}
    {
    Initialize $x_{t,0}^i \gets x_t$\;
    
    \For{$k \in \{ 0,1,...,K-1 \}$}
    {
    Randomly sample a batch $\xi_{t,k}^i$\;
    Compute $g_{t,k}^i=\nabla f_i(x_{t,k}^i\;\xi_{t,k}^i)$\;
    Update $x_{t,k+1}^i=x_{t,k}^i-\eta_l g_{t,k}^i$\;
    }
    $\Delta_t^i=x_t-x_{t,K}^i$\;
    }
    
    Server aggregates $\Delta_t=\frac{1}{| \mathcal{S}_t |}\sum_{i\in \mathcal{S}_t}\Delta_t^i$\;

    \colorbox{yellow}{Update $x_{t+1}=x_t-\eta \Delta_t$}

    \colorbox{green}{$m_{t}=(1-\beta)\Delta_{t}+\beta m_{t-1}$}

    \colorbox{green}{Update $x_{t+1}=x_t-\eta  m_t$}

    \colorbox{cyan}{$m_{t}=(1-\beta)\Delta_{t}+\beta m_{t-1}$}
    
    \colorbox{cyan}{$v_{t}=(1-\gamma)\Delta_{t}^2+\gamma v_{t-1}$}
    
    \colorbox{cyan}{Update $x_{t+1}=x_t-\eta \frac{m_t}{\sqrt{v_t}+\epsilon}$}
    
}
return $x_T$
\caption{\colorbox{yellow}{FedAvg} \citep{McMahan2017FedAvg}, \colorbox{green}{FedAvgM} \citep{Hsu2019MeasuringTE}, \colorbox{cyan}{FedAdam} \citep{reddi2020adaptive}}
\label{fed_adaptive_alg}
\end{algorithm2e}

Inspired by its convincing empirical performances in non-federated settings, existing efforts have been made to combine adaptive optimization approaches with federated learning. \cite{reddi2020adaptive,Chen20LocalAMSGrad,wang22adaptive} propose several federated adaptive optimization approaches, e.g., FedAdam, FedAdagrad, FedYogi, FedAMS, which essentially applies popular non-federated adaptive optimizers Adam \citep{Kingma2015AdamAM}, Yogi \cite{Zaheer18Yogi}, Adagrad  \citep{duchi11adaptive}, and AMSGrad \cite{reddi18adam_convergence} to server-level optimization, respectively. 

For expository purposes, we display FedAdam in Algorithm \ref{fed_adaptive_alg} (other variants can be formulated similarly). Specifically, the server regards $\Delta_t$ as a pseudo-gradient, and updates the global model by Adam optimizer, where $m_t$ is the ordinary ``momentum buffer'' as in FedAvgM, $v_t$ stores the accumulated squared values of past gradients, and the hyperparameter $\epsilon>0$ controls the \textit{degree of adaptivity}, with smaller $\epsilon$ representing higher degrees of adaptivity. By updating the global model with $x_{t+1}=x_t-\eta \frac{m_t}{\sqrt{v_t}+\epsilon}$, federated adaptive optimization mitigates client drift, and the server could adaptively adjust the learning rate on a per-feature basis, which is shown to achieve superior performances than FedAvg \cite{reddi2020adaptive}.

\section{Client-Centric Federated Adaptive Optimization}
\label{sec:client-centric}

In Section \ref{subsec:server-centric-limitation}, we first introduce the key limitations of existing FL algorithms, which requires heavy server-client coordination and ignores system heterogeneity. In light of the limitations, we then introduce our proposed algorithm Client-Centric Federated Adaptive Optimization in Section \ref{subsec:proposed_alg}.

\subsection{Limitations of Server-Centric FL}
\label{subsec:server-centric-limitation}

Most existing federated learning algorithms universally make the following assumption that rarely holds in real world deployment,

\begin{tcolorbox}
\begin{center}
\textbf{At a given round, all clients synchronize with the same global model and they conduct identical number of local computations.}
\end{center}
\end{tcolorbox}

Specifically, if we observe the clients' local computation part in Algorithm \ref{fed_adaptive_alg} (i.e. Line 5-9), we would summarize that the following assumptions have been intrinsically made, (a) \textbf{(Homogeneous Local Updates)} Clients run an identical number of $K$ steps; (b) \textbf{(Uniform Client Participation)} Each client is sampled by the server in a given round uniformly and independently according to an underlying distribution; (c) \textbf{(Synchronous Local Clients)} All participating clients synchronize at any round $t$, i.e., they initialize with the global model at current time $x_t$.

These assumptions provide a useful yet over-simplified characteristic of real-world system. Due to their simplicity, most existing FL algorithms and their corresponding convergence analysis are proposed and derived based on the above assumptions, see e.g. \cite{McMahan2017FedAvg,stich2018localSGD,li2020federated,yang2021achieving,Hsu2019MeasuringTE,Wang2020SlowMo,Li20FedProx,karimireddy2020scaffold,Wang20FedNova,reddi2020adaptive,wang22adaptive,zhang2020fedpd,acar2021feddyn} (please check Section \ref{sec:related_work} and Appendix \ref{sec:related_work_appen} for a comprehensive review of existing works).

With the above assumptions, server takes a centric role in the federated learning system, as all clients always synchronize with the server, and the server determines the client sampling scheme as well as number of local updates, we therefore refer to this setting as \underline{\textbf{Server-Centric Federated Learning System}}.

However, we would like to argue that server-centric federated learning is an unrealistic characterization of real-world system. Especially considering that it typically takes thousands of communication rounds to converge in large-scale cross-device deployments, it is impossible to ensure these assumptions hold throughout the entire process. Instead, clients in a realistic cross-device FL system usually have the following characteristics \cite{li2021survey,Kairouz21AdvancesProblems},

\begin{itemize}[leftmargin=*]
    \item \textbf{(Heterogeneous Client Capability)} The clients may have completely different computational capability, and faster clients are able to carry out more local computations than slower clients during the same time period. Requesting identical local updates would straggle the training and incur unnecessary energy waste.
    \item \textbf{(Unpredictable Client Availability)} The clients may have arbitrary availability due to various constraints, e.g., unstable network connection, battery outage, or low participation willingness. Thus, each client determines to participate or not in a given communication round purely dependent on its own condition and the decision is completely unpredictable by the server. 
    \item \textbf{(Asynchronous Local Model)} Due to the random nature of client participation and capability, a reasonable server does not wait until the slowest client completes local computation, but would instead start the next round as long as a sufficient number of clients respond. The rest of the clients would participate in a future round when they finish and decide to respond. Though this paradigm is more efficient than its synchronous counterpart, the behavior of clients may be more chaotic as clients may compute based on an outdated global model $x_{t-\tau_{t,i}}$ instead of $x_t$ (up to a time-varying, device-dependent random delay $\tau_{t,i}$).
\end{itemize}

In summary, due to their discrepancy against real-world settings, the applicability of these server-centric FL algorithms is questionable. Traditionally, most works that recognize client heterogeneity focus explicitly on statistical heterogeneity \citep{Li2020Fed-Non-IID,yang2021achieving}, i.e., the local data distribution distinguishes with each other, while in reality the system heterogeneity is as ubiquitous as statistical heterogeneity and largely remains unexplored.

\subsection{Methodology}
\label{subsec:proposed_alg}

\begin{algorithm2e}
\SetAlgoVlined
\KwIn{
Number of clients $n$, objective $f(x)=\frac{1}{n}\sum_{i=1}^n f_i(x)$, initialization $x_0$, 
Number of rounds $T$, \textbf{local} learning rate $\eta_l$,
Number of participating clients $m$, server-side hyperparameter $\eta,\beta,\gamma,\epsilon$}
\SetAlgoLined
\For{$t\in\{1,...,T\}$}
{   

    {\textbf{At Each Client}\par}

    \Indp Self-determine whether to participate in the training, if not, stay idle.

    Once determine to participate, retrieve a global model $x_\mu$ from the server ($\mu$ may not be $t$)

    Trigger \textbf{Client-Centric Local Computation} (Algorithm \ref{alg:cc_local}), $\Delta_\mu^i = \textbf{CC-Local}\left(i,t,\mu,\eta_l\right)$

    Send local update $\Delta_\mu^i$

    \Indm {\textbf{At Server (Concurrently with Client)}\par}
    
    \Indp Collect the local updates $\{\Delta_{t-\tau_{t,i}}^i, i\in\mathcal{S}_t\}$ returned from a subset of clients $\mathcal{S}_t$, where $\tau_{t,i}$ represents the random delay of the client $i$'s local update, $i\in\mathcal{S}_t$


    Aggregate: $\Delta_t=\frac{1}{m}\sum_{i\in \mathcal{S}_t}\Delta_{t-\tau_{t,i}}^i$

    Update momentum buffer $m_{t}=(1-\beta)\Delta_{t}+\beta m_{t-1}$\;
    
    \colorbox{lime}{$\hat{v}_{t}=\hat{v}_{t-1}+\Delta_{t}^2$}
    
    \colorbox{teal}{$\hat{v}_{t}=(1-\gamma)\Delta_{t}^2+\gamma \hat{v}_{t-1}$}
    
    \colorbox{yellow}{$v_{t}=(1-\gamma)\Delta_{t}^2+\gamma v_{t-1}$}
    
    \colorbox{yellow}{$\hat{v}_{t}= \max(\hat{v}_{t-1},v_{t})$}
    
    Update $x_{t+1}=x_t-\eta \frac{m_t}{\sqrt{\hat{v}_t}+\epsilon}$\;

}
return $x_T$
\caption{A class of \textbf{Client-Centric Federated Adaptive Optimization} approaches. \colorbox{lime}{CC-FedAdagrad} \colorbox{teal}{CC-FedAdam} \colorbox{yellow}{CC-FedAMS}}
\label{alg:cc_fed_adaptive_algs}
\end{algorithm2e}

\begin{algorithm2e}
\SetAlgoVlined
\KwIn{client index $i$, current round $t$, retrieved global model timestamp $\mu$, \textbf{local} learning rate $\eta_l$}
\SetAlgoLined
Initialize the local model $x_{\mu,0}^i=x_\mu$.

Determine a number of local steps $K_{t,i}$, which can be time-varying and device-dependent based on its own condition.

    \For{$k\in\{0,1,...,K_{t,i}-1\}$}
    {
    Randomly sample a batch $\xi_{\mu,k}^i$
    
    Compute $g_{\mu,k}^i=\nabla f_i(x_{\mu,k}^i; \xi_{\mu,k}^i)$
    
    Update $x_{\mu,k+1}^i=x_{\mu,k}^i-\eta_l g_{\mu,k}^i$
    }
    Compute model difference $\Delta_\mu^i=x_\mu-x_{\mu,K_{t,i}}^i$

    Normalize w.r.t. $K_{t,i}$, and send $\Delta_\mu^i=\frac{\Delta_\mu^i}{K_{t,i}}$

return $\Delta_\mu^i$
\caption{\textbf{Client-Centric Local Computation} $\textbf{CC-Local}\left(i,t,\mu,\eta_l\right)$}
\label{alg:cc_local}
\end{algorithm2e}

In light of the limitations of server-centric federated learning settings and the convincing empirical results of federated adaptive optimization, in this section, we propose a novel framework which we refer to as \underline{\textbf{Client-Centric Federated Adaptive Optimization}}, which is formalized in Algorithm \ref{alg:cc_fed_adaptive_algs}. 

\subsubsection{\textbf{Workflow of Algorithm \ref{alg:cc_fed_adaptive_algs}}}
\label{subsubsec:detail_alg_fed_adaptive}\hfill\\

In the Client-Centric FL (CC-FL) framework, each client takes an independent role in deciding the participation, and how much computation it carries out. Note that CC-Federated Adaptive Optimization is a general and flexible framework in which any adaptive optimizer could serve as a plug-and-play module, and we display only three examples in Algorithm \ref{alg:cc_fed_adaptive_algs}, i.e., CC-FedAdagrad/CC-FedAdam/CC-FedAMS. For expository purposes, we walk through the workflow of Client-Centric Federated Adaptive Optimization only with CC-FedAMS.

Suppose $T$ is the total number of rounds in CC-FedAMS. At round $t\in\{1,2,\dots,T\}$, each client can self-determine whether to participate, and staying idle in the entire training process is allowed. Once it determines to participate, it downloads the global model $x_\mu$ from the server. Note that since each client may choose to participate at a different timestamp, the global model $x_\mu$ may not be from the current timestamp $x_t$. The client then triggers local computation \textbf{CC-Local} (i.e., Algorithm \ref{alg:cc_local}). It first initializes $x_{\mu,0}^i=x_\mu$ locally and carries out $K_{t,i}$ steps of local SGD (i.e., $x_{\mu,k+1}^i=x_{\mu,k}^i-\eta_l g_{\mu,k}^i$ for $K_{t,i}$ steps). Note that in server-centric FL algorithm, $K_{t,i}$ is a constant that does not vary with $t$ and $i$ (see e.g., Algorithm \ref{fed_adaptive_alg}). Here, we allow $K_{t,i}$ to be time-varying and device-dependent. The client then computes a \textbf{normalized} model update by $\Delta_\mu^i=\frac{x_{\mu,0}^i-x_{\mu,K_{t,i}}^i}{K_{t,i}}$ and sends $\Delta_\mu^i$ to server. The normalization is to de-bias the global model to avoid favoring clients with more local updates. 

Concurrently with the local computation \textbf{CC-Local}, the server collects the model updates from responsive clients $\{\Delta_{t-\tau_{t,i}}^i, i\in\mathcal{S}_t\}$, where $\mathcal{S}_t$ is a ``buffer'' of responsive clients. Note that for each client $i\in\mathcal{S}_t$, it may participate at a historical timestamp $t-\tau_{t,i}$, which is up to a random delay $\tau_{t,i}$. In server-centric FL algorithm, $\tau_{t,i}=0$. The global update only takes place once the buffer collects $m$ client updates, i.e., as long as $|\mathcal{S}_t|$ reaches $m$, regardless of the status of the rest of the clients. Such parallel runs of clients and server avoid unnecessary waiting time and wasteful resources. The global updating rule is adapted from a popular non-federated adaptive optimizer AMSGrad \cite{reddi18adam_convergence}, which uses a max stabilization step $\hat{v}_{t}= \max(\hat{v}_{t-1},v_{t})$ that generates a non-decreasing $\hat{v}_{t}$ sequence to solve the non-convergence issue in Adam. 

\subsubsection{\textbf{Key Algorithmic Designs}}
\label{subsec:key-algorithmic-designs}\hfill\\

Several unique features of Algorithm \ref{alg:cc_fed_adaptive_algs} distinguish itself from ordinary server-centric FL, which include (a) \textbf{Normalized Model Update} (Line 9 in Alg \ref{alg:cc_local}), (b) \textbf{Size $m$ Buffer $\mathcal{S}_t$} (Line 8 in Alg \ref{alg:cc_fed_adaptive_algs}), (c) \textbf{Server-side Adaptive Optimization} (Line 11-15 in Alg \ref{alg:cc_fed_adaptive_algs}). These features are key to alleviate system heterogeneity and client drift. We summarize the details and reasoning in Appendix 
\section{Convergence Analysis}
\label{sec:convergence}

In this section, we provide the convergence analysis for our proposed Client-Centric Federated Adaptive Optimization and its practical implications. 

To our best knowledge, this is the first definitive convergence result on federated adaptive optimization with system heterogeneity. Please note that enabling adaptive scaling in asynchronous FL is challenging in theoretical proof and is not an extension of results in \cite{reddi2020adaptive} or \cite{Yang2021AnarchicFL}, since the chaotic iterates caused by asynchrony and a nonzero momentum factor $\beta$ are ignored in \cite{reddi2020adaptive, Yang2021AnarchicFL}.

We study general non-convex objective functions under statistical heterogeneity, i.e., each local loss $f_i(x)$ (and thus the global loss $f(x)$) may be non-convex \footnote{Objective functions in modern large-scale neural networks e.g., VGG \cite{Simonyan14VGG}, ResNet \cite{He2016DeepResNet}, and DenseNet \cite{Huang2017DenseNet} are non-convex \cite{Li2018Visualization}.}, and $\mathcal{D}_i\neq\mathcal{D}_j$ when $i\neq j$.

\begin{assumption}[Smoothness]
\label{smoothness_assumption} 
Each local loss $f_i(x)$ has $L$-Lipschitz continuous gradients, i.e., $\forall x, x^\prime\in \mathbb{R}^d$, we have \\ $\left\|\nabla f_i(x)-\nabla f_i(x^\prime)\right\| \leq L \left\| x-x^\prime\right\|$, where $L$ is the Lipschitz constant. And $f$ has finite optimal value, i.e., $f^\ast\triangleq\min_x f(x)$ exists, and $f^\ast>-\infty$.
\end{assumption}

\begin{assumption}[Unbiased Bounded Gradient]
We could access an unbiased estimator $g_{t,k}^i=\nabla f_i(x_{t,k}^i, \xi_{t,k}^i)$ of true gradient $\nabla f_i(x_{t,k}^i)$ for all $t,k,i$, where $g_{t,k}^i$ is the stochastic gradient with minibatch $\xi_{t,k}^i$. And the stochastic gradient is bounded, i.e., $\left\|g_{t,k}^i\right\|\leq G$.
\label{bounded_gradient_assumption}
\end{assumption}

\begin{assumption}[Bounded Local Variance]
\label{bounded_local_assumption}
Each stochastic gradient on the $i$-th client has a bounded local variance, i.e., we have $\mathbb{E}\left[\left\| g_{t,k}^i - \nabla f_i(x_{t,k}^i) \right\|^2\right] \leq \sigma^2_l$.
\end{assumption}

\begin{assumption}[Bounded Global Variance]
\label{bounded_global_assumption}
Local loss $\{f_i\}_{i=1}^n$ have bounded global variance, i.e., $\forall x$, $\frac{1}{n}\sum_{i=1}^{n}\left\| \nabla f_i(x)-\nabla f(x)\right\|^2\leq\sigma_g^2$.
\end{assumption}

Assumptions \ref{smoothness_assumption} - \ref{bounded_global_assumption} are all standard and mild assumptions in federated learning research, and have been universally adopted in most existing works \citep{Kingma2015AdamAM,reddi18adam_convergence,Li2020Fed-Non-IID,reddi2020adaptive,qi2021stochastic,khanduri2021stem,wang22adaptive, bao2020fast}. Assumption \ref{bounded_gradient_assumption} is a standard assumption when studying adaptive optimization approaches and has been widely adopted in \cite{Kingma2015AdamAM,reddi18adam_convergence,chen2018adam,reddi2020adaptive,wang22adaptive}. 

Note that Assumption \ref{bounded_global_assumption} marks that we are studying general \textit{non i.i.d.} settings, where we allow each client has heterogeneous data distributions, as $\sigma_g^2=0$ corresponds to the \textit{i.i.d.} setting. We also do not require stronger and unrealistic assumptions such as convex objective or Lipschitz Hessian that are used in existing works such as \cite{gu2021arbitraryunavailable,Xie2019AsynchronousFO}. 

Note that the convergence analysis is not a simple combination of existing results, the analysis needs to overcome novel challenges due to the chaotic behavior caused by asynchrony and heterogeneous local computing. As a matter of fact, there is no existing theoretical result on federated adaptive optimization with system heterogeneity to our best knowledge.

We state the main convergence theorem of Client-Centric Federated Adaptive Optimization in Theorem \ref{fedadaptive_free_uniform_arrival_convergence_theorem} \footnote{For expository purposes, we prove the convergence for CC-FedAdagrad. Similar convergence results could be easily generalized to other CC-Federated Adaptive Optimization approaches.}, and then analyze the practical implications in Corollaries and Remarks.

\begin{theorem}[Convergence of CC-Federated Adaptive Optimization]
\label{fedadaptive_free_uniform_arrival_convergence_theorem}
Suppose $\{f_i\}_{i=1}^n$ fulfills Assumptions \ref{smoothness_assumption}-\ref{bounded_global_assumption}. Suppose the maximum delay is bounded, i.e., $\tau_{t,i}\leq\tau<\infty$ for any $i\in\mathcal{S}_t$ and $t\in\{1,\dots,T\}$. 

Under the condition

$$\eta_l\leq\min\left\{\frac{1}{8K_{t,\text{max}}L},\sqrt{\frac{1}{ 120L^2 \epsilon \tau K_{t,\text{max}}^2}},\frac{\epsilon}{\sqrt{T}G} \right\}$$, where $K_{t,\text{max}} =\max_{i\in\mathcal{S}_t}K_{t,i} $, and suppose $H_1 \eta_l^2+H_2\eta_l\le\epsilon^2$, where $H_1\triangleq 2 \eta^2 L^2 \tau^2$, $H_2\triangleq 4 \eta L C_\beta^2 + 6 \eta L \epsilon + 2 G \epsilon$, $C_\beta=\frac{\beta}{1-\beta}$. \footnote{Such constraint on $\eta_l$ is standard and easily fulfilled by ordinary value assignment. Similar constraint has been used in \cite{yang2021achieving,wang22adaptive,Yang2021AnarchicFL}.} And further assume each client is included in $\mathcal{S}_t$ with probability $\frac{m}{n}$ uniformly and independently. We would have:

\begin{gather*}
\frac{1}{T}\sum_{t=0}^{T-1} \mathbb{E}\left[\left\| \nabla f(x_{t })\right\|^2\right]
\leq \frac{8\epsilon}{\eta\eta_l T} \left( f(z_0) -f^\ast \right)  
 + \frac{\Phi}{T}  +  \Phi_g \sigma_g^2  + \Phi_l \sigma_l^2 
 \label{fedadaptive_free_uniform_arrival_bound}
\end{gather*}

where \footnote{We denote $\frac{1}{K_t}=\frac{1}{m}\sum_{i\in\mathcal{S}_t}\frac{1}{K_{t,i}}$, $\bar{K}_t\triangleq \frac{1}{m}\sum_{i\in\mathcal{S}_t}K_{t,i}$, $\hat{K}_t^2 \triangleq \frac{1}{m}\sum_{i\in\mathcal{S}_t}K^2_{t,i}$, $\phi_1 \triangleq  \frac{1}{T}\sum_{t=0}^{T-1} \Bar{K}_t$, $\phi_2 \triangleq  \frac{1}{T}\sum_{t=0}^{T-1} \hat{K}_t^2$, and $\phi_3 \triangleq \frac{1}{T}\sum_{t=0}^{T-1} \frac{1}{K_t}$, for ease of notation. }

\begin{gather*}
\Phi \triangleq  4  d  G^2 C_\beta + L \eta \eta_l G^2 C_\beta^2 \frac{4 d}{ \epsilon }, \quad \Phi_g \triangleq 240   \eta_l^2 L^2 \phi_2,\\
\Phi_l \triangleq 40 \eta_l^2 L^2 \phi_1 +   \frac{4 \eta  \eta_l  L C_\beta^2 + 6 L \eta  \eta_l + 2  \eta_l G}{ m \epsilon}   \phi_3 + \frac{2 \eta^2 \eta_l^2 L^2 \tau^2}{ \epsilon^2 m}  \phi_3
\end{gather*}

\end{theorem}

\begin{proof}
We defer the proof to Appendix \ref{sec:appendix_proof} due to space limit.
\end{proof}

\begin{corollary}[Convergence Rate of Client-Centric Federated Adaptive Optimization]
Suppose an identical $K$ for all $t$ and $i$. By setting $\eta_l=\Theta\left(\frac{1}{\sqrt{T}}\right)$, $\eta=\Theta\left(\sqrt{mK}\right)$, we have the convergence rate as 
\begin{equation}
\begin{gathered}
 \frac{1}{T}\sum_{t=0}^{T-1} \mathbb{E}\left[\left\| \nabla f(x_{t })\right\|^2\right]  = \mathcal{O}\left(\sqrt{\frac{1}{mKT}}\right) +  \mathcal{O}\left(\frac{K^2}{T}\right) + \mathcal{O}\left(\frac{\tau^2}{T}\right)
\end{gathered}
\end{equation}

If we further have a sufficiently large $T$ (i.e. run the algorithm long enough) and an only moderately large $\tau$ (i.e. a manageable level of random delay), specifically, if we have $T\ge mK^5$ and $\tau \leq \left(\frac{T}{mK}\right)^\frac{1}{4}$, we obtain a rate,

\begin{equation}
\begin{gathered}
 \frac{1}{T}\sum_{t=0}^{T-1} \mathbb{E}\left[\left\| \nabla f(x_{t })\right\|^2\right]  = \mathcal{O}\left(\sqrt{\frac{1}{mKT}}\right) 
\end{gathered}
\end{equation}

\label{fedadaptive_free_uniform_arrival_convergence_rate}
\end{corollary}

\begin{remark}[How Random Delay Impacts Convergence?]
Intuitively, due to the chaotic behavior brought by asynchronous clients, the convergence of CC-federated adaptive optimization is expected to be negatively impacted by the random delay $\tau$. Corollary \ref{fedadaptive_free_uniform_arrival_convergence_rate} indicates the slowdown effect from the random delay $\tau$ through the $\mathcal{O}\left(\frac{\tau^2}{T}\right)$ term. But fortunately, Corollary \ref{fedadaptive_free_uniform_arrival_convergence_rate} also implies, if $\tau$ is only moderately large, then client-centric federated adaptive optimization can obtain an $\mathcal{O}\left(\sqrt{\frac{1}{mKT}}\right)$ rate, which does not depend on $\tau$, i.e. the negative impact of using outdated information vanishes asymptotically. Such $\mathcal{O}\left(\sqrt{\frac{1}{mKT}}\right)$ rate indicates our client-centric federated adaptive optimization can converge with chaotic system heterogeneity, and matches the best known convergence rate in asynchronous computing \cite{Lian15Asynchronous,Avdiukhin21arbitrarycommunication,Yang2021AnarchicFL}.
\end{remark}

\begin{remark}[Convergence Rate of Arbitrary Participation] A natural question is that what happens if the clients participate entirely arbitrarily, since we assume client is included in participation uniformly at random in Theorem \ref{fedadaptive_free_uniform_arrival_convergence_theorem}. We could show the resulting convergence rate is $\mathcal{O}\left(\frac{1}{\sqrt{mKT}}\right)+ \mathcal{O}\left(\frac{\tau^2}{T}\right) +  \mathcal{O}\left( \frac{K^2}{T} \right) + \Omega\left(\sigma_g^2\right)$. And we could further show $\Omega(\sigma_g^2)$ is indeed the lower-bound rate (i.e. unavoidable) by constructing worst-case scenario (imagine only two clients $f(x)=\frac{1}{2}(f_1+f_2)$, where the loss for 1st client is $f_1(x)=(x+G)^2$ and the 2nd client is $f_1(x)=(x-G)^2$. In this case $\sigma_g^2=4G^2$. Suppose client 1 never participates, then the optimum $x^\ast$ any algorithm could find is $G$, where $\mathbb{E} [\left\| \nabla f(x^\ast)\right\|^2 ]=\Omega(\sigma_g^2)$). Thus, if there is no condition for the participation, any algorithm is subject to non-convergence. Theorem \ref{fedadaptive_free_uniform_arrival_convergence_theorem} is shown with one particular participation pattern, which can be relaxed or altered.
\end{remark}

\begin{remark}[Linear Speedup w.r.t $m,K$]
The $\mathcal{O}\left(\sqrt{\frac{1}{mKT}}\right)$ rate also reveals an appealing linear speedup effect of number of participating clients $m$ and number of local epochs $K$. Linear speedup in terms of $m$ indicates the convergence is faster with more participating clients. Such speedup is not revealed by existing bound in \cite{Nguyen2021FedBuff}. The speedup in terms of $K$ reflects an important trade-off between local computation and server-client communication, i.e., more local computation (larger $K$) can shorten the convergence, thus requiring fewer rounds of communication (smaller $T$). 
\end{remark}
\section{Experiments}
\label{sec:exp}

We present extensive experimental results on vision and language tasks in \ref{subsec:benchmark_perf} that validate the effectiveness of our proposed approaches. We also analyze the hyperparameter sensitivity in \ref{subsec:hyper_sensitivity}. We defer extra experimental results to Appendix 
due to space limit. Our codes are available at \url{https://anonymous.4open.science/r/CC-Federated-Adaptive-Optimization-FB02/}.

\begin{figure*}[h]
\vspace*{-12pt}
\centering
\subfigure[Training Curves with $\alpha=3.0$]{
\hspace{0pt}
\includegraphics[width=.22\textwidth]{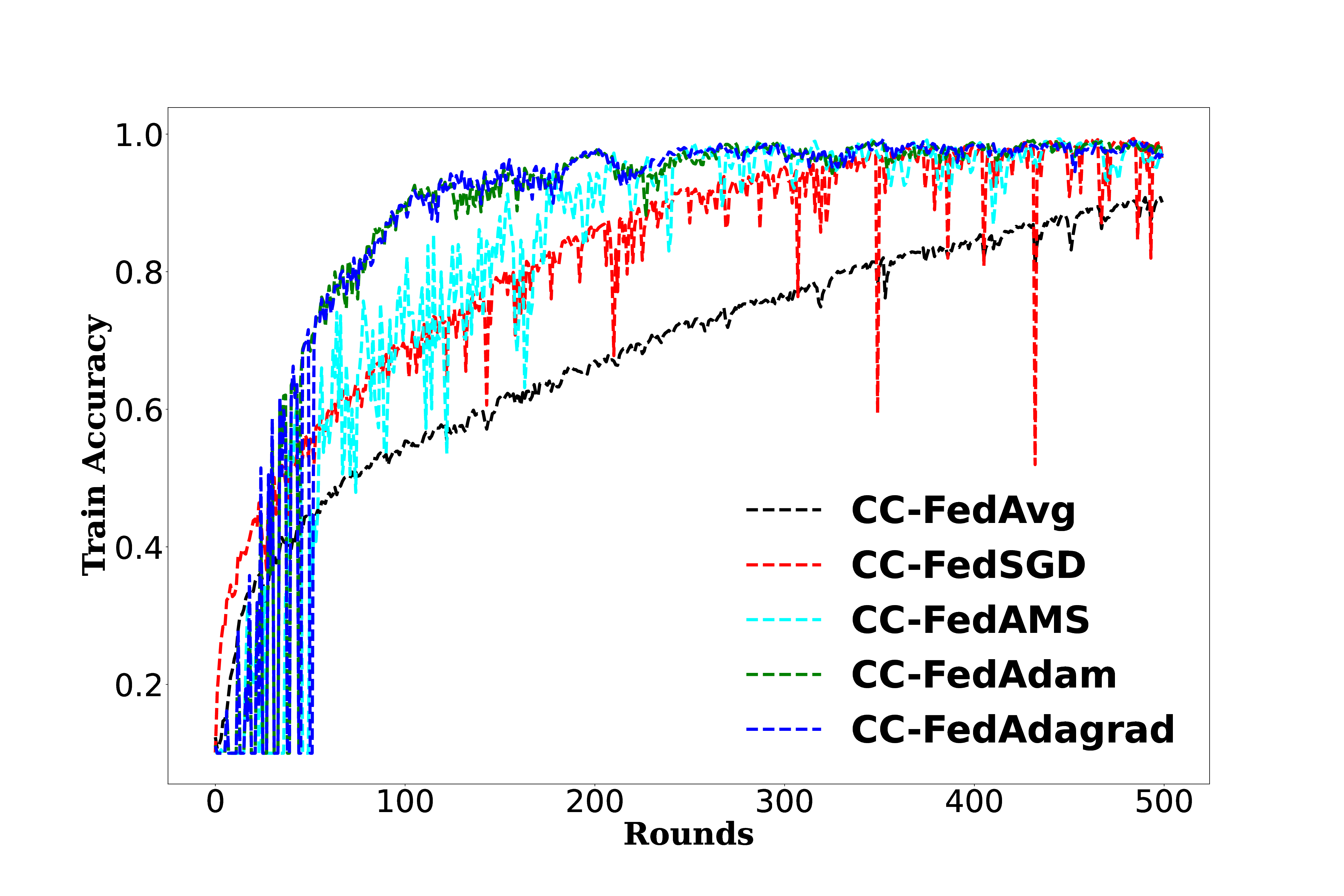}
\label{subfig:cifar10_resnet_adaptive_delay_5_niid_3_train}
}
\subfigure[Testing Curves with $\alpha=3.0$]{
\hspace{0pt}
\includegraphics[width=.22\textwidth]{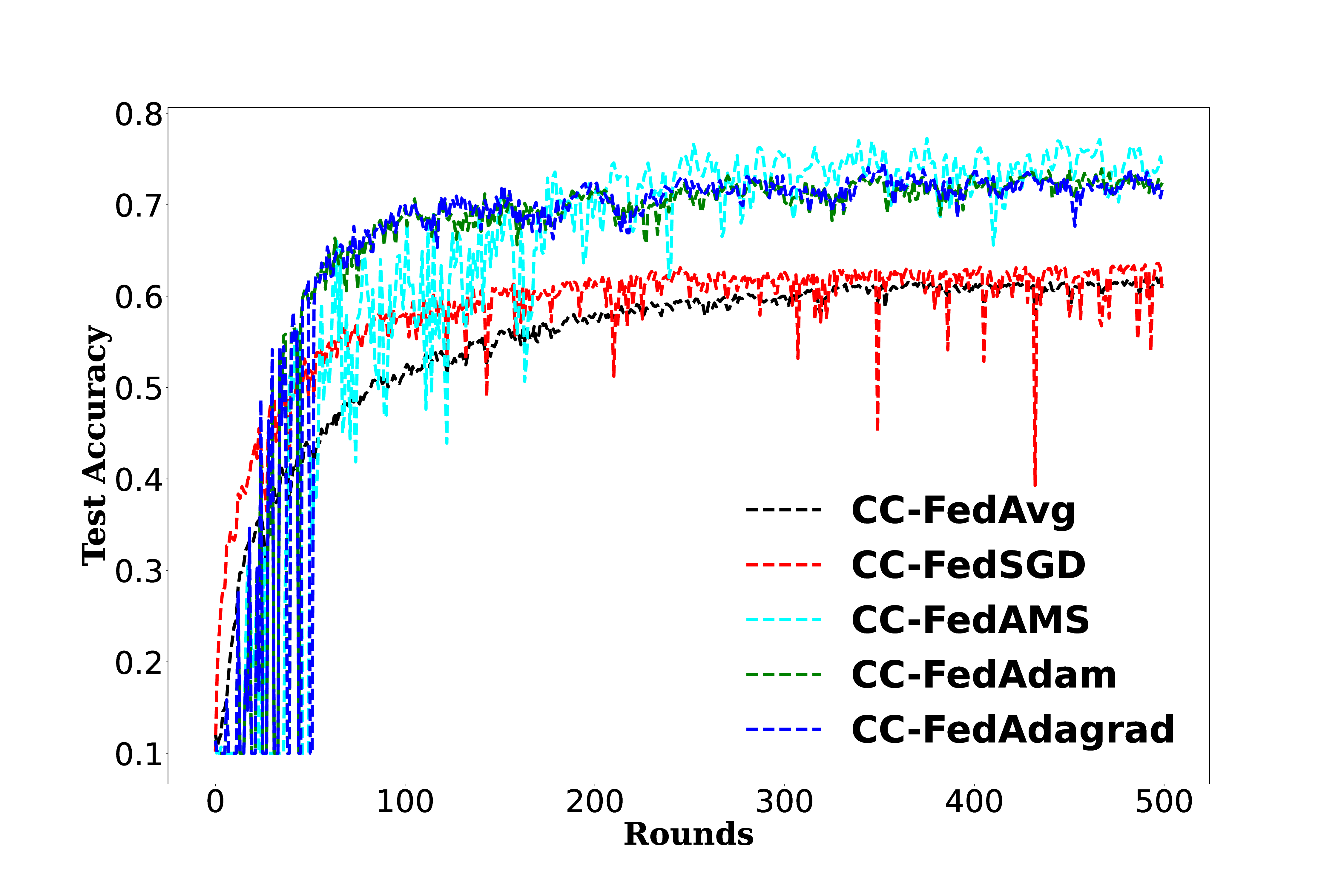}
\label{subfig:cifar10_resnet_adaptive_delay_5_niid_3_test}
}
\subfigure[Training Curves with $\alpha=0.5$]{
\hspace{0pt}
\includegraphics[width=.22\textwidth]{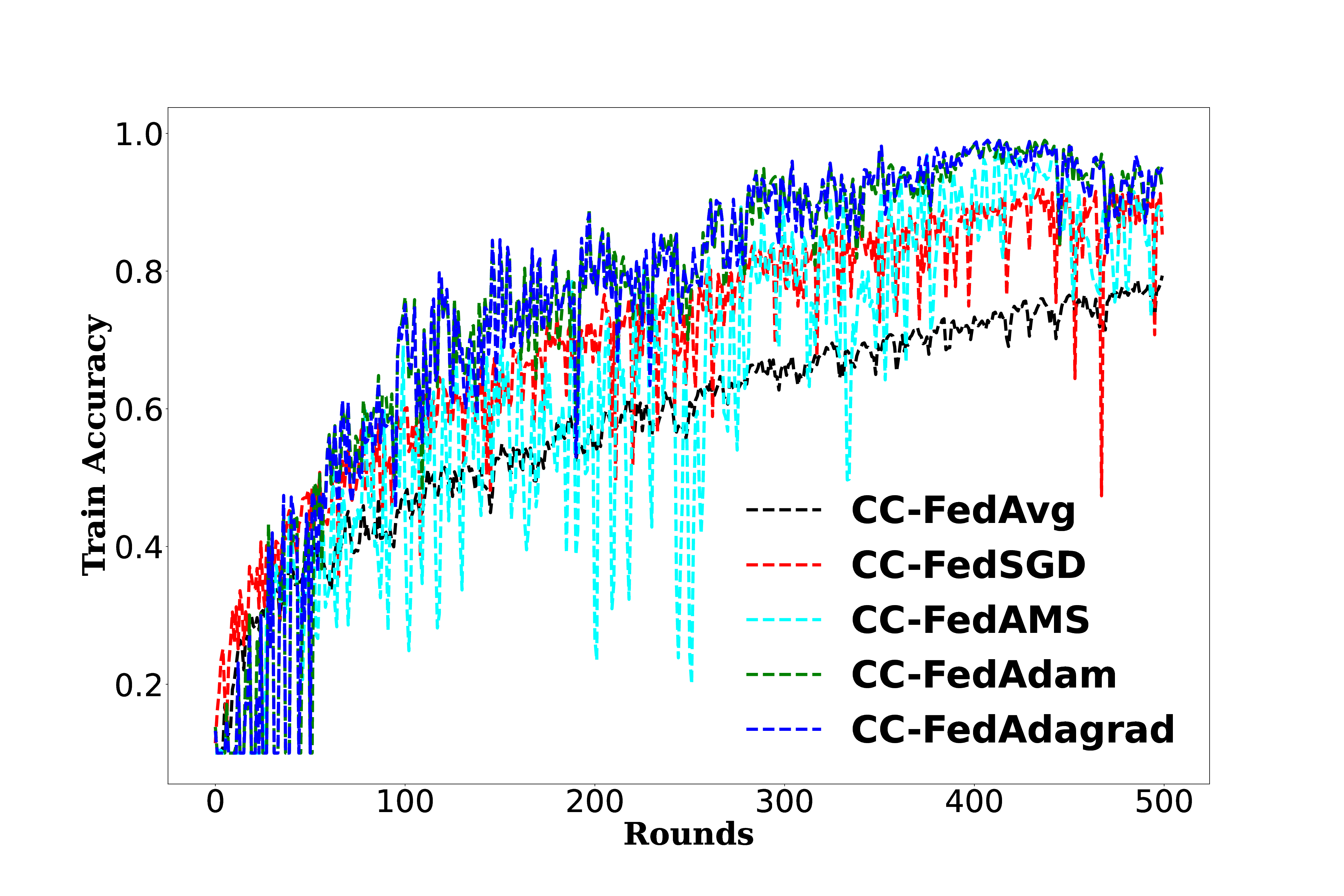}
\label{subfig:cifar10_resnet_adaptive_delay_5_train}
}
\subfigure[Testing Curves with $\alpha=0.5$]{
\hspace{0pt}
\includegraphics[width=.22\textwidth]{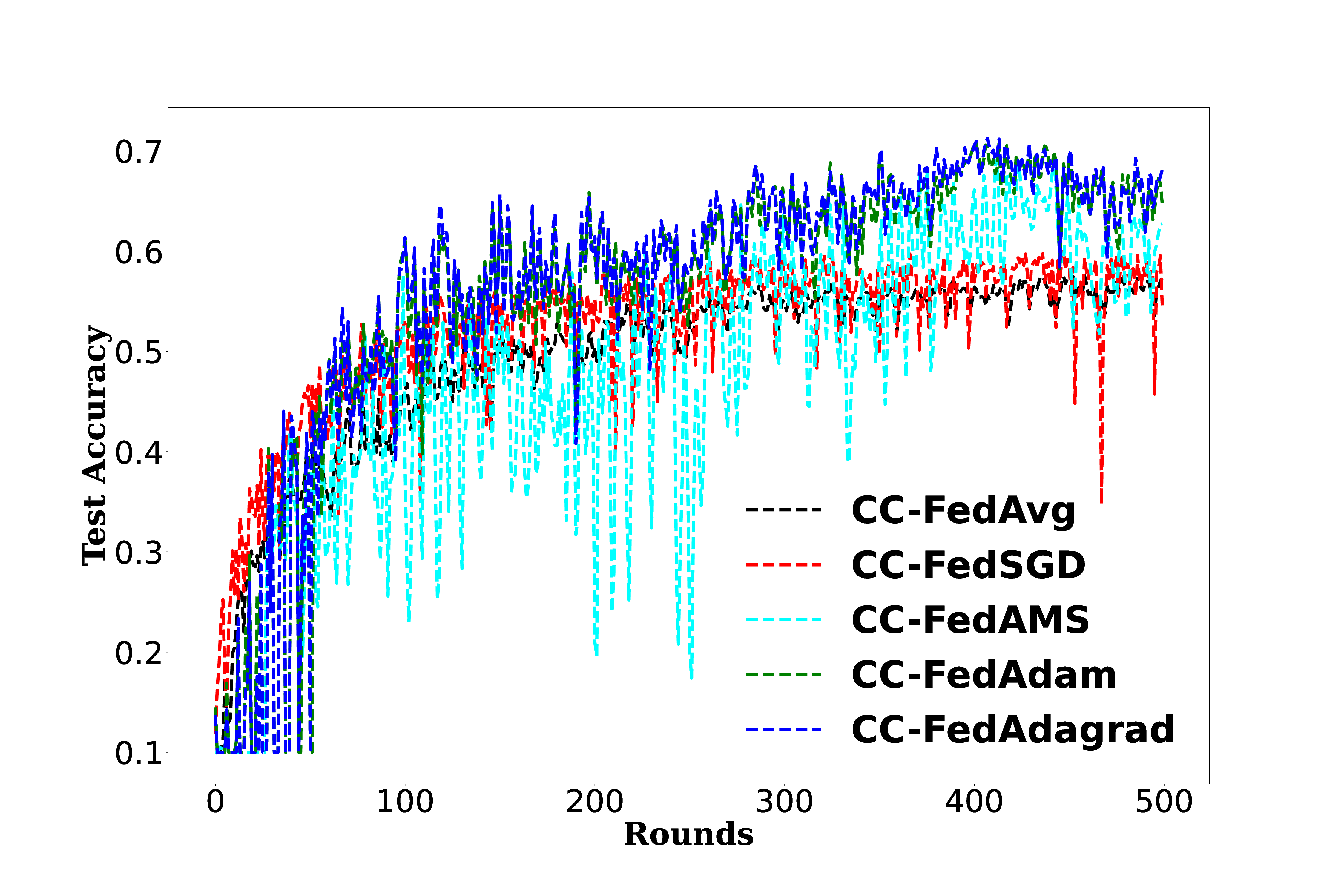}
\label{subfig:cifar10_resnet_adaptive_delay_5_test}
}
\vspace*{-6pt}
\caption{Training and testing curves for various CC-Federated Adaptive Optimizers (ResNet on CIFAR-10) under different Concentration Parameters $\alpha$.}
\label{fig:cifar10_resnet_adaptive_delay_5_niid_3_result}
\end{figure*}

\begin{figure*}[h]
\vspace*{-12pt}
\centering
\subfigure[Testing curve for ResNet on CIFAR-10 with $\tau=10$.]{
\hspace{0pt}
\includegraphics[width=.28\textwidth]{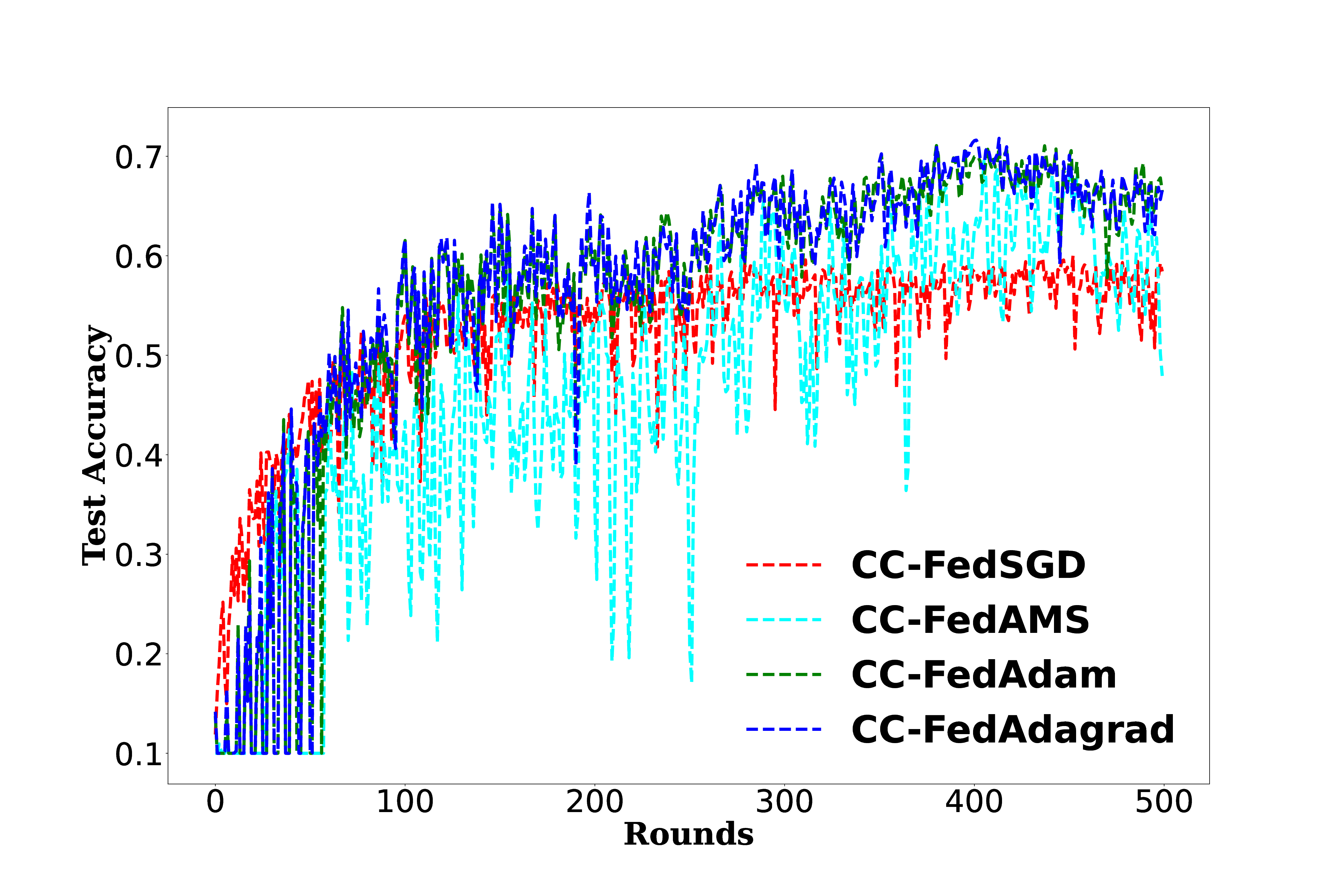}
\label{subfig:cifar10_resnet_adaptive_delay_10_test}
}
\subfigure[Testing curve for ResNet on CIFAR-10 with $R=3$.]{
\hspace{0pt}
\includegraphics[width=.28\textwidth]{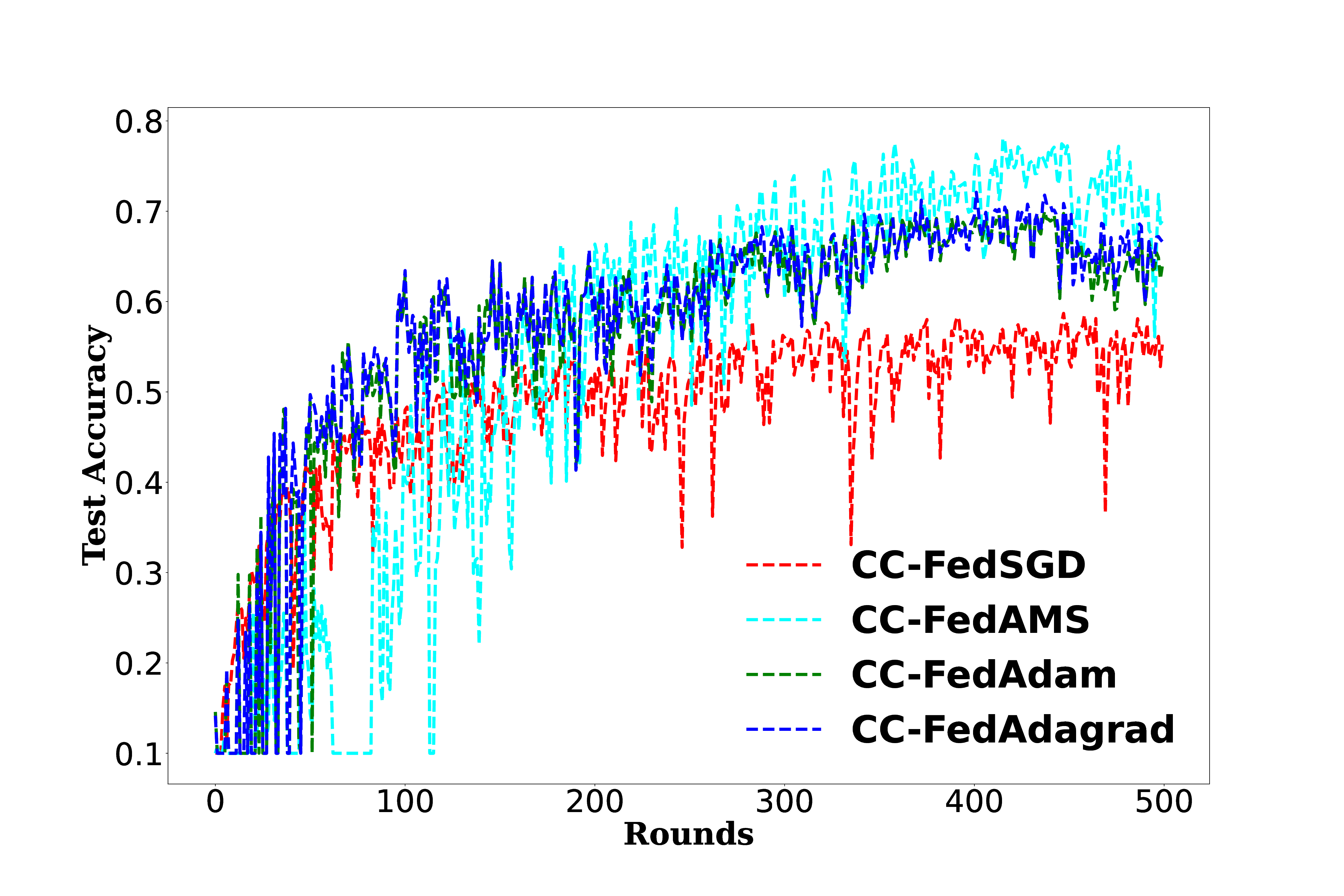}
\label{subfig:cifar10_resnet_delay_5_randomness_3_niid_05_test}
}
\subfigure[Testing curve for a shallow CNN on FMNIST.]{
\hspace{0pt}
\includegraphics[width=.28\textwidth]{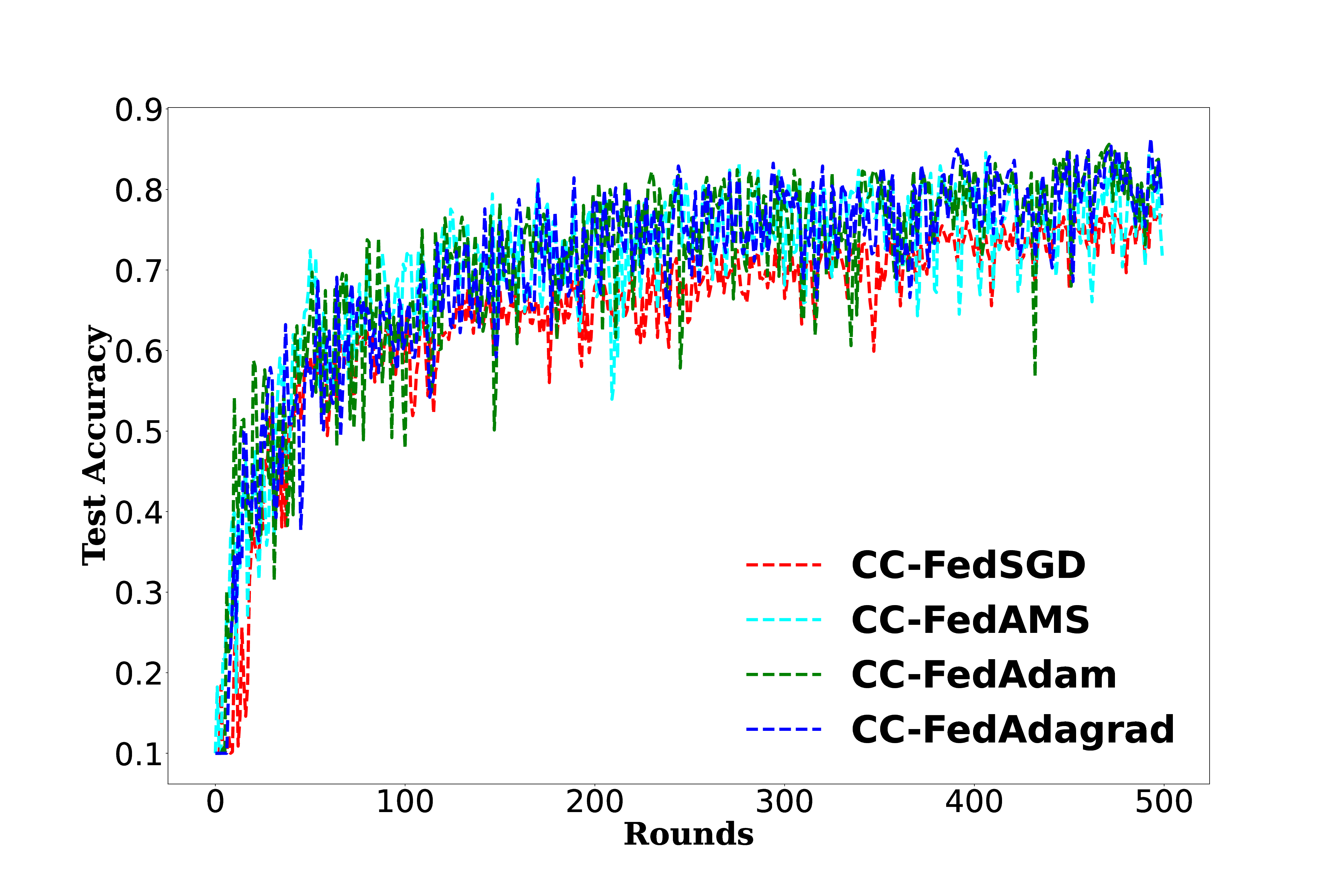}
\label{subfig:fmnist_simple_cnn_adaptive_delay_10_niid_03_test}
}
\vspace*{-12pt}
\caption{Testing curve for various CC-Federated Adaptive Optimizers on different settings of $\tau$, $R$, datasets, and architectures.}
\vspace*{-6pt}
\label{fig:cifar10_resnet_adaptive_various_setting_result}
\end{figure*}

\subsection{Experimental Setting}
\label{subsec:exp_setting}

We experiment with benchmark datasets such as Fashion-MNIST \cite{Xiao2017FashionMNISTAN}, CIFAR10, CIFAR100 \citep{Krizhevsky2009CIFAR}, and StackOverflow. For baseline and our proposed approaches, we sweep a wide range of hyperparameters and display the curves under best hyperparameter settings  \cite{che2021fedtrinet,zhou2022pac, mao2022trace,zheng2023potter}. 

\vspace*{5pt}
\noindent\textbf{How Statistical Heterogeneity is Implemented?}
\vspace*{5pt}

We enforce the label imbalance across all clients to simulate the statistical heterogeneity (i.e. \textit{non i.i.d.} settings). Specifically, each client is allocated a proportion of the samples of each label according to a Dirichlet distribution \footnote{The training examples on each client are sampled independently with class labels following a categorical distribution parameterized by a vector $q$ over $C$ classes (10 in CIFAR-10). Apparently, $q$ fulfills the following properties, $q_i\ge0$ for $i\in\{1,\dots,C\}$, and $\sum_{i=1}^C=1$. We draw $q\sim\text{Dir}(\alpha p)$ from a Dirichlet distribution, and $\alpha>0$ is a concentration parameter controlling the identicalness among clients. We refer full details of the simulation procedure to \cite{Hsu2019MeasuringTE,li2022NIIDBenchmark}} which is parameterized by a \textit{concentration} parameter $\alpha>0$. $\alpha$ controls the degree of non i.i.d. across clients, with $\alpha\to\infty$ indicates all clients have identical distributions (i.e., no statistical heterogeneity) and $\alpha \to 0$ indicates each client only has samples from one random class. Same procedure has been used in \cite{Hsu2019MeasuringTE,li2022NIIDBenchmark,Yurochkin2019BayesianNF,Wang20FedNova}.

\vspace*{5pt}
\noindent\textbf{How System Heterogeneity is Implemented?}
\vspace*{5pt}

System heterogeneity has two facets in this paper, i.e., asynchronous aggregation and heterogeneous local computing. To simulate the asynchrony, we allow each client to select one global model randomly from the last recent $\tau+1$ global models, where $\tau$ is the maximum random delay and controls the degree of asynchrony. Note that ordinary synchronous setting is a special example when $\tau=0$. To simulate the heterogeneous local computation, we allow each worker to randomly select local epoch number from $\{1,2, \dots, K*R\}$ at each round so that each worker has a time-varying, device-dependent local epoch, where $K$ is a default number of local epoch, and $R$ is a degree of randomness. For example, if $K=3$, ordinary homogeneous setting enforces all clients carry out 3 local epochs, but we allow each client randomly select a number of local epoch from $\{1,2, \dots, 6\}$ if $R=2$. Larger $R$ obviously indicates larger randomness.

\vspace*{5pt}
\noindent\textbf{Default Experimental Setting}
\vspace*{5pt}

Unless specified otherwise, we take the following default experimental settings as summarized in Table \ref{default_setting_table}. We have 100 clients in all experiments, and the buffer size $m$ is 5 (i.e. server updates globally once it collects 5 local updates), concentration parameter $\alpha=0.5$, maximum random delay $\tau=5$, default local epoch $K$ is 3, and degree of randomness $R$ is 2. We fix $\beta=0.9$, $\gamma=0.99$ following the settings in \cite{reddi2020adaptive}, except in Section \ref{subsec:hyper_sensitivity} where we test the sensitivity of these hyperparameters.

\begin{table}[htbp]
\vskip -6pt
    \caption{Default Experimental Settings}
    \centering
    \begin{tabular}{c|c}
    \hline
    Number of Clients: 100 & Buffer Size $m$: 5\\ 
    \hline
    Concentration Parameter: $\alpha=0.5$ & Default local epoch $K$: 3\\ 
    \hline
    Local Learning Rate: $\eta_l=0.01$ &  Total Number of Rounds: 500\\ 
    \hline
    $\beta$: 0.9 & $\gamma$: 0.99 \\ 
    \hline
    $\tau$: 5 & Local Momentum: Disabled\\ 
    \hline
    \end{tabular}
    \label{default_setting_table}
\vskip -6pt
\end{table}

\subsection{Performances on Benchmark Datasets}
\label{subsec:benchmark_perf}

We sweep $\eta$ in a wide grid $\eta\in\{10^{-3},10^{-2.5},\dots,10^1\}$, and show the curves with best hyperparameter settings below \cite{mao2022accelerating,zhang2021joint,wang2022ptseformer,li2024integrated}.

Figure \ref{subfig:cifar10_resnet_adaptive_delay_5_niid_3_train} and \ref{subfig:cifar10_resnet_adaptive_delay_5_niid_3_test} show the training/testing performances of training ResNet \cite{He16Res} on CIFAR-10 for 500 rounds. The concentration parameter $\alpha=3.0$ and maximum delay $\tau=5$. Note that CC-FedAvg in the figures denotes the ordinary FedAvg \cite{McMahan2017FedAvg} (i.e. $\eta=1$) in CC framework, while CC-FedSGD denotes a generalized FedAvg (i.e. search the best $\eta$ from the same grid as our proposed approaches). We could observe that this extra degree of freedom brings a significant acceleration in training (CC-FedSGD converges much faster than CC-FedAvg in Figure \ref{subfig:cifar10_resnet_adaptive_delay_5_niid_3_train}). Though it only exhibits a marginal improvement in testing, CC-FedSGD is apparently a stronger baseline than CC-FedAvg. As of our proposed approaches, any of CC-FedAMS/CC-FedAdam/CC-FedAdagrad outperforms CC-FedSGD by more than 10\% in testing performances after only 300 rounds. CC-FedSGD could catch up the training curve of CC-FedAMS in Figure \ref{subfig:cifar10_resnet_adaptive_delay_5_niid_3_train}, while the large gap between testing curves persists in Figure \ref{subfig:cifar10_resnet_adaptive_delay_5_niid_3_test}.

\begin{figure*}[h]
\vspace*{-12pt}
\centering
\subfigure[Testing Accuracy vs. $\beta$]{
\hspace{0pt}
\includegraphics[width=.28\textwidth]{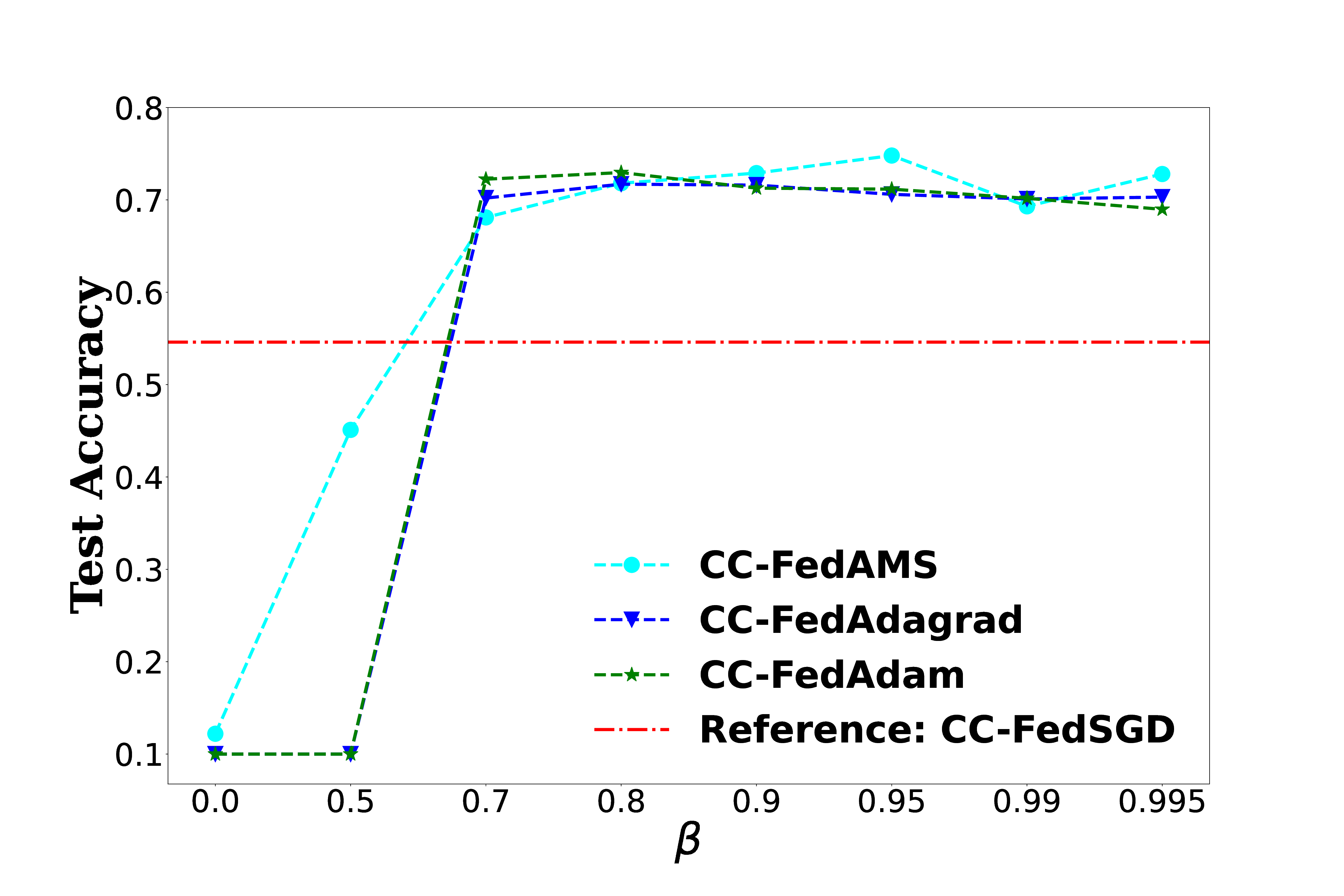}
\label{subfig:beta_sensitivity}
}
\subfigure[Testing Accuracy vs. $\gamma$]{
\hspace{0pt}
\includegraphics[width=.28\textwidth]{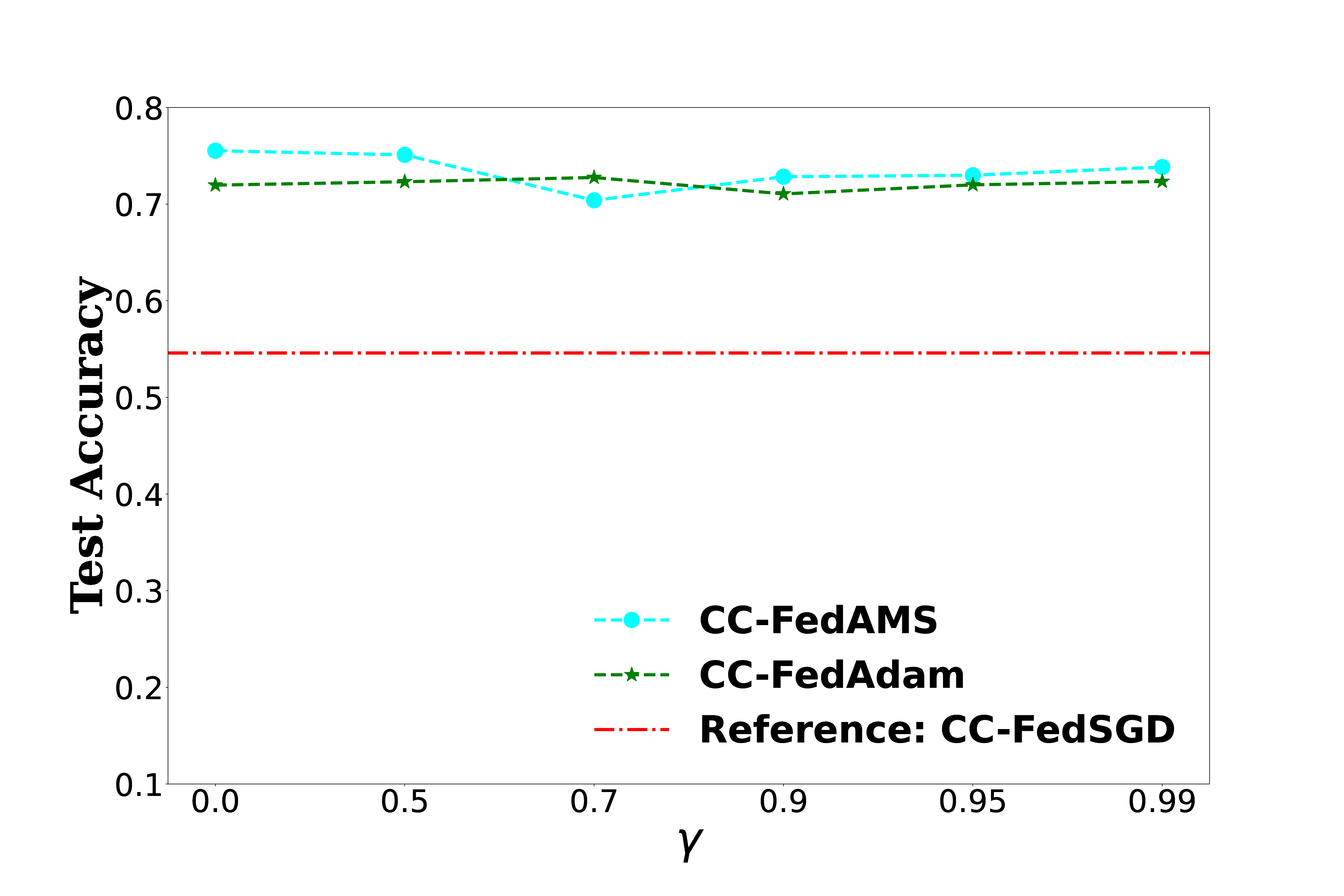}
\label{subfig:gamma_sensitivity}
}
\subfigure[Testing Accuracy vs. $\epsilon$]{
\hspace{0pt}
\includegraphics[width=.28\textwidth]{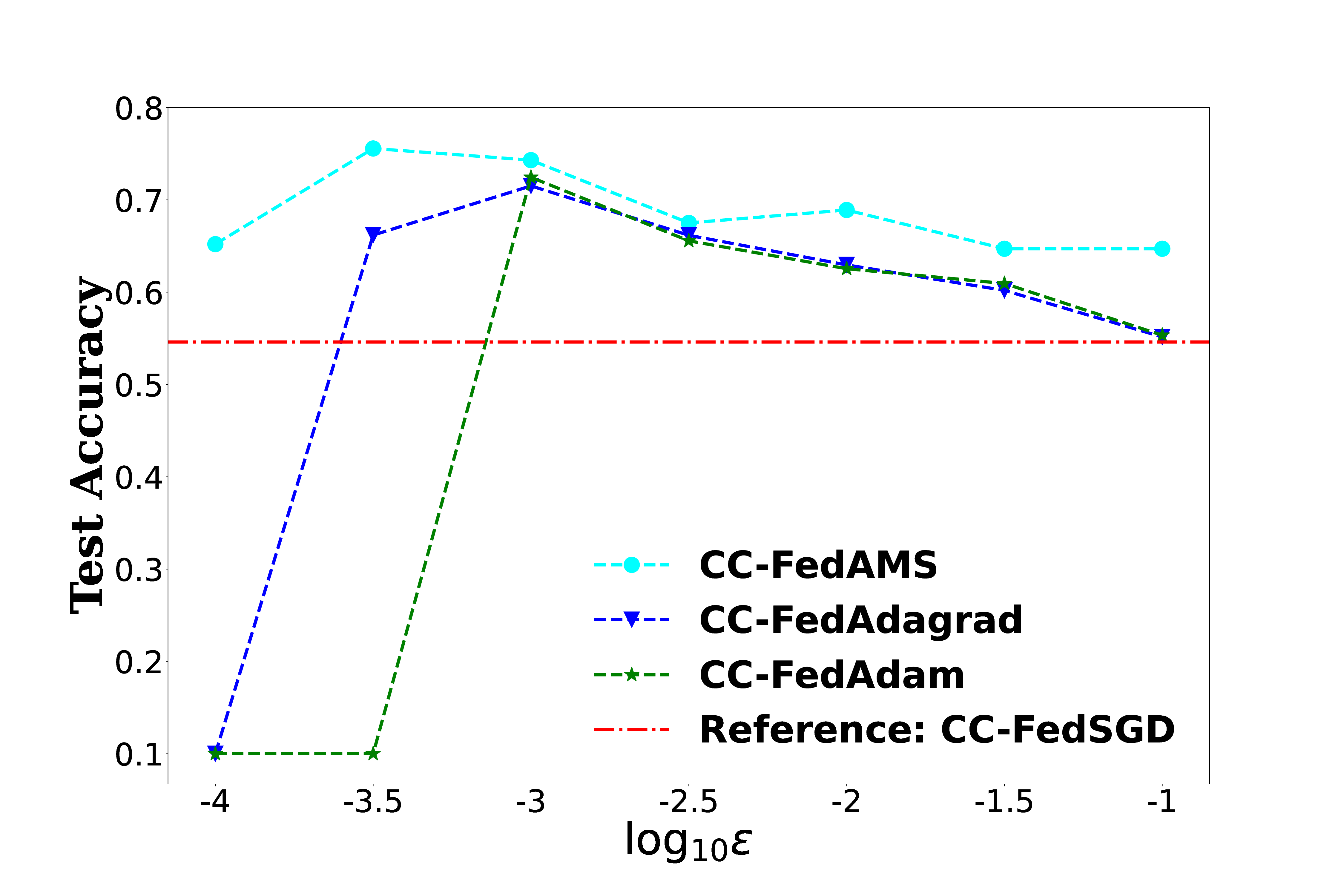}
\label{subfig:epsilon_sensitivity}
}
\vspace*{-12pt}
\caption{Hyperparameter sensitivity of CC-Federated Adaptive Optimizers. Note we do not plot CC-FedAdagrad in Figure \ref{subfig:gamma_sensitivity}, as there is no $\gamma$ in CC-FedAdagrad.}
\vspace*{-6pt}
\label{fig:hyper_sensitivity_result}
\end{figure*}

\begin{table}
\caption{Results on CIFAR-100 \& StackOverflow}
\begin{minipage}{\linewidth}
\centering
\begin{tabular}{c|c|c}
\toprule
\diagbox{Approach}{Dataset} & \thead{CIFAR-100 \\ (Accuracy)} & \thead{StackOverflow \\ (Recall@5)}\\ 
\midrule
CC-FedSGD & 44.0 &  29.8  \\ \hline
CC-FedAdagrad (\textbf{Ours}) & 47.3 &  \textbf{65.1} \\ \hline
CC-FedAdam (\textbf{Ours}) & \textbf{51.5} &  64.8  \\ \hline
\bottomrule
\end{tabular}
\label{extra_exp_table}
\end{minipage}
\end{table}

\vspace*{5pt}
\noindent\textbf{Improvement is Consistent under Various Settings}
\vspace*{5pt}

We test the performances across different settings of $\tau$, $R$, datasets, and architectures. 

We run experiments on more benchmarks, one image task (CIFAR-100) and one language task (tag prediction on StackOverflow). The results are presented in Table \ref{extra_exp_table}. Our proposed approaches outperform the baseline significantly across benchmarks, especially in NLP task (approximately 30\%), where gradient tends to be more sparse which adaptive optimizers can capitalize on \cite{reddi2020adaptive}.

In Figure \ref{subfig:cifar10_resnet_adaptive_delay_5_train} and \ref{subfig:cifar10_resnet_adaptive_delay_5_test}, we show the training/testing performances with a higher level of statistical heterogeneity, i.e. $\alpha=0.5$ while all other settings are identical to Figure \ref{subfig:cifar10_resnet_adaptive_delay_5_niid_3_train} and \ref{subfig:cifar10_resnet_adaptive_delay_5_niid_3_test}. The similarly superior performances of CC-FedAMS/CC-FedAdam/CC-FedAdagrad persist (by a 10\% margin over CC-FedSGD). 

In Figure \ref{fig:cifar10_resnet_adaptive_various_setting_result}, we show results across different $\tau$, $R$, and benchmarks. Note that we only show the testing curves here as it is a more important metric than training curves. We defer training curves to Appendix 
due to space limit.

Specifically, in Figure \ref{subfig:cifar10_resnet_adaptive_delay_10_test}, we test a higher level of system heterogeneity $\tau=10$, i.e., each client is allowed to randomly sample from the most recent 10 global models, while all other settings are identical to Figure \ref{subfig:cifar10_resnet_adaptive_delay_5_test} where $\tau=5$. Still, under a more asynchronous environment, it is obvious our proposed approaches outperform CC-FedSGD consistently. In Figure \ref{subfig:cifar10_resnet_delay_5_randomness_3_niid_05_test}, we test $R=3$, i.e., each client is allowed to randomly select a local epoch in $\{1,2, \dots, 9\}$ instead of $R=2$ in Figure \ref{subfig:cifar10_resnet_adaptive_delay_5_test}. We could observe, the performance gap between CC-FedAMS and CC-FedSGD actually increases to a 15\% margin with more heterogeneous local computations. 

Figure \ref{subfig:fmnist_simple_cnn_adaptive_delay_10_niid_03_test} shows the results of running a shallow CNN architecture from \cite{McMahan2017FedAvg} on Fashion-MNIST dataset. The CC-Federated Adaptive family of approaches still perform better than CC-FedSGD. But the gap is not as large as ResNet. The reason is likely due to adaptivity is most advantageous in settings where the feature space is sparse or anisotropic, which is common in training overparameterized deep models such as ResNet.

More experimental results are deferred to Appendix. 

\subsection{Ease of Hyperparameter Tuning}
\label{subsec:hyper_sensitivity}

We sweep a wide range of key hyperparameters $\{\beta,\gamma,\epsilon\}$ in CC-Federated Adaptive Optimizers under default experimental settings. Specifically, we sweep a $\beta$ grid: $\{0.0,0.5,0.7,0.8,0.9, \dots, 0.995\}$, $\epsilon$ grid: $\{10^{-4},10^{-3.5},\dots,10^{-1}\}$, and $\gamma$ grid: $\{0.0,0.5,0.7,0.9,0.95,0.99\}$, and report results in Figure \ref{fig:hyper_sensitivity_result}. We plot the best performance of CC-FedSGD as a reference. We could see it is quite easy to tune the hyperparameters to obtain a much better performance compared to CC-FedSGD, e.g., $\beta\ge0.7$, $\epsilon\in\left[10^{-3.5},10^{-1.5}\right]$ and $\gamma$ is flexible.

\section{Related Work}
\label{sec:related_work}

In this section, we discuss related work and how this paper develops on top of prior arts. We mainly review works on (1) Adaptive Optimization Approaches, (2) Server-Centric Federated Optimization, and (3) System Heterogeneity Aware FL.

\subsection{Adaptive Optimization Approaches}

Machine learning models have been widely applied in many different domains, e.g. \cite{Mnih2013PlayingAW,He2016DeepResNet,Devlin2019BERT,Jure2017GNN,google16deep&wide,vaswani2017attention,SuoICHI19,Xun2020CorrelationNF,Wen2023FeaturesplittingAF,Wen24Nonconvex,wei2023towards,Wang24GraphICML}, and the de facto optimizer for these ML models especially in the era of deep models is a large class of adaptive optimizers represented by Adam. 

There is a wealth of work on adaptive optimization approaches in non-federated settings. As this paper is mainly focused on FL, we only provide a brief review of this line of research. SGD type optimizers rely heavily and sensitively on proper hyperparameters (learning rate, batch size) setup \cite{Smith18DontDecay,GoyalDGNWKTJH17LargeMinibatch,jastrzebski2018three,London2017APA,He19ControlBatch,Sun21KDDHyperparameter,Sun22KDDHyper,Sun2023TKDD,Sun23Enhance} and also motivated by the poor performances of SGD type optimizers in presence of sparse features and heavy-tail stochastic gradient noise distributions, adaptive optimizers, is shown to converge much faster than SGD in many applications \cite{Wilson2017Generalization}. Most adaptive optimizers share similar design, i.e. scale coordinates of the gradient by square roots of some form of averaging of the squared coordinates on a per-feature basis \citep{reddi18adam_convergence}. Most representative adaptive optimizers include AdaGrad \citep{duchi11adaptive}, Adam \citep{Kingma2015AdamAM}, and their variants, e.g.  RMSProp \citep{tieleman2012rmsprop}, AdaDelta \citep{Zeiler2012ADADELTAAA}, AdaBound \citep{luo2018adabound}, AMSGrad \citep{reddi18adam_convergence}, AdaBelief \citep{zhuang2020adabelief}, RAdam \citep{Liu2020RAdam}.

Note that most of the adaptive optimizers listed above can serve as a plug-and-play server-side module in our proposed client-centric federated adaptive optimization framework, similarly as how we construct CC-FedAdam. This paper shows the empirical and theoretical properties of three of them, i.e. Adam/AMSGrad/Adagrad. 

\subsection{Server-Centric Federated Optimization}

As we explained in Section \ref{sec:intro}, most of the existing federated optimization researches belong to the server-centric FL regime, which naturally distinguish from this paper. We mainly review the works through the lens of how they mitigate \textit{statistical heterogeneity}, which this paper also considers as a source of the \textit{client drift}. We refer readers to \cite{Wang2021survey_fed_opt} for a complete survey of federated optimization.

FedAvg can diverge under high degrees of heterogeneity in worst case scenarios. To stabilize training even in presence of heterogeneity, the following approaches have been proposed, (1) penalize local models that are far away from the global model by regularizing the local objectives, e.g., FedProx \cite{Li20FedProx}, FedPD \cite{zhang2020fedpd}, and FedDyn \cite{acar2021feddyn}; (2) de-bias client's local model updates using techniques like control variates in SCAFFOLD \cite{karimireddy2020scaffold} or local posterior sampling in FedPA \cite{alshedivat2021fedpa}; and (3) treat local updates as pseudo-gradient and incorporate momentum/adaptive optimizers in either server or client optimization, e.g. FedAvgM \cite{Hsu2019MeasuringTE}, SlowMo \cite{Wang2020SlowMo}, FedAdam \cite{reddi2020adaptive}, FedAMS \cite{wang22adaptive}, which are most related to this paper. These server-centric federated momentum/adaptive works largely ignore system heterogeneity, which this paper focuses on.

\vspace*{-10pt}
\subsection{System Heterogeneity Aware FL}

Existing works that study system heterogeneity can be categorized into the following classes, 

(1) \textit{Heterogeneous local computing but synchronous aggregation}. \cite{Wang20FedNova} first shows heterogeneous number of local updates results in a mismatched global convergence and proposes an approach called FedNova to effectively alleviate such mismatch. Other representative approaches in this class include \cite{Basu19Qsparse-Local-SGD,Avdiukhin21arbitrarycommunication}.

(2) \textit{Flexible participation scheme but synchronous aggregation}. There are two themes of research in this class. The first line concentrates on biased client selection \citep{Nishio2018ClientSelection, Chen2020ClientSampling, cho22biased_selection}, most of which study the biased sampling strategy that the probability a client is sampled is related to its local loss to accelerate training \citep{Goetz2019ActiveFL,Ribero2020clientsampling}. This line of works is less related to this paper. The second line studies different patterns of client participation \cite{Yan2020DistributedClient,gu2021arbitraryunavailable,wang2022arbitraryparticipation,Ruan2021FlexibleParticipation}. For example, FedLaAvg \cite{Yan2020DistributedClient} and MIFA \cite{gu2021arbitraryunavailable} both allow clients to participate arbitrarily as long as all devices participate once in the first round. But the clients are still strictly synchronized and the theoretical analysis of MIFA requires a Lipschitz Hessian assumption, which is too stringent and our analysis gets rid of such assumption.

(3) \textit{Asynchronous aggregation}. This class is mostly related to this paper, but there are very limited works on asynchronous FL. Most related to ours are \cite{Xie2019AsynchronousFO,Nguyen2021FedBuff,Yang2021AnarchicFL}. Specifically, FedAsync \cite{Xie2019AsynchronousFO} allows the server immediately updates the global model whenever it receives a single local model to enable asynchrony. However, since the update from each individual client is no longer anonymized in an aggregate, FedAsync has privacy concerns. Moreover, their theoretical analysis only applies to convex objectives. FedBuff \cite{Nguyen2021FedBuff} and AFL \cite{Yang2021AnarchicFL} both propose to trigger a global update whenever the server receives $m$ local updates to ensure anonymity. However, FedBuff only considers homogeneous local computing. And both FedBuff and AFL ignore the lack of adaptivity issue.

In summary, the above existing works either relax only one of the unrealistic assumptions server-centric FL makes but not all of them, or require stringent assumptions in theoretical analysis.

\vspace*{-4pt}
\section{Conclusion}
\label{sec:conclusion}

In this paper, we proposed the Client-Centric Federated Adaptive Optimization, which is a class of novel federated optimization approaches. In contrast to most existing literature, we enable arbitrary client participation, asynchronous server aggregation, and heterogeneous local computing to fully characterize the ubiquitous system heterogeneity in real-world applications. We show both theoretically and empirically the convincing performances of our proposed algorithms. To our best knowledge, this is the first work that addresses statistical/system heterogeneity and lack of adaptivity issues simultaneously. The proposed client-centric FL regime is of independent interest for future FL algorithmic development as a more realistic testbed. 

\newpage
\bibliographystyle{ACM-Reference-Format}
\bibliography{jianhui}

\newpage
\appendix
\onecolumn

\centerline{\huge\textbf{Appendix}}
\vspace{10pt}

In Section \ref{subsubsec:key_alg_design}, we discuss key model designs omitted from Section \ref{subsec:key-algorithmic-designs}. In Section \ref{sec:appendix_proof}, we provide complete proof of Theorem \ref{fedadaptive_free_uniform_arrival_convergence_theorem} and Corollary \ref{fedadaptive_free_uniform_arrival_convergence_rate}. In Section \ref{sec:appendix_more_exp}, we provide more experimental results which are omitted from main text.

\section{Key Algorithmic Designs} 
\label{subsubsec:key_alg_design}

We would like to summarize the following unique features of Algorithm \ref{alg:cc_fed_adaptive_algs} that distinguish itself from server-centric FL,

\begin{itemize}[leftmargin=*]
    \item Global update by server happens concurrently with local update by client, which avoids low-capacity clients from straggling training or any system locking.
    \item Client participates whenever it intends to, and each client can self-determine a time-varying and device-dependent number of local updates $K_{t,i}$. 
    \item Each participating client can work with an asynchronous view of the global model from an outdated timestamp $t-\tau_{t,i}$.
\end{itemize}

We highlight several algorithmic designs that are key to Algorithm \ref{alg:cc_fed_adaptive_algs} in this section.

\vspace*{6pt}
\noindent \textbf{(a) Normalized Model Update}
\vspace*{6pt}

In CC-FedAMS, we allow a time-varying and device-dependent $K_{t,i}$, which would result in the global update biased towards the client with larger $K_{t,i}$. Figure \ref{fig:fednova_toy_example} demonstrates this phenomenon with a two-client toy example. Specifically, suppose client 1 and client 2 carry out two and five local updates, and reach $x_{0,2}^1$ and $x_{0,5}^2$ (following the same $x_{t,k}^i$ notation in Algorithm \ref{alg:cc_fed_adaptive_algs}), respectively. $x_\ast$, $x_\ast^1$, and $x_\ast^2$ denote the global optimum, local optimum for client 1 and 2, respectively. As is evident in this example, $x_{0,5}^2$ would likely drag the global update (red solid arrow) more towards $x_\ast^2$ since client 2 takes more local updates and $x_0-x_{0,5}^2$ thus has a larger scale. This will make the convergence towards global optimum $x_\ast$ slowing down or even failing in worst case \cite{Wang20FedNova}.

\begin{figure}[htp]
    \centering
    \includegraphics[height=30mm, width=50mm]{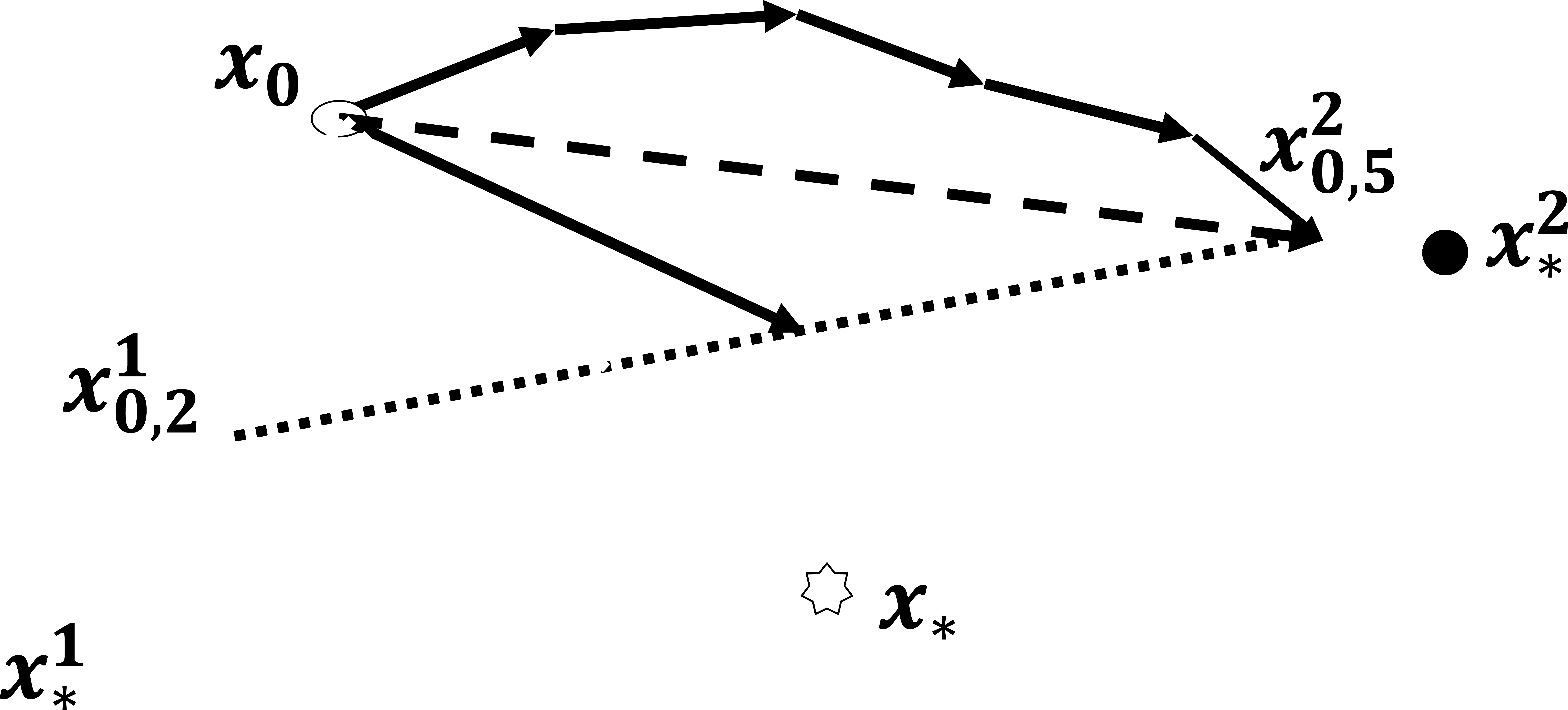}
    \caption{A toy example of a two-client FL setting. There is a mismatch between the ideal direction towards global optimum (green dashed arrow) and the actual search direction (red solid line) if the two clients take a different number of local updates.}
    \label{fig:fednova_toy_example}
\end{figure}

Therefore, before sending the model difference to the server, we request each client normalizes the model update by $K_{t,i}$, i.e., $\Delta_\mu^i=\frac{x_{\mu,0}^i-x_{\mu,K_{t,i}}^i}{K_{t,i}}$.

\vspace*{6pt}
\noindent \textbf{(b) Size $m$ Buffer $\mathcal{S}_t$}
\vspace*{6pt}

In CC-FedAMS, we concurrently run server and client computations. The server maintains a ``buffer'' of responsive clients $\mathcal{S}_t$ and only triggers global update once the buffer reaches size $m$. There is a fundamental trade-off in the selection of $m$. Specifically, if $m=1$ as in \cite{Xie2019AsynchronousFO}, the global update takes place once one client responds, which has the fastest updating frequency and avoids the idle time to collect a size $m$ buffer. However, $m=1$ has the following significant drawbacks. First, it reveals the model update from one single client, which has profound privacy concerns. One of the most important motivations for FL is to anonymize any single client by server-side aggregation, which $m=1$ clearly violates. Second, it is more vulnerable to client drift. $m=1$ indicates the global update relies entirely on one single client at each round, and the global convergence would be slowdown if that client is subject to dramatic distributional shift. Third, too fast global update results in more chaotic client behaviors. Too frequent global update would make clients up to a larger delay, and potentially negatively impact the convergence. \footnote{In Section \ref{sec:convergence} Corollary \ref{fedadaptive_free_uniform_arrival_convergence_rate}, we can see the convergence rate has a term $\mathcal{O}\left(\frac{\tau^2}{T}\right)$, which theoretically validates larger randomly delay $\tau$ slows down global convergence and thus too small $m$ may be harmful.} Empirically, we find an $m>1$ setting converges much faster than $m=1$ in our experiments. ($m=5$ when there are 100 clients in our experiments is a good setting across benchmarks and does not need tuning).

\vspace*{6pt}
\noindent \textbf{(c) Server-side Adaptive Optimization}
\vspace*{6pt}

In CC-FedAMS, a critical ingredient is the server side AMSGrad-type global update. We would like to highlight why adaptive optimization is necessary here. (1) Adaptivity is crucial when training large-scale overparameterized models. As was explained in Section \ref{sec:intro}, adaptively scaling learning rate on a per-feature basis could accelerate convergence with sparse feature space. (2) Adaptive optimization helps mitigate client drift. Note that the design of adaptive optimizer leverages the average of historical gradient information, e.g., it maintains a momentum by recursively aggregating historical model updates $m_{t}=(1-\beta)\Delta_{t}+\beta m_{t-1}$ and it scales coordinates of the gradient by square roots of the recursive averaging of the squared coordinates $x_{t+1}=x_t-\eta \frac{m_t}{\sqrt{\hat{v}_t}+\epsilon}$. If without adaptive optimization, the global update would rely entirely on the averaged current model update, which is subject to unexpected drift. When clients are dynamically heterogeneous, historical gradient information is virtually a regularizer to restrain the search direction from going wild.

\section{Proof of Theorem \ref{fedadaptive_free_uniform_arrival_convergence_theorem} and Corollary \ref{fedadaptive_free_uniform_arrival_convergence_rate}}
\label{sec:appendix_proof}

\subsection{Proof of Theorem \ref{fedadaptive_free_uniform_arrival_convergence_theorem}}

\begin{proof}[Proof of Theorem \ref{fedadaptive_free_uniform_arrival_convergence_theorem}]

We introduce a Lyapunov sequence $\{z_t\}_{t=0}^{T-1}$ which is devised as follows:

\begin{equation}
\label{auxiliary_seq}
z_t= x_t+\frac{\beta}{1-\beta}\left(x_t-x_{t-1}\right)
\end{equation}

We could verify \footnote{Momentum buffer $d_t$ here is $m_t$ in main text, also in Appendix, the zero-index of time $t\in\{0,1,\dots,T-1\}$ is equivalent to the one-index of time $t\in\{1,\dots,T\}$ in main text.}, 

\begin{equation}
\begin{gathered}
z_{t+1}-z_t=\frac{1}{1-\beta}\left(x_{t+1}-x_t\right)-\frac{\beta}{1-\beta}\left(x_t-x_{t-1}\right)\\
=\frac{1}{1-\beta}\left(\eta\frac{d_{t+1}}{\sqrt{\hat{v}_{t+1}}+\epsilon}\right)-\frac{\beta}{1-\beta}\left( \eta\frac{d_{t}}{\sqrt{\hat{v}_{t}}+\epsilon}\right)\\
= \frac{\eta}{1-\beta}\frac{1}{\sqrt{\hat{v}_{t+1}}+\epsilon}\left(\left(1-\beta\right)\Delta_t+\beta d_t\right)-\frac{\beta}{1-\beta}\eta \frac{1}{\sqrt{\hat{v}_{t}}+\epsilon}d_t\\
= \eta\frac{\Delta_t}{\sqrt{\hat{v}_{t+1}}+\epsilon}-\eta\frac{\beta}{1-\beta}\left(\frac{1}{\sqrt{\hat{v}_{t}}+\epsilon}-\frac{1}{\sqrt{\hat{v}_{t+1}}+\epsilon}\right)d_t
\end{gathered}\nonumber
\end{equation}

Since $f$ is $L$-smooth, taking conditional expectation with respect to all randomness prior to step $t$, we have

\begin{equation}
\begin{gathered}
\mathbb{E}\left[f(z_{t+1})\right] \leq f(z_t)+\mathbb{E}\left[\left\langle \nabla f(z_t),z_{t+1}-z_t \right\rangle\right]+\frac{L}{2}\mathbb{E}\left[\left\| z_{t+1}-z_t\right\|^2\right] \\
\leq f(z_t)+ \mathbb{E}\left[\left\langle \nabla f(z_t), \eta\frac{\Delta_t}{\sqrt{\hat{v}_{t+1}}+\epsilon}-\eta\frac{\beta}{1-\beta}\left(\frac{1}{\sqrt{\hat{v}_{t}}+\epsilon}-\frac{1}{\sqrt{\hat{v}_{t+1}}+\epsilon}\right)d_t \right\rangle\right]\\
+\frac{L\eta^2}{2}\mathbb{E}\left[\left\| \frac{\Delta_t}{\sqrt{\hat{v}_{t+1}}+\epsilon}-\frac{\beta}{1-\beta}\left(\frac{1}{\sqrt{\hat{v}_{t}}+\epsilon}-\frac{1}{\sqrt{\hat{v}_{t+1}}+\epsilon}\right)d_t\right\|^2\right]\\
\leq f(z_t)+ \underbrace{\mathbb{E}\left[\left\langle \nabla f(z_t)-\nabla f(x_t), \eta\frac{\Delta_t}{\sqrt{\hat{v}_{t+1}}+\epsilon}\right\rangle\right]}_{A_1} + \underbrace{\mathbb{E}\left[\left\langle \nabla f(x_t), \eta\frac{\Delta_t}{\sqrt{\hat{v}_{t+1}}+\epsilon}\right\rangle\right]}_{A_2} \\
+\underbrace{\mathbb{E}\left[\left\langle \nabla f(z_t),-\eta\frac{\beta}{1-\beta}\left(\frac{1}{\sqrt{\hat{v}_{t}}+\epsilon}-\frac{1}{\sqrt{\hat{v}_{t+1}}+\epsilon}\right)d_t \right \rangle \right]}_{A_3} \\
+ \underbrace{\frac{L\eta^2}{2}\mathbb{E}\left[\left\|  \frac{\Delta_t}{\sqrt{\hat{v}_{t+1}}+\epsilon}-\frac{\beta}{1-\beta}\left(\frac{1}{\sqrt{\hat{v}_{t}}+\epsilon}-\frac{1}{\sqrt{\hat{v}_{t+1}}+\epsilon}\right)d_t\right\|^2\right]}_{A_4} 
\end{gathered}\nonumber
\end{equation}

\textbf{Bounding} $A_1$:
\begin{equation}
\begin{gathered}
A_1 =\mathbb{E}\left[\left\langle \nabla f(z_t)-\nabla f(x_t), \eta\frac{\Delta_t}{\sqrt{\hat{v}_{t+1}}+\epsilon}\right\rangle\right]\\
\underset{(i)}{\leq} \mathbb{E}\left[\left\| \nabla f(z_t)-\nabla f(x_t)\right\| \cdot \left\| \eta\frac{\Delta_t}{\sqrt{\hat{v}_{t+1}}+\epsilon}\right\|\right]\\
\underset{(ii)}{\leq} L \mathbb{E}\left[\left\| z_t - x_t \right\| \cdot \left\| \eta\frac{\Delta_t}{\sqrt{\hat{v}_{t+1}}+\epsilon}\right\|\right] 
\underset{(iii)}{\leq} \frac{\eta^2 L}{2}\mathbb{E}\left[\left\| \frac{\beta}{1-\beta} \frac{d_t}{\sqrt{\hat{v}_{t}}+\epsilon} \right\|^2\right] + \frac{\eta^2 L}{2}\mathbb{E}\left[\left\| \frac{\Delta_t}{\sqrt{\hat{v}_{t+1}}+\epsilon} \right\|^2 \right]\\
\underset{(iv)}{\leq} \frac{\eta^2 L}{2\epsilon^2}\left(\frac{\beta}{1-\beta}\right)^2 \mathbb{E}\left[\left\| d_t \right\|^2\right] + \frac{\eta^2 L}{2\epsilon^2} \mathbb{E}\left[\left\| \Delta_t \right\|^2\right]
\end{gathered}\nonumber
\end{equation}

where $(i)$ holds by applying Cauchy-Schwarz inequality, $(ii)$ holds as $f$ is $L$-smooth, $(iii)$ follows from Young’s inequality and the definition of $z_t$, and $(iv)$ holds as $\sqrt{\hat{v}_{t}}+\epsilon\ge\epsilon$.

\textbf{Bounding} $A_2$:
\begin{equation}
\begin{gathered}
A_2 =\mathbb{E}\left[\left\langle \nabla f(x_t), \eta\frac{\Delta_t}{\sqrt{\hat{v}_{t+1}}+\epsilon}\right\rangle\right]\\
=\underbrace{\mathbb{E}\left[\left\langle \nabla f(x_t), \eta\frac{\Delta_t}{\sqrt{\hat{v}_{t+1}}+\epsilon}  - \eta\frac{\Delta_t}{\sqrt{\hat{v}_{t}}+\epsilon}\right\rangle\right]}_{A_5} + \underbrace{\mathbb{E}
\left[\left\langle \nabla f(x_t),  \eta\frac{\Delta_t}{\sqrt{\hat{v}_{t}}+\epsilon}\right\rangle\right]}_{A_6}
\end{gathered}\nonumber
\end{equation}

We could bound $A_5$ as follows,

\begin{equation}
\begin{gathered}
A_5 =\mathbb{E}\left[\left\langle \nabla f(x_t), \eta\frac{\Delta_t}{\sqrt{\hat{v}_{t+1}}+\epsilon}  - \eta\frac{\Delta_t}{\sqrt{\hat{v}_{t}}+\epsilon}\right\rangle\right]\\
\underset{(i)}{\leq} \eta\left\|\nabla f(x_t)\right\|\mathbb{E}\left[\left\| \frac{1}{\sqrt{\hat{v}_{t+1}}+\epsilon}  -  \frac{1}{\sqrt{\hat{v}_{t}}+\epsilon}\right\| \cdot \left\|\Delta_t\right\|\right]\\
\leq \eta\left\|\nabla f(x_t)\right\|\mathbb{E}\left[\left\| \frac{\left(\sqrt{\hat{v}_{t+1}}-\sqrt{\hat{v}_{t}}\right) \left(\sqrt{\hat{v}_{t+1}}+\sqrt{\hat{v}_{t}}\right)}{\left(\sqrt{\hat{v}_{t+1}}+\epsilon\right)\left(\sqrt{\hat{v}_{t}}+\epsilon\right)\left(\sqrt{\hat{v}_{t+1}}+\sqrt{\hat{v}_{t}}\right)} \right\| \cdot \left\|\Delta_t\right\|\right]\\
\underset{(ii)}{\leq} \eta\left\|\nabla f(x_t)\right\|\mathbb{E}\left[\left\| \frac{\left\|\Delta_t\right\|^2}{ \left(2\epsilon\right)^2 \left\|\Delta_t\right\|  } \right\| \cdot \left\|\Delta_t\right\|\right]\\
\underset{(iii)}{\leq} \eta \cdot G \cdot  \frac{1}{4\epsilon^2} \mathbb{E}\left[ \left\|\Delta_t\right\|^2 \right] 
\end{gathered}\nonumber
\end{equation}
where $(i)$ holds due to Cauchy-Schwarz inequality, $(ii)$ holds from the definition of $\hat{v}_{t+1}$ (\& $\hat{v}_{t}$) and the fact that $\hat{v}_{t}$ is non-decreasing and $\hat{v}_{t} \ge \dots \ge \hat{v}_{0}\ge\epsilon^2$, $(iii)$ holds due to bounded gradient assumption.

We could bound $A_6$ as follows,
\begin{equation}
\begin{gathered}
A_6 = \mathbb{E}\left[\left\langle \nabla f(x_t),  \eta\frac{\Delta_t}{\sqrt{\hat{v}_{t}}+\epsilon}\right\rangle\right]\\
= \eta\mathbb{E}\left[\left\langle \frac{\nabla f(x_t)}{\sqrt{\hat{v}_{t}}+\epsilon}, \Delta_t + \eta_l \nabla f(x_t) - \eta_l \nabla f(x_t) \right\rangle\right] = -\eta\eta_l\mathbb{E}\left[ \left\|\frac{\nabla f(x_t)}{\sqrt{\sqrt{\hat{v}_{t}}+\epsilon}}\right\|^2\right] \\
+ \underbrace{\eta\eta_l\mathbb{E}\left[\left\langle \frac{\nabla f(x_t)}{\sqrt{\sqrt{\hat{v}_{t}}+\epsilon}}, \frac{1}{\sqrt{\sqrt{\hat{v}_{t}}+\epsilon}} \left(\nabla f(x_t)-\frac{1}{m}\sum_{i\in\mathcal{S}_t}\frac{1}{K_{t,i}}\sum_{k=0}^{K_{t,i}-1}g_{t-\tau_{t,i},k}^i\right) \right\rangle\right]}_{A_7}
\end{gathered}\nonumber
\end{equation}

We could bound $A_7$ as follows,

\begin{equation}
\begin{gathered}
A_7 = \eta\eta_l\mathbb{E}\left[\left\langle \frac{\nabla f(x_t)}{\sqrt{\sqrt{\hat{v}_{t}}+\epsilon}}, \frac{1}{\sqrt{\sqrt{\hat{v}_{t}}+\epsilon}} \left(\nabla f(x_t)-\frac{1}{m}\sum_{i\in\mathcal{S}_t}\frac{1}{K_{t,i}}\sum_{k=0}^{K_{t,i}-1}g_{t-\tau_{t,i},k}^i\right) \right\rangle\right]\\
\underset{(i)}{=} \eta \eta_l \mathbb{E}\left[\left\langle \frac{\nabla f(x_t)}{\sqrt{\sqrt{\hat{v}_{t}}+\epsilon}}, \frac{1}{\sqrt{\sqrt{\hat{v}_{t}}+\epsilon}} \left(\nabla f(x_t)-\frac{1}{n}\sum_{i=1}^n\frac{1}{K_{t,i}}\sum_{k=0}^{K_{t,i}-1}\nabla f_i(x_{t-\tau_{t,i},k}^i) \right)\right\rangle\right]
\end{gathered}\nonumber
\end{equation}
where $(i)$ holds as we take conditional expectation with respect to all randomness prior to step $t$.

We further have,

\begin{equation}
\begin{gathered}
A_7 = \frac{\eta\eta_l}{2}\mathbb{E}\left[ \left\|\frac{\nabla f(x_t)}{\sqrt{\sqrt{\hat{v}_{t}}+\epsilon}}\right\|^2\right] - \frac{\eta\eta_l}{2} \mathbb{E}\left[\left\| \frac{1}{\sqrt{\sqrt{\hat{v}_{t}}+\epsilon}} \frac{1}{n}\sum_{i=1}^n\frac{1}{K_{t,i}}\sum_{k=0}^{K_{t,i}-1}\nabla f_i(x_{t-\tau_{t,i},k}^i) \right\|^2\right] \\
+ \frac{\eta\eta_l}{2} \mathbb{E}\left[\left\| \frac{1}{\sqrt{\sqrt{\hat{v}_{t}}+\epsilon}} \left(\nabla f(x_t)-\frac{1}{n}\sum_{i=1}^n\frac{1}{K_{t,i}}\sum_{k=0}^{K_{t,i}-1}\nabla f_i(x_{t-\tau_{t,i},k}^i)\right) \right\|^2\right]\\
\underset{(i)}{\leq} \frac{\eta\eta_l}{2}\mathbb{E}\left[ \left\|\frac{\nabla f(x_t)}{\sqrt{\sqrt{\hat{v}_{t}}+\epsilon}}\right\|^2\right] - \frac{\eta\eta_l }{2(\sqrt{T}\eta_l  G+\epsilon)}  \mathbb{E}\left[\left\| \frac{1}{n}\sum_{i=1}^n\frac{1}{K_{t,i}}\sum_{k=0}^{K_{t,i}-1}\nabla f_i(x_{t-\tau_{t,i},k}^i) \right\|^2\right]\\
+ \underbrace{\frac{\eta\eta_l}{2} \mathbb{E}\left[\left\| \frac{1}{\sqrt{\sqrt{\hat{v}_{t}}+\epsilon}} \left(\nabla f(x_t)-\frac{1}{n}\sum_{i=1}^n\frac{1}{K_{t,i}}\sum_{k=0}^{K_{t,i}-1}\nabla f_i(x_{t-\tau_{t,i},k}^i)\right) \right\|^2\right]}_{A_8}
\end{gathered}\nonumber
\end{equation}
where $(i)$ holds as $\hat{v}_{t}\leq T \eta_l^2 G^2$.

We further bound $A_8$ as follows,

\begin{equation}
\begin{gathered}
A_8 = \frac{\eta\eta_l}{2} \mathbb{E}\left[\left\| \frac{1}{\sqrt{\sqrt{\hat{v}_{t}}+\epsilon}} \left(\nabla f(x_t)-\frac{1}{n}\sum_{i=1}^n\frac{1}{K_{t,i}}\sum_{k=0}^{K_{t,i}-1}\nabla f_i(x_{t-\tau_{t,i},k}^i)\right) \right\|^2\right] 
\leq \frac{\eta\eta_l}{2\epsilon} \cdot \\ 
\mathbb{E}\left[\left\|   \left( \nabla f(x_t) - \frac{1}{n}\sum_{i=1}^n\nabla f_i(x_{t-\tau_{t,i}})\right) 
+ \left(\frac{1}{n}\sum_{i=1}^n\nabla f_i(x_{t-\tau_{t,i}}) - \frac{1}{n}\sum_{i=1}^n\frac{1}{K_{t,i}}\sum_{k=0}^{K_{t,i}-1}\nabla f_i(x_{t-\tau_{t,i},k}^i)\right)  \right\|^2\right]\\
\underset{(i)}{\leq} \frac{\eta\eta_l}{\epsilon}\mathbb{E}\left[\left\| \frac{1}{n}\sum_{i=1}^n\nabla f_i(x_t) - \frac{1}{n}\sum_{i=1}^n\nabla f_i(x_{t-\tau_{t,i}}) \right\|^2\right] \\
+ \frac{\eta\eta_l}{\epsilon}\mathbb{E}\left[\left\| \frac{1}{n}\sum_{i=1}^n\nabla f_i(x_{t-\tau_{t,i}}) - \frac{1}{n}\sum_{i=1}^n\frac{1}{K_{t,i}}\sum_{k=0}^{K_{t,i}-1}\nabla f_i(x_{t-\tau_{t,i},k}^i) \right\|^2\right] \\
\underset{(ii)}{\leq} \frac{\eta \eta_l L^2}{\epsilon} \frac{1}{n}\sum_{i=1}^n \mathbb{E}\left[\left\|  x_t  -  x_{t-\tau_{t,i}}  \right\|^2\right] 
+ \frac{\eta \eta_l L^2}{\epsilon} \frac{1}{n}\sum_{i=1}^n \frac{1}{K_{t,i}}\sum_{k=0}^{K_{t,i}-1} \mathbb{E}\left[\left\| x_{t-\tau_{t,i}} - x_{t-\tau_{t,i},k}^i \right\|^2\right]  
\end{gathered}\nonumber
\end{equation}

where $(i)$ holds due to $\left\|\sum_{i=1}^n x_i\right\|^2 \leq n \sum_{i=1}^n\left\| x_i \right\|^2$, $(ii)$ holds due to $L$-smoothness of $f_i$.

When $\eta_l \leq \frac{1}{8 K L}$, we have,
\begin{equation}
\begin{gathered}
\mathbb{E}\left[\left\| x_{t-\tau_{t,i}}-x_{t-\tau_{t,i},k}^i  \right\|^2\right] \leq 5 K_{t,i}\eta_l^2\left(\sigma_l^2+6K_{t,i}\sigma_g^2\right)+30K_{t,i}^2\eta_l^2 \mathbb{E}\left[\left\| \nabla f(x_{t-\tau_{t,i}})\right\|^2\right]
\end{gathered}\nonumber
\end{equation}

We can further bound $\frac{1}{n}\sum_{i=1}^n \mathbb{E}\left[\left\| x_t - x_{t-\tau_{t,i}} \right\|^2\right]$
\begin{equation}
\begin{gathered}
\frac{1}{n}\sum_{i=1}^n \mathbb{E}\left[\left\| x_t - x_{t-\tau_{t,i}} \right\|^2\right]
\leq  \mathbb{E}\left[\left\| x_t - x_{t-\tau_{t,u}} \right\|^2\right] = \mathbb{E}\left[\left\| \sum_{k=t-\tau_{t,u}}^{t-1} \left(x_{k+1}-x_k\right) \right\|^2\right] \\
\leq \mathbb{E}\left[\left\| \sum_{k=t-\tau_{t,u}}^{t-1} \eta y_k \right\|^2\right] \leq \tau \eta^2 \sum_{k=t-\tau_{t,u}}^{t-1}\mathbb{E}\left[\left\|  y_k \right\|^2\right]
\end{gathered}\nonumber
\end{equation}
where $y_k\triangleq x_{k+1}-x_k=\eta\frac{d_{k+1}}{\sqrt{\hat{v}_{k+1}}+\epsilon}$.

\begin{equation}
\begin{gathered}
A_8 
\leq \frac{\eta^3 \eta_l L^2 \tau}{\epsilon}  \sum_{k=t-\tau_{t,u}}^{t-1}\mathbb{E}\left[\left\|  y_k \right\|^2\right]
+ \frac{5 \eta \eta_l^3 L^2 \bar{K}_t}{\epsilon}\sigma_l^2 +  \frac{30 \eta \eta_l^3 L^2 \hat{K}_t^2}{\epsilon}\sigma_g^2 \\
+ \frac{30 \eta \eta_l^3 L^2}{\epsilon} \frac{1}{n}\sum_{i=1}^n K_{t,i}^2 \mathbb{E}\left[\left\| \nabla f(x_{t-\tau_{t,i}})\right\|^2\right] 
\end{gathered}\nonumber
\end{equation}

Merging all pieces together,
\begin{equation}
\begin{gathered}
A_2 =\mathbb{E}\left[\left\langle \nabla f(x_t), \eta\frac{\Delta_t}{\sqrt{\hat{v}_{t+1}}+\epsilon}\right\rangle\right]\\
\leq  \frac{\eta G}{4\epsilon^2} \mathbb{E}\left[ \left\|\Delta_t\right\|^2 \right]  - \frac{\eta\eta_l}{2}\mathbb{E}\left[ \left\|\frac{\nabla f(x_t)}{\sqrt{\sqrt{\hat{v}_{t}}+\epsilon}}\right\|^2\right] - \frac{\eta\eta_l }{2(\sqrt{T}\eta_l  G+\epsilon)}  \mathbb{E}\left[\left\| \frac{1}{n}\sum_{i=1}^n\frac{1}{K_{t,i}}\sum_{k=0}^{K_{t,i}-1}\nabla f_i(x_{t-\tau_{t,i},k}^i) \right\|^2\right]\\
+ \frac{\eta^3 \eta_l L^2 \tau}{\epsilon}  \sum_{k=t-\tau_{t,u}}^{t-1}\mathbb{E}\left[\left\|  y_k \right\|^2\right]
+ \frac{5 \eta \eta_l^3 L^2 \bar{K}_t}{\epsilon}\sigma_l^2 +  \frac{30 \eta \eta_l^3 L^2 \hat{K}_t^2}{\epsilon}\sigma_g^2 
+ \frac{30 \eta \eta_l^3 L^2}{\epsilon} \frac{1}{n}\sum_{i=1}^n K_{t,i}^2 \mathbb{E}\left[\left\| \nabla f(x_{t-\tau_{t,i}})\right\|^2\right] 
\end{gathered}\nonumber
\end{equation}

\textbf{Bounding} $A_3$:
\begin{equation}
\begin{gathered}
A_3 =\mathbb{E}\left[\left\langle \nabla f(z_t),-\eta\frac{\beta}{1-\beta}\left(\frac{1}{\sqrt{\hat{v}_{t}}+\epsilon}-\frac{1}{\sqrt{\hat{v}_{t+1}}+\epsilon}\right)d_t \right\rangle\right]\\
\leq \mathbb{E}\left[\left\langle \nabla f(z_t) - \nabla f(x_t) + \nabla f(x_t) ,-\eta\frac{\beta}{1-\beta}\left(\frac{1}{\sqrt{\hat{v}_{t}}+\epsilon}-\frac{1}{\sqrt{\hat{v}_{t+1}}+\epsilon}\right)d_t \right\rangle\right]\\
\underset{(i)}{\leq} \eta \mathbb{E}\left[\left\| \nabla f(z_t) - \nabla f(x_t) \right\| \cdot \left\| \frac{\beta}{1-\beta}\left(\frac{1}{\sqrt{\hat{v}_{t}}+\epsilon}-\frac{1}{\sqrt{\hat{v}_{t+1}}+\epsilon}\right)d_t \right\|\right] \\
+ \eta \mathbb{E}\left[\left\|  \nabla f(x_t) \right\| \cdot \left\| \frac{\beta}{1-\beta}\left(\frac{1}{\sqrt{\hat{v}_{t}}+\epsilon}-\frac{1}{\sqrt{\hat{v}_{t+1}}+\epsilon}\right)d_t \right\|\right]
\end{gathered}\nonumber
\end{equation}
where $(i)$ holds due to Cauchy-Schwarz inequality.

And we would further have,

\begin{equation}
\begin{gathered}
A_3 \underset{(i)}{\leq} \eta L \mathbb{E}\left[\left\| \frac{\beta}{1-\beta}\left( \eta\frac{d_{t}}{\sqrt{\hat{v}_{t}}+\epsilon}\right) \right\| \cdot \left\| \frac{\beta}{1-\beta}\left(\frac{1}{\sqrt{\hat{v}_{t}}+\epsilon}-\frac{1}{\sqrt{\hat{v}_{t+1}}+\epsilon}\right)d_t \right\|\right] \\
+ \eta G \mathbb{E}\left[  \left\| \frac{\beta}{1-\beta}\left(\frac{1}{\sqrt{\hat{v}_{t}}+\epsilon}-\frac{1}{\sqrt{\hat{v}_{t+1}}+\epsilon}\right)d_t \right\|\right]\\
\underset{(ii)}{\leq} \eta^2\eta_l^2G^2L\left(\frac{\beta}{1-\beta}\right)^2 \mathbb{E}\left[\left\| \frac{1}{\sqrt{\hat{v}_{t}}+\epsilon} \right\| \cdot \left\| \frac{1}{\sqrt{\hat{v}_{t}}+\epsilon}-\frac{1}{\sqrt{\hat{v}_{t+1}}+\epsilon} \right\|\right] \\
+ \eta \eta_l G^2 \frac{\beta}{1-\beta} \mathbb{E}\left[\left\| \frac{1}{\sqrt{\hat{v}_{t}}+\epsilon}-\frac{1}{\sqrt{\hat{v}_{t+1}}+\epsilon} \right\|\right]\\
\underset{(iii)}{\leq} \eta^2\eta_l^2G^2L\left(\frac{\beta}{1-\beta}\right)^2 \frac{1}{2\epsilon} \mathbb{E}\left[\left\| \frac{1}{\sqrt{\hat{v}_{t}}+\epsilon}-\frac{1}{\sqrt{\hat{v}_{t+1}}+\epsilon} \right\|\right] + \eta \eta_l G^2 \frac{\beta}{1-\beta} \mathbb{E}\left[\left\| \frac{1}{\sqrt{\hat{v}_{t}}+\epsilon}-\frac{1}{\sqrt{\hat{v}_{t+1}}+\epsilon} \right\|\right] \\
\le \left( \frac{\eta^2\eta_l^2G^2L}{2\epsilon} 
\left(\frac{\beta}{1-\beta}\right)^2 + \eta \eta_l G^2 \frac{\beta}{1-\beta} \right) \mathbb{E}\left[\left\| \frac{1}{\sqrt{\hat{v}_{t}}+\epsilon}-\frac{1}{\sqrt{\hat{v}_{t+1}}+\epsilon} \right\|\right]
\end{gathered}\nonumber
\end{equation}
where $(i)$, $(ii)$ and $(iii)$ hold as we have $\left\| d_t \right\| \leq \eta_l G$, $\left\| \Delta_t \right\| \leq \eta_l G$, $\left\| \frac{1}{\sqrt{\hat{v}_{t}}+\epsilon} \right\| \leq \frac{1}{2\epsilon}$, $\mathbb{E}\left[\left\| \frac{1}{\sqrt{\hat{v}_{t}}+\epsilon}-\frac{1}{\sqrt{\hat{v}_{t+1}}+\epsilon} \right\|\right]\leq \frac{1}{4\epsilon^2}\mathbb{E}\left[\left\| \Delta_t\right\| \right]$.

\textbf{Bounding} $A_4$:
\begin{equation}
\begin{gathered}
A_4 = \frac{L\eta^2}{2}\mathbb{E}\left[\left\| \frac{\Delta_t}{\sqrt{\hat{v}_{t+1}}+\epsilon}-\frac{\beta}{1-\beta}\left(\frac{1}{\sqrt{\hat{v}_{t}}+\epsilon}-\frac{1}{\sqrt{\hat{v}_{t+1}}+\epsilon}\right)d_t\right\|^2\right]\\
\underset{(i)}{\leq}  L\eta^2 \mathbb{E}\left[\left\| \frac{\Delta_t}{\sqrt{\hat{v}_{t+1}}+\epsilon} \right\|^2\right] + L\eta^2 \mathbb{E}\left[\left\|  \frac{\beta}{1-\beta}\left(\frac{1}{\sqrt{\hat{v}_{t}}+\epsilon}-\frac{1}{\sqrt{\hat{v}_{t+1}}+\epsilon}\right)d_t\right\|^2\right]\\
\underset{(ii)}{\leq} \frac{L\eta^2}{4\epsilon^2}\mathbb{E}\left[\left\| \Delta_t\right\|^2\right]+L\eta^2\eta_l^2 G^2 \left(\frac{\beta}{1-\beta}\right)^2\mathbb{E}\left[\left\| \frac{1}{\sqrt{\hat{v}_{t}}+\epsilon}-\frac{1}{\sqrt{\hat{v}_{t+1}}+\epsilon}  \right\|^2\right]
\end{gathered}\nonumber
\end{equation}

where $(i)$ holds due to Cauchy-Schwarz inequality, $(ii)$ holds as $\left\| d_t \right\| \leq \eta_l G$.

\textbf{Bounding} $\sum_{t=0}^{T-1}\mathbb{E}\left[\left\| \frac{1}{\sqrt{\hat{v}_{t}}+\epsilon}-\frac{1}{\sqrt{\hat{v}_{t+1}}+\epsilon}  \right\|^2\right]$:
\begin{equation}
\begin{gathered}
\sum_{t=0}^{T-1}\mathbb{E}\left[\left\| \frac{1}{\sqrt{\hat{v}_{t}}+\epsilon}-\frac{1}{\sqrt{\hat{v}_{t+1}}+\epsilon}  \right\|^2\right] 
\underset{(i)}{=}  \sum_{t=0}^{T-1} \sum_{j=1}^d \mathbb{E}\left[ \left( \left[\frac{1}{\sqrt{\hat{v}_{t}}+\epsilon}\right]_j - \left[\frac{1}{\sqrt{\hat{v}_{t+1}}+\epsilon}\right]_j \right)  ^2\right]\\
\leq \sum_{t=0}^{T-1} \sum_{j=1}^d \mathbb{E}\left[  \left[\frac{1}{\sqrt{\hat{v}_{t}}+\epsilon}\right]_j^2 - \left[\frac{1}{\sqrt{\hat{v}_{t+1}}+\epsilon}\right]_j^2\right] \\
\underset{(ii)}{\leq} \mathbb{E}\left[\left\| \frac{1}{\sqrt{\hat{v}_{0}}+\epsilon} \right\|^2\right] - \mathbb{E}\left[\left\| \frac{1}{\sqrt{\hat{v}_{T}}+\epsilon} \right\|^2\right] \underset{(iii)}{\leq} \frac{d}{4\epsilon^2}
\end{gathered}\nonumber
\end{equation}

where $\left[\frac{1}{\sqrt{\hat{v}_{t}}+\epsilon}\right]_j$ denotes the $j$-th element of vector $\frac{1}{\sqrt{\hat{v}_{t}}+\epsilon}$ and $(i)$ follows from the definition of $\left\|\cdot\right\|^2$. $(ii)$ holds as $\hat{v}_{t}$ is non-decreasing, $(iii)$ holds as every element of $\frac{1}{\sqrt{\hat{v}_{0}}+\epsilon}$ is smaller than $\frac{1}{2\epsilon}$.

\textbf{Bounding} $\sum_{t=0}^{T-1}\mathbb{E}\left[\left\| \frac{1}{\sqrt{\hat{v}_{t}}+\epsilon}-\frac{1}{\sqrt{\hat{v}_{t+1}}+\epsilon}  \right\|\right]$:
\begin{equation}
\begin{gathered}
\sum_{t=0}^{T-1}\mathbb{E}\left[\left\| \frac{1}{\sqrt{\hat{v}_{t}}+\epsilon}-\frac{1}{\sqrt{\hat{v}_{t+1}}+\epsilon}  \right\| \right] 
=  \sum_{t=0}^{T-1} \sum_{j=1}^d \mathbb{E}\left[ \left| \left[\frac{1}{\sqrt{\hat{v}_{t}}+\epsilon}\right]_j - \left[\frac{1}{\sqrt{\hat{v}_{t+1}}+\epsilon}\right]_j \right|  \right]\\
\leq \sum_{t=0}^{T-1} \sum_{j=1}^d \mathbb{E}\left[  \left| \left[\frac{1}{\sqrt{\hat{v}_{t}}+\epsilon}\right]_j \right| - \left| \left[\frac{1}{\sqrt{\hat{v}_{t+1}}+\epsilon}\right]_j \right| \right] \\
\underset{(i)}{\leq} \mathbb{E}\left[\left\| \frac{1}{\sqrt{\hat{v}_{0}}+\epsilon} \right\| \right] - \mathbb{E}\left[\left\| \frac{1}{\sqrt{\hat{v}_{T}}+\epsilon} \right\| \right] \leq \frac{d}{2\epsilon}
\end{gathered}\nonumber
\end{equation}
Similarly, $(i)$ holds due to $\hat{v}_{t}$ is non-decreasing.

\textbf{Bounding} $\mathbb{E}\left[\left\| \Delta_t\right\|^2\right]$:

\begin{equation}
\begin{gathered}
\mathbb{E}\left[\left\| \Delta_t\right\|^2\right] = \mathbb{E}\left[\left\| \frac{1}{m}\sum_{i\in\mathcal{S}_t}\Delta^i_{t-\tau_{t,i}} \right\|^2\right]\\
=\mathbb{E}\left[\left\| \frac{1}{m}\sum_{i\in\mathcal{S}_t}  \frac{\eta_l}{K_{t,i}} \sum_{k=0}^{K_{t,i}-1} g_{t-\tau_{t,i},k}^i \right\|^2\right] \\
\leq \mathbb{E}\left[\left\| \frac{1}{m}\sum_{i\in\mathcal{S}_t} \frac{\eta_l}{K_{t,i}} \sum_{k=0}^{K_{t,i}-1} \left\{ g_{t-\tau_{t,i},k}^i - \nabla f_i(x_{t-\tau_{t,i},k}^i) + \nabla f_i(x_{t-\tau_{t,i},k}^i)\right\} \right\|^2\right] \\
\leq \mathbb{E}\left[\left\| \frac{1}{m}\sum_{i\in\mathcal{S}_t} \frac{\eta_l}{K_{t,i}} \sum_{k=0}^{K_{t,i}-1} \left\{ g_{t-\tau_{t,i},k}^i - \nabla f_i(x_{t-\tau_{t,i},k}^i) \right\}  \right\|^2\right] + \mathbb{E}\left[\left\| \frac{1}{m}\sum_{i\in\mathcal{S}_t} \frac{\eta_l}{K_{t,i}} \sum_{k=0}^{K_{t,i}-1}  \nabla f_i(x_{t-\tau_{t,i},k}^i) \right\|^2\right]\\
\leq \frac{1}{m^2}\sum_{i\in\mathcal{S}_t}\frac{\eta_l^2}{K^2_{t,i}}\sum_{k=0}^{K_{t,i}-1}\sigma_l^2 + \mathbb{E}\left[\left\| \frac{1}{m}\sum_{i\in\mathcal{S}_t} \frac{\eta_l}{K_{t,i}} \sum_{k=0}^{K_{t,i}-1}  \nabla f_i(x_{t-\tau_{t,i},k}^i) \right\|^2\right]\\
\leq \frac{\eta_l^2}{m}\frac{1}{K_t}\sigma^2_l + \frac{\eta_l^2}{m^2}  \mathbb{E}\left[\left\| \sum_{i=1}^n \mathbbm{1}\left\{i\in\mathcal{S}_t\right\} \frac{1}{K_{t,i}} \sum_{k=0}^{K_{t,i}-1}  \nabla f_i(x_{t-\tau_{t,i},k}^i) \right\|^2\right]
\end{gathered}\nonumber
\end{equation}
where $\frac{1}{K_t}=\sum_{i\in\mathcal{S}_t}\frac{1}{K_{t,i}}$.

\textbf{Bounding} $\sum_{t=0}^{T-1}\mathbb{E}\left[\left\| d_t\right\|^2\right]$:

We could verify:
\begin{equation}
\begin{gathered}
 d_t = \sum_{p=0}^t a_{t,p}\Delta_p,         \quad  \text{where} \quad  a_{t,p}=\left(1-\beta_p\right)\prod_{q=p+1}^t\beta_q
\end{gathered}\nonumber
\end{equation}
We further get,
\begin{equation}
\begin{gathered}
\mathbb{E}\left[\left\| d_t\right\|^2\right]=\mathbb{E}\left[\left\| \sum_{p=0}^t a_{t,p}\Delta_p\right\|^2\right]\\
\leq \sum_{e=1}^d \mathbb{E}\left[\sum_{p=0}^t a_{t,p}\Delta_{p,e}\right]^2 
\leq \sum_{e=1}^d \mathbb{E}\left[ \left(\sum_{p=0}^t a_{t,p}\right) \left(\sum_{p=0}^t a_{t,p}\Delta_{p,e}^2\right) \right] \leq \left(1- \prod_{q=0}^t \beta_q\right)\sum_{p=0}^t a_{t,p}\mathbb{E}\left[\left\| \Delta_p \right\|^2\right]\\
\le \left(1- \prod_{q=0}^t \beta_q\right)\sum_{p=0}^t a_{t,p}\left\{ \frac{\eta_l^2}{m}\frac{1}{K_t}\sigma^2_l + \frac{\eta_l^2}{m^2}  \mathbb{E}\left[\left\| \sum_{i=1}^n\mathbbm{1}\left\{i\in\mathcal{S}_p\right\} \frac{1}{K_{p,i}} \sum_{k=0}^{K_{p,i}-1}  \nabla f_i(x_{p-\tau_{p,i},k}^i) \right\|^2\right] \right\}\\
\leq \frac{\eta_l^2}{m}\frac{1}{K_t}\sigma^2_l + \frac{\eta_l^2}{m^2} \sum_{p=0}^t a_{t,p} \cdot \mathbb{E}\left[\left\| \sum_{i=1}^n\mathbbm{1}\left\{i\in\mathcal{S}_p\right\} \frac{1}{K_{p,i}} \sum_{k=0}^{K_{p,i}-1}  \nabla f_i(x_{p-\tau_{p,i},k}^i) \right\|^2\right] 
\end{gathered}\nonumber
\end{equation}

Summing over $t\in\{0,1,\dots,T-1\}$,
\begin{equation}
\begin{gathered}
\sum_{t=0}^{T-1}\mathbb{E}\left[\left\| d_t\right\|^2\right] \leq
\frac{\eta_l^2}{m}\sum_{t=0}^{T-1}\frac{1}{K_t}\sigma^2_l + \frac{\eta_l^2}{m^2}\sum_{t=0}^{T-1} \sum_{p=0}^t a_{t,p} \cdot \mathbb{E}\left[\left\| \sum_{i=1}^n\mathbbm{1}\left\{i\in\mathcal{S}_p\right\} \frac{1}{K_{p,i}} \sum_{k=0}^{K_{p,i}-1}  \nabla f_i(x_{p-\tau_{p,i},k}^i) \right\|^2\right] \\
\leq \frac{\eta_l^2}{m}\sum_{t=0}^{T-1}\frac{1}{K_t}\sigma^2_l + \frac{\eta_l^2}{m^2} \sum_{p=0}^{T-1}\left(\sum_{t=p}^{T-1}a_{t,p}\right) \mathbb{E}\left[\left\| \sum_{i=1}^n\mathbbm{1}\left\{i\in\mathcal{S}_p\right\} \frac{1}{K_{p,i}} \sum_{k=0}^{K_{p,i}-1}  \nabla f_i(x_{p-\tau_{p,i},k}^i) \right\|^2\right]\\
\leq \frac{\eta_l^2}{m}\sum_{t=0}^{T-1}\frac{1}{K_t}\sigma^2_l + \frac{\eta_l^2}{m^2} \sum_{t=0}^{T-1} \mathbb{E}\left[\left\| \sum_{i=1}^n\mathbbm{1}\left\{i\in\mathcal{S}_t\right\} \frac{1}{K_{t,i}} \sum_{k=0}^{K_{t,i}-1}  \nabla f_i(x_{t-\tau_{t,i},k}^i) \right\|^2\right]\\
\end{gathered}\nonumber
\end{equation}

\textbf{Bounding} $\sum_{t=0}^{T-1}\mathbb{E}[\lVert y_t\rVert^2]$:

\begin{equation}
\begin{gathered}
\sum_{t=0}^{T-1}\mathbb{E}\left[\left\| y_t\right\|^2\right] = \sum_{t=0}^{T-1}\mathbb{E}\left[\left\|  \frac{d_{t+1}}{\sqrt{\hat{v}_{k+1}}+\epsilon} \right\|^2\right] \leq \frac{1}{4\epsilon^2}\sum_{t=0}^{T-1}\mathbb{E}\left[\left\| d_{t} \right\|^2\right] \\ 
\leq \frac{\eta_l^2 }{4\epsilon^2 m} \sum_{t=0}^{T-1}\frac{1}{K_t} \sigma^2_l + \frac{ \eta_l^2}{4 \epsilon^2 m^2}   \sum_{t=0}^{T-1} \mathbb{E}\left[\left\| \sum_{i=1}^n\mathbbm{1}\left\{i\in\mathcal{S}_t\right\} \frac{1}{K_{t,i}} \sum_{k=0}^{K_{t,i}-1}  \nabla f_i(x_{t-\tau_{t,i},k}^i) \right\|^2\right]\\
\end{gathered}\nonumber
\end{equation}

Merging $A_1$, $A_2$, $A_3$, and $A_4$ together, we get,

\begin{equation}
\begin{gathered}
\mathbb{E}\left[f(z_{t+1})\right]\leq f(z_t) + A_1 + A_2 + A_3 + A_4\\
\leq f(z_t) + \frac{\eta^2 L}{2\epsilon^2}\left(\frac{\beta}{1-\beta}\right)^2 \mathbb{E}\left[\left\| d_t \right\|^2\right] + \frac{\eta^2 L}{2\epsilon^2} \mathbb{E}\left[\left\| \Delta_t \right\|^2\right] \\
+ \frac{\eta G}{4\epsilon^2} \mathbb{E}\left[ \left\|\Delta_t\right\|^2 \right] - \frac{\eta\eta_l}{2}\mathbb{E}\left[ \left\|\frac{\nabla f(x_t)}{\sqrt{\sqrt{\hat{v}_{t}}+\epsilon}}\right\|^2\right] - \frac{\eta\eta_l }{2(\sqrt{T}\eta_l  G+\epsilon)}  \mathbb{E}\left[\left\| \frac{1}{n}\sum_{i=1}^n\frac{1}{K_{t,i}}\sum_{k=0}^{K_{t,i}-1} \nabla f_i(x_{t-\tau_{t,i},k}^i)  \right\|^2\right]\\
+ \frac{\eta^3 \eta_l L^2 \tau}{\epsilon}  \sum_{k=t-\tau_{t,u}}^{t-1}\mathbb{E}\left[\left\|  y_k \right\|^2\right]
+ \frac{5 \eta \eta_l^3 L^2 \bar{K}_t}{\epsilon}\sigma_l^2 +  \frac{30 \eta \eta_l^3 L^2 \hat{K}_t^2}{\epsilon}\sigma_g^2 
+ \frac{30 \eta \eta_l^3 L^2}{\epsilon} \frac{1}{n}\sum_{i=1}^n K_{t,i}^2 \mathbb{E}\left[\left\| \nabla f(x_{t-\tau_{t,i}})\right\|^2\right] \\
+ \left( \frac{\eta^2\eta_l^2G^2L}{2\epsilon} \left(\frac{\beta}{1-\beta}\right)^2 + \eta \eta_l G^2 \frac{\beta}{1-\beta} \right) \mathbb{E}\left[\left\| \frac{1}{\sqrt{\hat{v}_{t}}+\epsilon}-\frac{1}{\sqrt{\hat{v}_{t+1}}+\epsilon} \right\|\right]\\
+ \frac{L\eta^2}{4\epsilon^2}\mathbb{E}\left[\left\| \Delta_t\right\|^2\right]+L\eta^2\eta_l^2G^2\left(\frac{\beta}{1-\beta}\right)^2\mathbb{E}\left[\left\| \frac{1}{\sqrt{\hat{v}_{t}}+\epsilon}-\frac{1}{\sqrt{\hat{v}_{t+1}}+\epsilon}  \right\|^2\right]
\end{gathered}\nonumber
\end{equation}

Sum over $t\in\{0,1,\dots,T-1\}$,

\begin{equation}
\begin{gathered}
\mathbb{E}\left[f(z_{T})\right] \leq f(z_0) + \\
 \frac{\eta^2 L}{2\epsilon^2}\left(\frac{\beta}{1-\beta}\right)^2 \sum_{t=0}^{T-1}\mathbb{E}\left[\left\| d_t \right\|^2\right] + \frac{\eta^2 L}{2\epsilon^2} \sum_{t=0}^{T-1}\mathbb{E}\left[\left\| \Delta_t \right\|^2\right] 
+ \frac{\eta G}{4\epsilon^2} \sum_{t=0}^{T-1}\mathbb{E}\left[ \left\|\Delta_t\right\|^2 \right] - \frac{\eta\eta_l}{2}\sum_{t=0}^{T-1}\mathbb{E}\left[ \left\|\frac{\nabla f(x_t)}{\sqrt{\sqrt{\hat{v}_{t}}+\epsilon}}\right\|^2\right]\\
- \frac{\eta\eta_l }{2(\sqrt{T}\eta_l  G+\epsilon)}  \sum_{t=0}^{T-1}\mathbb{E}\left[\left\| \frac{1}{n}\sum_{i=1}^n\frac{1}{K_{t,i}}\sum_{k=0}^{K_{t,i}-1} \nabla f_i(x_{t-\tau_{t,i},k}^i)  \right\|^2\right]+ \frac{\eta^3 \eta_l L^2 \tau}{\epsilon}  \sum_{t=0}^{T-1}\sum_{k=t-\tau_{t,u}}^{t-1}\mathbb{E}\left[\left\|  y_k \right\|^2\right]\\
+ \frac{5 \eta \eta_l^3 L^2 \sum_{t=0}^{T-1}\bar{K}_t}{\epsilon}\sigma_l^2 +  \frac{30 \eta \eta_l^3 L^2 \sum_{t=0}^{T-1}\hat{K}_t^2}{\epsilon}\sigma_g^2 
+ \frac{30 \eta \eta_l^3 L^2}{\epsilon} \sum_{t=0}^{T-1}\frac{1}{n}\sum_{i=1}^n K_{t,i}^2 \mathbb{E}\left[\left\| \nabla f(x_{t-\tau_{t,i}})\right\|^2\right] \\
+ \left( \frac{\eta^2\eta_l^2G^2L}{2\epsilon} \left(\frac{\beta}{1-\beta}\right)^2 + \eta \eta_l G^2 \frac{\beta}{1-\beta} \right) \sum_{t=0}^{T-1}\mathbb{E}\left[\left\| \frac{1}{\sqrt{\hat{v}_{t}}+\epsilon}-\frac{1}{\sqrt{\hat{v}_{t+1}}+\epsilon} \right\|\right]\\
+ \frac{L\eta^2}{4\epsilon^2} \sum_{t=0}^{T-1}\mathbb{E}\left[\left\| \Delta_t\right\|^2\right]+L\eta^2\eta_l^2G^2\left(\frac{\beta}{1-\beta}\right)^2\sum_{t=0}^{T-1}\mathbb{E}\left[\left\| \frac{1}{\sqrt{\hat{v}_{t}}+\epsilon}-\frac{1}{\sqrt{\hat{v}_{t+1}}+\epsilon}  \right\|^2\right]
\end{gathered}\nonumber
\end{equation}

Plug in the bounds for $\sum_{t=0}^{T-1} \mathbb{E}\left[\left\| \frac{1}{\sqrt{\hat{v}_{t}}+\epsilon}-\frac{1}{\sqrt{\hat{v}_{t+1}}+\epsilon} \right\|\right]$ and $\sum_{t=0}^{T-1}\mathbb{E}\left[\left\| \frac{1}{\sqrt{\hat{v}_{t}}+\epsilon}-\frac{1}{\sqrt{\hat{v}_{t+1}}+\epsilon}  \right\|^2\right]$, we have,

\begin{equation}
\begin{gathered}
\mathbb{E}\left[f(z_{T})\right] \leq f(z_0) + \\
 \frac{\eta^2 L}{2\epsilon^2}\left(\frac{\beta}{1-\beta}\right)^2 \sum_{t=0}^{T-1}\mathbb{E}\left[\left\| d_t \right\|^2\right] + 
 \left(\frac{ 3 L \eta^2 + \eta G}{4\epsilon^2} \right)\sum_{t=0}^{T-1}\mathbb{E}\left[\left\| \Delta_t \right\|^2\right] 
- \frac{\eta\eta_l}{2}\sum_{t=0}^{T-1}\mathbb{E}\left[ \left\|\frac{\nabla f(x_t)}{\sqrt{\sqrt{\hat{v}_{t}}+\epsilon}}\right\|^2\right]\\
- \frac{\eta\eta_l }{2(\sqrt{T}\eta_l  G+\epsilon)}  \sum_{t=0}^{T-1}\mathbb{E}\left[\left\| \frac{1}{n}\sum_{i=1}^n\frac{1}{K_{t,i}}\sum_{k=0}^{K_{t,i}-1} \nabla f_i(x_{t-\tau_{t,i},k}^i)  \right\|^2\right]+ \frac{\eta^3 \eta_l L^2 \tau}{\epsilon}  \sum_{t=0}^{T-1}\sum_{k=t-\tau_{t,u}}^{t-1}\mathbb{E}\left[\left\|  y_k \right\|^2\right]\\
+ \frac{5 \eta \eta_l^3 L^2 \sum_{t=0}^{T-1}\bar{K}_t}{\epsilon}\sigma_l^2 +  \frac{30 \eta \eta_l^3 L^2 \sum_{t=0}^{T-1}\hat{K}_t^2}{\epsilon}\sigma_g^2 
+ \frac{30 \eta \eta_l^3 L^2}{\epsilon} \sum_{t=0}^{T-1}\frac{1}{n}\sum_{i=1}^n K_{t,i}^2 \mathbb{E}\left[\left\| \nabla f(x_{t-\tau_{t,i}})\right\|^2\right] \\
+ \left( \frac{\eta^2\eta_l^2G^2L}{2\epsilon} \left(\frac{\beta}{1-\beta}\right)^2 + \eta \eta_l G^2 \frac{\beta}{1-\beta} \right) \frac{d}{2\epsilon}
+L\eta^2\eta_l^2G^2\left(\frac{\beta}{1-\beta}\right)^2 \frac{d}{4\epsilon^2}
\end{gathered}\nonumber
\end{equation}

We could further get,

\begin{equation}
\begin{gathered}
\mathbb{E}\left[f(z_{T})\right] \leq f(z_0) - \frac{\eta\eta_l}{2}\sum_{t=0}^{T-1}\mathbb{E}\left[ \left\|\frac{\nabla f(x_t)}{\sqrt{\sqrt{\hat{v}_{t}}+\epsilon}}\right\|^2\right] + \frac{30 \eta \eta_l^3 L^2}{\epsilon} \sum_{t=0}^{T-1}\frac{1}{n}\sum_{i=1}^n K_{t,i}^2 \mathbb{E}\left[\left\| \nabla f(x_{t-\tau_{t,i}})\right\|^2\right] \\
+ \left( \left( \frac{\eta^2 L}{2\epsilon^2}\left(\frac{\beta}{1-\beta}\right)^2 + \frac{ 3 L \eta^2 + \eta G}{4\epsilon^2}  \right) \frac{\eta_l^2}{m^2} + \frac{\eta^3 \eta_l L^2 \tau^2}{\epsilon} \frac{\eta_l^2}{4\epsilon^2m^2} \right)\cdot \\
\sum_{t=0}^{T-1} \mathbb{E}\left[\left\| \sum_{i=1}^n\mathbbm{1}\left\{i\in\mathcal{S}_t\right\} \frac{1}{K_{t,i}} \sum_{k=0}^{K_{t,i}-1}  \nabla f_i(x_{t-\tau_{t,i},k}^i) \right\|^2\right]\\
- \frac{\eta\eta_l }{2(\sqrt{T}\eta_l  G+\epsilon)n^2}  \sum_{t=0}^{T-1}\mathbb{E}\left[\left\| \sum_{i=1}^n\frac{1}{K_{t,i}}\sum_{k=0}^{K_{t,i}-1} \nabla f_i(x_{t-\tau_{t,i},k}^i)  \right\|^2\right] +  \frac{30 \eta \eta_l^3 L^2 \sum_{t=0}^{T-1}\hat{K}_t^2}{\epsilon}\sigma_g^2 \\
+ \left( \frac{\eta^2\eta_l^2G^2L}{2\epsilon} \left(\frac{\beta}{1-\beta}\right)^2 + \eta \eta_l G^2 \frac{\beta}{1-\beta} \right) \frac{d}{2\epsilon} + L\eta^2\eta_l^2G^2\left(\frac{\beta}{1-\beta}\right)^2 \frac{d}{4\epsilon^2}\\
+ \left( \frac{5 \eta \eta_l^3 L^2 \sum_{t=0}^{T-1}\bar{K}_t}{\epsilon} +   \frac{2 \eta^2 \eta_l^2 L C_\beta^2 + 3 L \eta^2 \eta_l^2 + \eta \eta_l^2 G}{4 m \epsilon^2}    \sum_{t=0}^{T-1}\frac{1}{K_t} + \frac{\eta^3 \eta_l^3 L^2 \tau^2}{4 \epsilon^3 m}  \sum_{t=0}^{T-1}\frac{1}{K_t} \right)\sigma_l^2 
\end{gathered}\nonumber
\end{equation}

We have the following inequality,

\begin{equation}
\begin{gathered}
\mathbb{E}\left[\left\| \sum_{i=1}^n \frac{1}{K_{t,i}} \sum_{k=0}^{K_{t,i}-1} \nabla f_i(x_{t-\tau_{t,i},k}^i) \right\|^2\right]=n\sum_{i=1}^n\mathbb{E}\left[\left\| \frac{1}{K_{t,i}} \sum_{k=0}^{K_{t,i}-1} \nabla f_i(x_{t-\tau_{t,i},k}^i) \right\|^2\right] \\- \frac{1}{2}\sum_{i\neq j} \left\| \frac{1}{K_{t,i}} \sum_{k=0}^{K_{t,i}-1} \nabla f_i(x_{t-\tau_{t,i},k}^i) - \frac{1}{K_{t,j}} \sum_{k=0}^{K_{t,j}-1} \nabla f_j(x_{t-\tau_{t,j},k}^j) \right\|^2
\end{gathered}\nonumber
\end{equation}

We also know the following,

\begin{equation}
\begin{gathered}
\mathbb{E}\left[\left\| \sum_{i=1}^n\mathbbm{1}\left\{i\in\mathcal{S}_t\right\} \frac{1}{K_{t,i}} \sum_{k=0}^{K_{t,i}-1} \nabla f_i(x_{t-\tau_{t,i},k}^i) \right\|^2\right]
= \sum_{i=1}^n \mathbb{P}\left\{i\in\mathcal{S}_t\right\} \left\| \frac{1}{K_{t,i}} \sum_{k=0}^{K_{t,i}-1} \nabla f_i(x_{t-\tau_{t,i},k}^i)\right\|^2 \\+ \sum_{i\neq j}\mathbb{P}\left\{i,j\in\mathcal{S}_t\right\}\left\langle \frac{1}{K_{t,i}} \sum_{k=0}^{K_{t,i}-1} \nabla f_i(x_{t-\tau_{t,i},k}^i),\frac{1}{K_{t,j}} \sum_{k=0}^{K_{t,j}-1} \nabla f_j(x_{t-\tau_{t,j},k}^j)\right\rangle\\
= \frac{m}{n}\sum_{i=1}^n  \left\| \frac{1}{K_{t,i}} \sum_{k=0}^{K_{t,i}-1} \nabla f_i(x_{t-\tau_{t,i},k}^i)\right\|^2 + \frac{m\left(m-1\right)}{n\left(n-1\right)} \sum_{i\neq j} \left\langle \frac{1}{K_{t,i}} \sum_{k=0}^{K_{t,i}-1} \nabla f_i(x_{t-\tau_{t,i},k}^i),\frac{1}{K_{t,j}} \sum_{k=0}^{K_{t,j}-1} \nabla f_j(x_{t-\tau_{t,j},k}^j)\right\rangle\\
= \frac{m^2}{n}\sum_{i=1}^n  \left\| \frac{1}{K_{t,i}} \sum_{k=0}^{K_{t,i}-1} \nabla f_i(x_{t-\tau_{t,i},k}^i)\right\|^2 \\
- \frac{m\left(m-1\right)}{2 n\left(n-1\right)} \sum_{i\neq j} \left\| \frac{1}{K_{t,i}} \sum_{k=0}^{K_{t,i}-1} \nabla f_i(x_{t-\tau_{t,i},k}^i) - \frac{1}{K_{t,j}} \sum_{k=0}^{K_{t,j}-1} \nabla f_j(x_{t-\tau_{t,j},k}^j) \right\|^2
\end{gathered}\nonumber
\end{equation}

Merging these two pieces together, and assume $\eta_l\le \frac{\epsilon}{\sqrt{T}G}$, we have $- \frac{\eta\eta_l }{2(\sqrt{T}\eta_l  G+\epsilon)} \le - \frac{\eta\eta_l }{4\epsilon }$. And if $H_1 \eta_l^2+H_2\eta_l\le\epsilon^2$, where $H_1\triangleq 2 \eta^2 L^2 \tau^2$, $H_2\triangleq 4 \eta L C_\beta^2 + 6 \eta L \epsilon + 2 G \epsilon$, $C_\beta=\frac{\beta}{1-\beta}$, we have the following inequality,

\begin{equation}
\begin{gathered}
 \left( \frac{\eta^2 L}{2\epsilon^2}\left(\frac{\beta}{1-\beta}\right)^2 + \frac{ 3 L \eta^2 + \eta G}{4\epsilon^2}  \right) \frac{\eta_l^2}{m^2} + \frac{\eta^3 \eta_l L^2 \tau^2}{\epsilon} \frac{\eta_l^2}{4\epsilon^2m^2} \leq \frac{\eta\eta_l}{8 \epsilon m^2}
\end{gathered}\nonumber
\end{equation}

thus, we have,

\begin{equation}
\begin{gathered}
 \left( \left( \frac{\eta^2 L}{2\epsilon^2}\left(\frac{\beta}{1-\beta}\right)^2 + \frac{ 3 L \eta^2 + \eta G}{4\epsilon^2}  \right) \frac{\eta_l^2}{m^2} + \frac{\eta^3 \eta_l L^2 \tau^2}{\epsilon} \frac{ \eta_l^2}{4\epsilon^2m^2} \right)\cdot \\
\sum_{t=0}^{T-1} \mathbb{E}\left[\left\| \sum_{i=1}^n\mathbbm{1}\left\{i\in\mathcal{S}_t\right\} \frac{1}{K_{t,i}} \sum_{k=0}^{K_{t,i}-1}  \nabla f_i(x_{t-\tau_{t,i},k}^i) \right\|^2\right]\\
- \frac{\eta\eta_l }{2(\sqrt{T}\eta_l  G+\epsilon)n^2}  \sum_{t=0}^{T-1}\mathbb{E}\left[\left\| \sum_{i=1}^n\frac{1}{K_{t,i}}\sum_{k=0}^{K_{t,i}-1} \nabla f_i(x_{t-\tau_{t,i},k}^i)  \right\|^2\right]   \\
\leq \frac{\eta\eta_l}{ 8 \epsilon m^2} \cdot  
\sum_{t=0}^{T-1} \mathbb{E}\left[\left\| \sum_{i=1}^n\mathbbm{1}\left\{i\in\mathcal{S}_t\right\} \frac{1}{K_{t,i}} \sum_{k=0}^{K_{t,i}-1}  \nabla f_i(x_{t-\tau_{t,i},k}^i) \right\|^2\right]\\
- \frac{\eta\eta_l }{4 \epsilon n^2}  \sum_{t=0}^{T-1}\mathbb{E}\left[\left\| \sum_{i=1}^n\frac{1}{K_{t,i}}\sum_{k=0}^{K_{t,i}-1} \nabla f_i(x_{t-\tau_{t,i},k}^i)  \right\|^2\right]\\
= \left( -\frac{\eta\eta_l}{4\epsilon n} + \frac{\eta\eta_l}{8 \epsilon n} \right) \cdot \mathcal{G}_1 
+ \left( \frac{\eta\eta_l}{8 \epsilon n^2} - \frac{\eta\eta_l(m-1)}{16 \epsilon m (n-1) n} \right) \cdot \mathcal{G}_2
\end{gathered}\nonumber
\end{equation}

where,

\begin{equation}
\begin{gathered}
\mathcal{G}_1 \triangleq \sum_{t=0}^{T-1} \sum_{i=1}^n  \left\| \frac{1}{K_{t,i}} \sum_{k=0}^{K_{t,i}-1} \nabla f_i(x_{t-\tau_{t,i},k}^i)\right\|^2\\
\mathcal{G}_2 \triangleq \sum_{t=0}^{T-1} \sum_{i\neq j} \left\| \frac{1}{K_{t,i}} \sum_{k=0}^{K_{t,i}-1} \nabla f_i(x_{t-\tau_{t,i},k}^i) - \frac{1}{K_{t,j}} \sum_{k=0}^{K_{t,j}-1} \nabla f_j(x_{t-\tau_{t,j},k}^j) \right\|^2
\end{gathered}\nonumber 
\end{equation}
 
We can verify $\mathcal{G}_2 \leq 2 \left(n-1\right) \mathcal{G}_1$, thus, we could verify the above term is non-positive.

Plugging back into the original terms,

\begin{equation}
\begin{gathered}
\mathbb{E}\left[f(z_{T})\right] \leq f(z_0) - \frac{\eta\eta_l}{2}\sum_{t=0}^{T-1}\mathbb{E}\left[ \left\|\frac{\nabla f(x_t)}{\sqrt{\sqrt{\hat{v}_{t}}+\epsilon}}\right\|^2\right] + \frac{30 \eta \eta_l^3 L^2}{\epsilon} \sum_{t=0}^{T-1}\frac{1}{n}\sum_{i=1}^n K_{t,i}^2 \mathbb{E}\left[\left\| \nabla f(x_{t-\tau_{t,i}})\right\|^2\right] \\
 +  \frac{30 \eta \eta_l^3 L^2 \sum_{t=0}^{T-1}\hat{K}_t^2}{\epsilon}\sigma_g^2 \\
 + \left( \frac{5 \eta \eta_l^3 L^2 \sum_{t=0}^{T-1}\bar{K}_t}{\epsilon} +   \frac{2 \eta^2 \eta_l^2 L C_\beta^2 + 3 L \eta^2 \eta_l^2 + \eta \eta_l^2 G}{4 m \epsilon^2}    \sum_{t=0}^{T-1}\frac{1}{K_t} + \frac{\eta^3 \eta_l^3 L^2 \tau^2}{4 \epsilon^3 m}  \sum_{t=0}^{T-1}\frac{1}{K_t} \right)\sigma_l^2  \\
+ \left( \frac{\eta^2\eta_l^2G^2L}{2\epsilon} \left(\frac{\beta}{1-\beta}\right)^2 + \eta \eta_l G^2 \frac{\beta}{1-\beta} \right) \frac{d}{2\epsilon} + L\eta^2\eta_l^2G^2\left(\frac{\beta}{1-\beta}\right)^2 \frac{d}{4\epsilon^2}
\end{gathered}\nonumber
\end{equation}

Since $\left\|\hat{v}_{t}\right\| \leq \eta_l^2 G^2 T$ and $\eta_l\leq \frac{\epsilon}{\sqrt{T}G}$, we have,

\begin{equation}
\begin{gathered}
- \frac{\eta\eta_l}{2} \sum_{t=0}^{T-1} \mathbb{E}\left[ \left\|\frac{\nabla f(x_t)}{\sqrt{\sqrt{\hat{v}_{t}}+\epsilon}}\right\|^2\right] \leq - \frac{\eta\eta_l}{4\epsilon} \sum_{t=0}^{T-1} \mathbb{E}\left[ \left\| \nabla f(x_t) \right\|^2\right]
\end{gathered}\nonumber
\end{equation}

Assume $\eta_l \leq \frac{1}{\sqrt{120\epsilon\tau}K_\text{max}L}$, we have

\begin{equation}
\begin{gathered}
- \frac{\eta\eta_l}{2}\sum_{t=0}^{T-1}\mathbb{E}\left[ \left\|\frac{\nabla f(x_t)}{\sqrt{\sqrt{\hat{v}_{t}}+\epsilon}}\right\|^2\right] + \frac{30 \eta \eta_l^3 L^2}{\epsilon} \sum_{t=0}^{T-1}\frac{1}{n}\sum_{i=1}^n K_{t,i}^2 \mathbb{E}\left[\left\| \nabla f(x_{t-\tau_{t,i}})\right\|^2\right]\\
\leq - \frac{\eta\eta_l}{2} \sum_{t=0}^{T-1} \mathbb{E}\left[ \left\|\frac{\nabla f(x_t)}{\sqrt{\sqrt{\hat{v}_{t}}+\epsilon}}\right\|^2\right] + \frac{30 \eta \eta_l^3 L^2}{\epsilon} K_{\text{max}}^2\tau \sum_{t=0}^{T-1} \mathbb{E}\left[\left\| \nabla f(x_{t })\right\|^2\right]\\
\leq -\frac{\eta\eta_l}{8\epsilon}\sum_{t=0}^{T-1} \mathbb{E}\left[\left\| \nabla f(x_{t })\right\|^2\right]
\end{gathered}\nonumber
\end{equation}

We have,

\begin{equation}
\begin{gathered}
\frac{\eta\eta_l}{8\epsilon}\sum_{t=0}^{T-1} \mathbb{E}\left[\left\| \nabla f(x_{t })\right\|^2\right]  \leq f(z_0) -\mathbb{E}\left[f(z_{T})\right]  \\
 +  \frac{30 \eta \eta_l^3 L^2 \sum_{t=0}^{T-1}\hat{K}_t^2}{\epsilon}\sigma_g^2 \\+ \left( \frac{5 \eta \eta_l^3 L^2 \sum_{t=0}^{T-1}\bar{K}_t}{\epsilon} +   \frac{2 \eta^2 \eta_l^2 L C_\beta^2 + 3 L \eta^2 \eta_l^2 + \eta \eta_l^2 G}{4 m \epsilon^2}    \sum_{t=0}^{T-1}\frac{1}{K_t} + \frac{\eta^3 \eta_l^3 L^2 \tau^2}{4 \epsilon^3 m}  \sum_{t=0}^{T-1}\frac{1}{K_t} \right)\sigma_l^2  \\
+ \left( \frac{\eta^2\eta_l^2G^2L}{2\epsilon} \left(\frac{\beta}{1-\beta}\right)^2 + \eta \eta_l G^2 \frac{\beta}{1-\beta} \right) \frac{d}{2\epsilon} + L\eta^2\eta_l^2G^2\left(\frac{\beta}{1-\beta}\right)^2 \frac{d}{4\epsilon^2}
\end{gathered}\nonumber
\end{equation}

Average with respect to $T$, we have,

\begin{equation}
\begin{gathered}
 \frac{1}{T}\sum_{t=0}^{T-1} \mathbb{E}\left[\left\| \nabla f(x_{t })\right\|^2\right]  \leq \frac{8\epsilon}{\eta\eta_l T} \left( f(z_0) -f^\ast \right)  \\
 +  240   \eta_l^2 L^2 \frac{1}{T}\sum_{t=0}^{T-1}\hat{K}_t^2 \sigma_g^2 + \frac{\Phi}{T} \\
 + \left( 40 \eta_l^2 L^2 \frac{1}{T} \sum_{t=0}^{T-1}\bar{K}_t +   \frac{4 \eta  \eta_l  L C_\beta^2 + 6 L \eta  \eta_l + 2  \eta_l G}{ m \epsilon}  \frac{1}{T}  \sum_{t=0}^{T-1}\frac{1}{K_t} + \frac{2 \eta^2 \eta_l^2 L^2 \tau^2}{ \epsilon^2 m}  \frac{1}{T} \sum_{t=0}^{T-1}\frac{1}{K_t} \right)\sigma_l^2  
\end{gathered}\nonumber
\end{equation}

We have the following,

\begin{equation}
\begin{gathered}
\Phi \triangleq    \left( \frac{\eta \eta_l G^2L}{2\epsilon} \left(\frac{\beta}{1-\beta}\right)^2 +   G^2 \frac{\beta}{1-\beta} \right) 4d + L\eta \eta_l G^2\left(\frac{\beta}{1-\beta}\right)^2 \frac{2 d}{ \epsilon }     
\end{gathered}\nonumber
\end{equation}

\end{proof}

\subsection{Proof of Corollary \ref{fedadaptive_free_uniform_arrival_convergence_rate}}

\begin{proof}[Proof of Corollary \ref{fedadaptive_free_uniform_arrival_convergence_rate}]
By setting $\eta_l=\Theta\left(\frac{1}{\sqrt{T}}\right)$, $\eta=\Theta\left(\sqrt{mK}\right)$, we have

\begin{equation}
\begin{gathered}
\Phi = \Theta\left(1\right) + \Theta\left(\sqrt{\frac{mK}{T}}\right)\\
\Phi_l = \Theta\left(\frac{K}{T}\right) +  \Theta\left(\sqrt{\frac{1}{mKT}}\right) + \Theta\left(\frac{\tau^2}{T}\right), \quad
\Phi_g = \Theta\left(\frac{K^2}{T}\right)
\end{gathered}\nonumber
\end{equation}

Plug in the complexity of $\Phi$, $\Phi_l$, and $\Phi_g$ back in, we could get,

\begin{equation}
\begin{gathered}
 \frac{1}{T}\sum_{t=0}^{T-1} \mathbb{E}\left[\left\| \nabla f(x_{t })\right\|^2\right]  \leq \Theta\left(\sqrt{\frac{1}{mKT}}\right) \left( f(z_0) -f^\ast \right)  \\
 +  \Theta\left(\frac{K^2}{T}\right) \sigma_g^2 +  \Theta\left(\frac{1}{T}\right) + \Theta\left(\sqrt{\frac{mK}{T^3}}\right) 
 + \left(\Theta\left(\frac{K}{T}\right) +  \Theta\left(\sqrt{\frac{1}{mKT}}\right) + \Theta\left(\frac{\tau^2}{T}\right)\right) \sigma_l^2  
\end{gathered}\nonumber
\end{equation}

Only keep the dominant terms, we have the convergence rate as,

\begin{equation}
\begin{gathered}
 \frac{1}{T}\sum_{t=0}^{T-1} \mathbb{E}\left[\left\| \nabla f(x_{t })\right\|^2\right]  = \mathcal{O}\left(\sqrt{\frac{1}{mKT}}\right) +  \mathcal{O}\left(\frac{K^2}{T}\right) + \mathcal{O}\left(\frac{\tau^2}{T}\right)
\end{gathered}\nonumber
\end{equation}

\end{proof}

\section{More Experimental Results}
\label{sec:appendix_more_exp}

We leave extra experimental results here. Basically, these are experiments with different concentration parameter $\alpha$, randomness $R$, maximum delay $\tau$. All experiments demonstrate the consistent superiority of our proposed approaches.

Figure \ref{fig:cifar10_resnet_adaptive_delay_10_result_appen}: Training/Testing Curves in Figure \ref{subfig:cifar10_resnet_adaptive_delay_10_test}, i.e. experimental results with $\tau=10$.

\begin{figure*}[h]
\vspace*{-12pt}
\centering
\subfigure[Training Curves]{
\hspace{0pt}
\includegraphics[width=.35\textwidth]{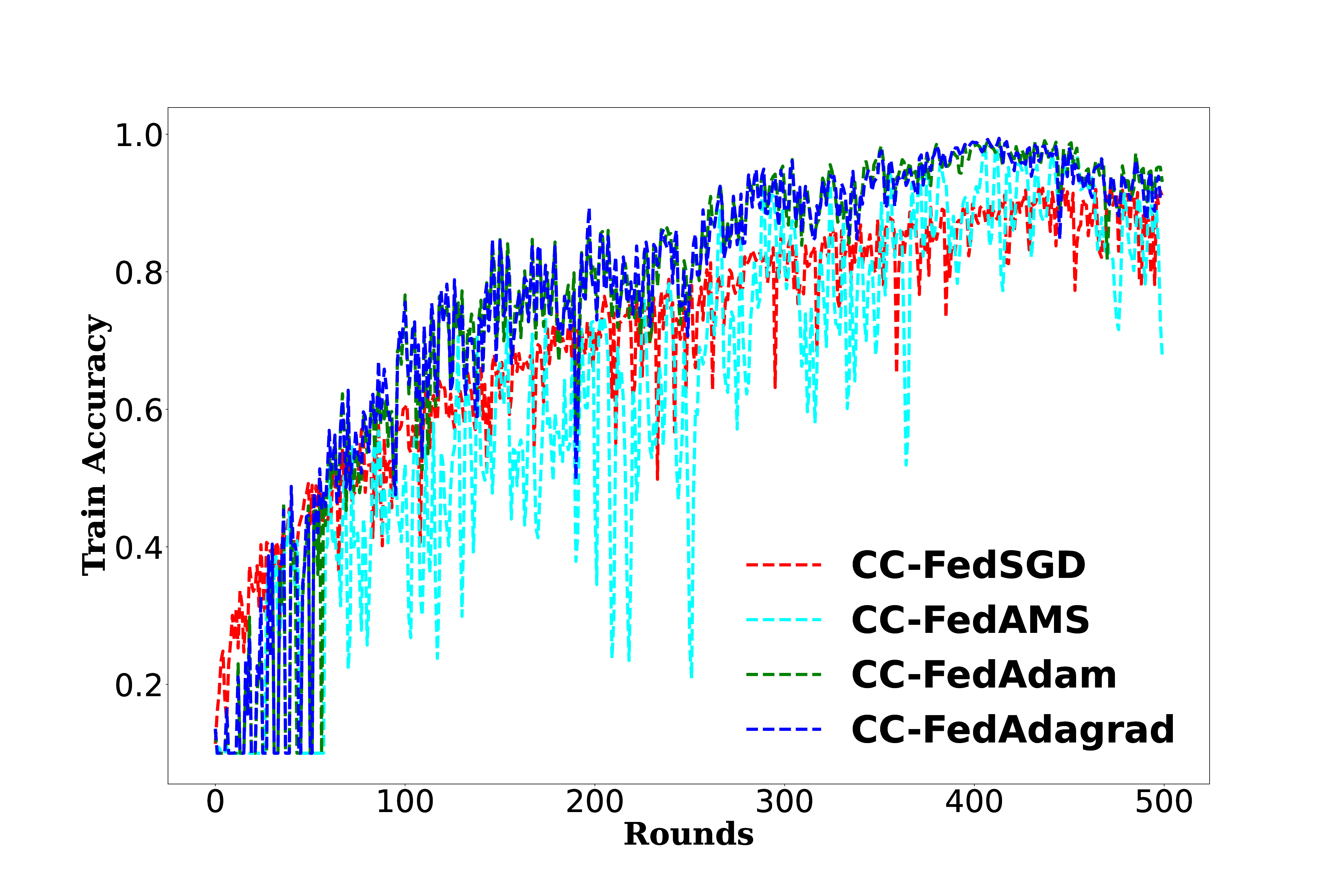}
\label{subfig:cifar10_resnet_adaptive_delay_10_train_appen}
}
\subfigure[Testing Curves]{
\hspace{0pt}
\includegraphics[width=.35\textwidth]{figs/cifar10_resnet_adaptive_delay_10_test.eps}
\label{subfig:cifar10_resnet_adaptive_delay_10_test_appen}
}
\vspace*{-6pt}
\caption{Training and testing curves for various CC-Federated Adaptive Optimizers (ResNet on CIFAR-10) with $\tau=10$.}
\label{fig:cifar10_resnet_adaptive_delay_10_result_appen}
\end{figure*}

Figure \ref{fig:cifar10_resnet_delay_5_randomness_3_niid_05_result_appen}: Training/Testing Curves in Figure \ref{subfig:cifar10_resnet_delay_5_randomness_3_niid_05_test}, i.e. experimental results with $R=3$.

\begin{figure*}[h]
\vspace*{-12pt}
\centering
\subfigure[Training Curves]{
\hspace{0pt}
\includegraphics[width=.35\textwidth]{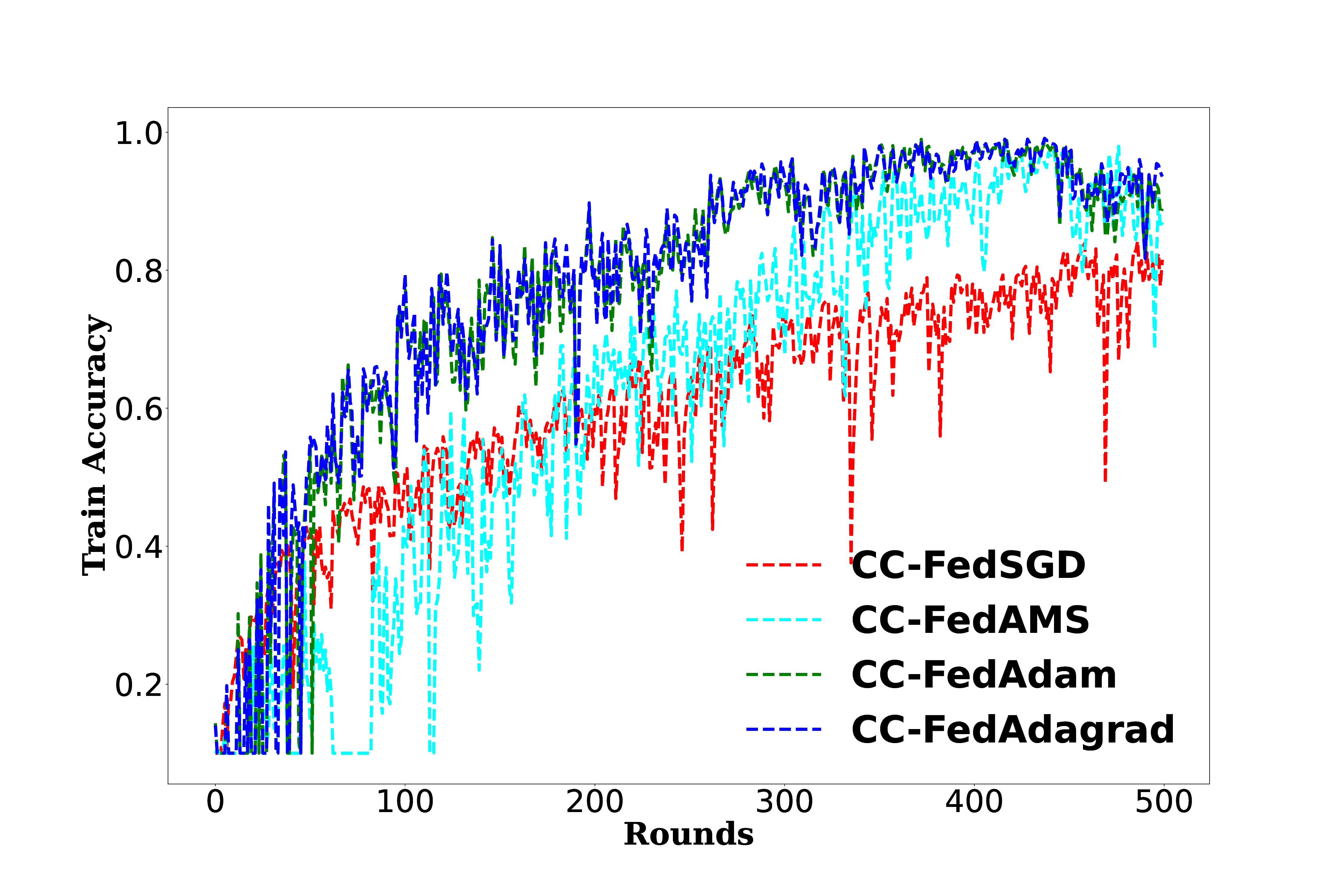}
\label{subfig:cifar10_resnet_delay_5_randomness_3_niid_05_train_appen}
}
\subfigure[Testing Curves]{
\hspace{0pt}
\includegraphics[width=.35\textwidth]{figs/cifar10_resnet_delay_5_randomness_3_niid_05_test.eps}
\label{subfig:cifar10_resnet_delay_5_randomness_3_niid_05_test_appen}
}
\vspace*{-6pt}
\caption{Training and testing curves for various CC-Federated Adaptive Optimizers (ResNet on CIFAR-10) with $R=3$.}
\label{fig:cifar10_resnet_delay_5_randomness_3_niid_05_result_appen}
\end{figure*}

Figure \ref{fig:fmnist_simple_cnn_adaptive_delay_10_niid_03_result_appen}: Training/Testing Curves in Figure \ref{subfig:fmnist_simple_cnn_adaptive_delay_10_niid_03_test}, i.e. experimental results with shallow CNN on FMNIST. Note that the CNN \cite{li2022NIIDBenchmark} we use has the following structure, two 5x5 convolution layers followed by 2x2 max pooling (the first with 6 channels and the second with 16 channels) and two fully connected layers with ReLU activation (the first with 120 units and the second with 84 units).

\begin{figure*}[h]
\vspace*{-12pt}
\centering
\subfigure[Training Curves]{
\hspace{0pt}
\includegraphics[width=.35\textwidth]{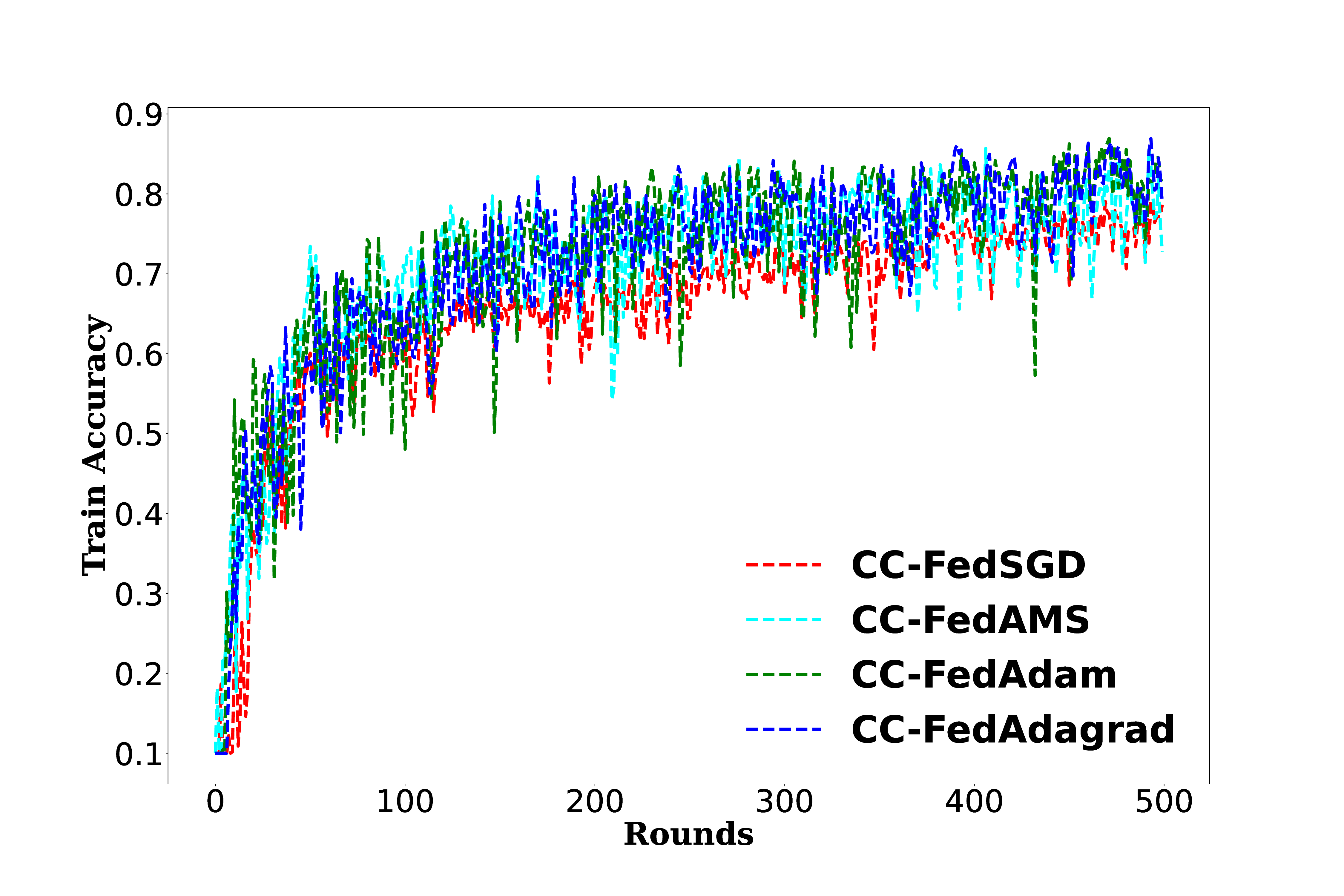}
\label{subfig:fmnist_simple_cnn_adaptive_delay_10_niid_03_train_appen}
}
\subfigure[Testing Curves]{
\hspace{0pt}
\includegraphics[width=.35\textwidth]{figs/fmnist_simple_cnn_adaptive_delay_10_niid_03_test.eps}
\label{subfig:fmnist_simple_cnn_adaptive_delay_10_niid_03_test_appen}
}
\vspace*{-6pt}
\caption{Training and testing curves for various CC-Federated Adaptive Optimizers (shallow CNN on FMNIST).}
\label{fig:fmnist_simple_cnn_adaptive_delay_10_niid_03_result_appen}
\end{figure*}

Figure \ref{fig:cifar10_resnet_delay_5_randomness_2_niid_05_result_appen}, extra experimental results with $R=2$.

\begin{figure*}[h]
\vspace*{-12pt}
\centering
\subfigure[Training Curves]{
\hspace{0pt}
\includegraphics[width=.35\textwidth]{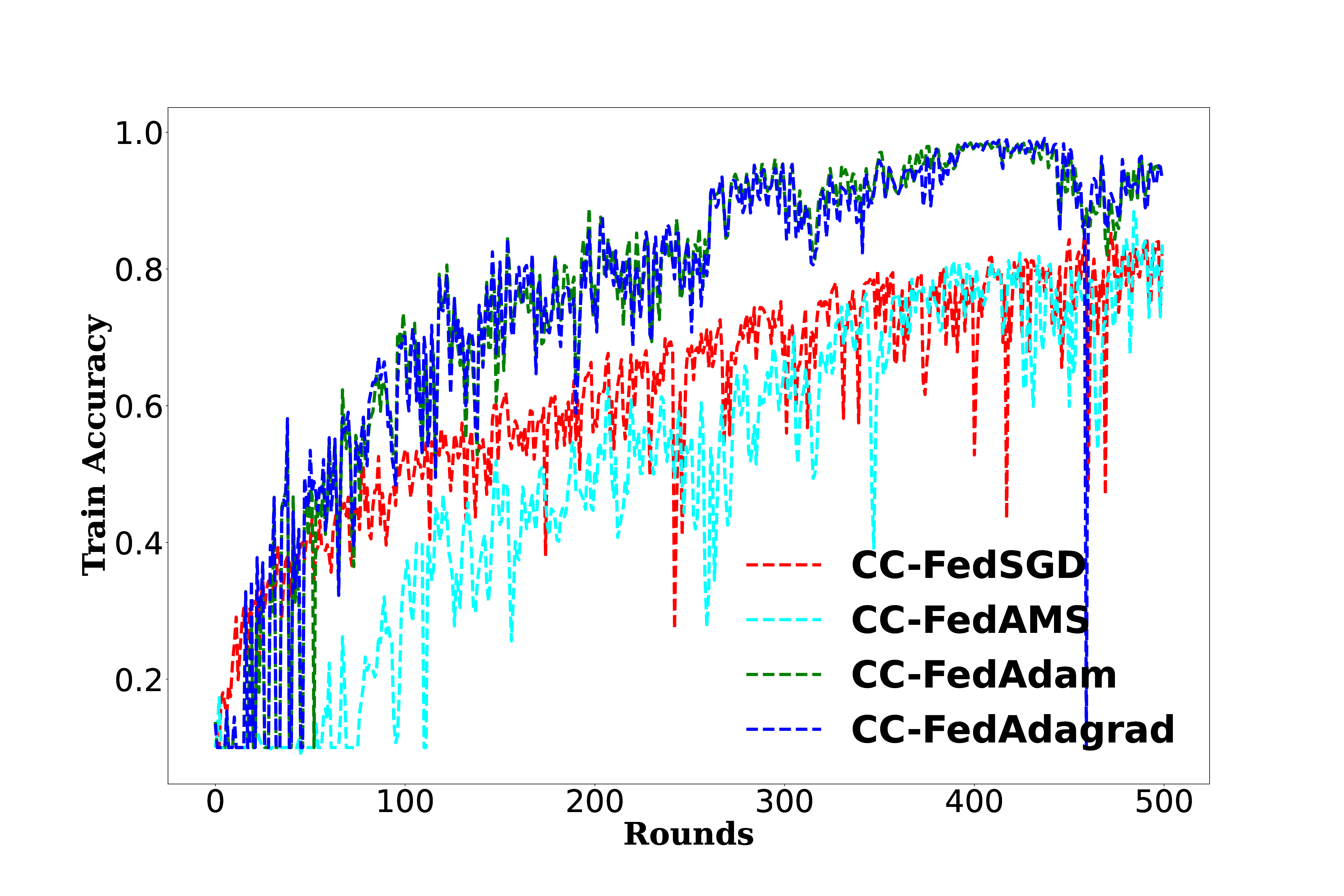}
\label{subfig:cifar10_resnet_delay_5_randomness_2_niid_05_train_appen}
}
\subfigure[Testing Curves]{
\hspace{0pt}
\includegraphics[width=.35\textwidth]{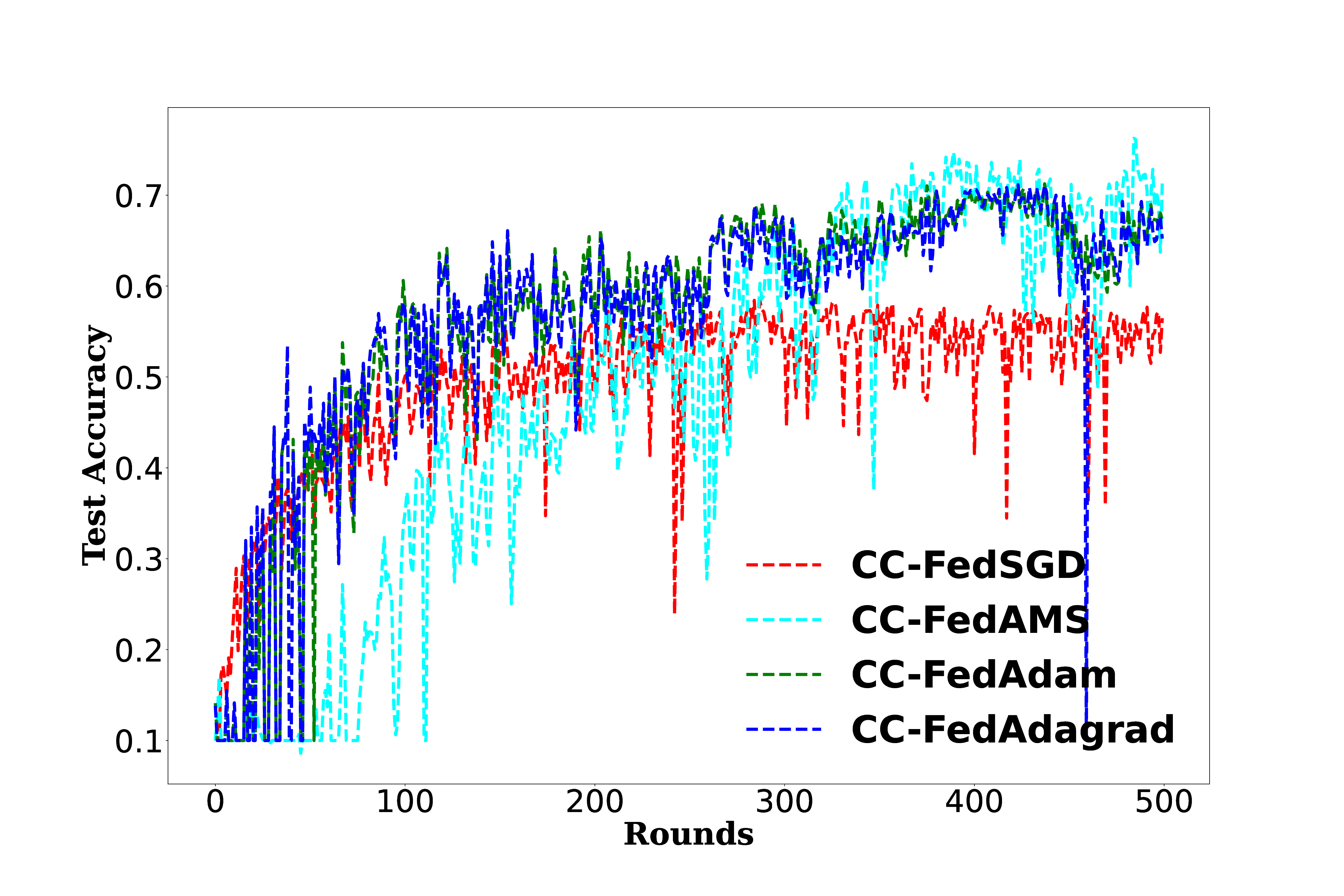}
\label{subfig:cifar10_resnet_delay_5_randomness_2_niid_05_test_appen}
}
\vspace*{-6pt}
\caption{Training and testing curves for various CC-Federated Adaptive Optimizers (ResNet on CIFAR-10) with $R=2$.}
\label{fig:cifar10_resnet_delay_5_randomness_2_niid_05_result_appen}
\end{figure*}

Figure \ref{fig:cifar10_resnet_delay_5_randomness_1_niid_05_result_appen}, extra experimental results with $R=1$.

\begin{figure*}[h]
\vspace*{-12pt}
\centering
\subfigure[Training Curves]{
\hspace{0pt}
\includegraphics[width=.35\textwidth]{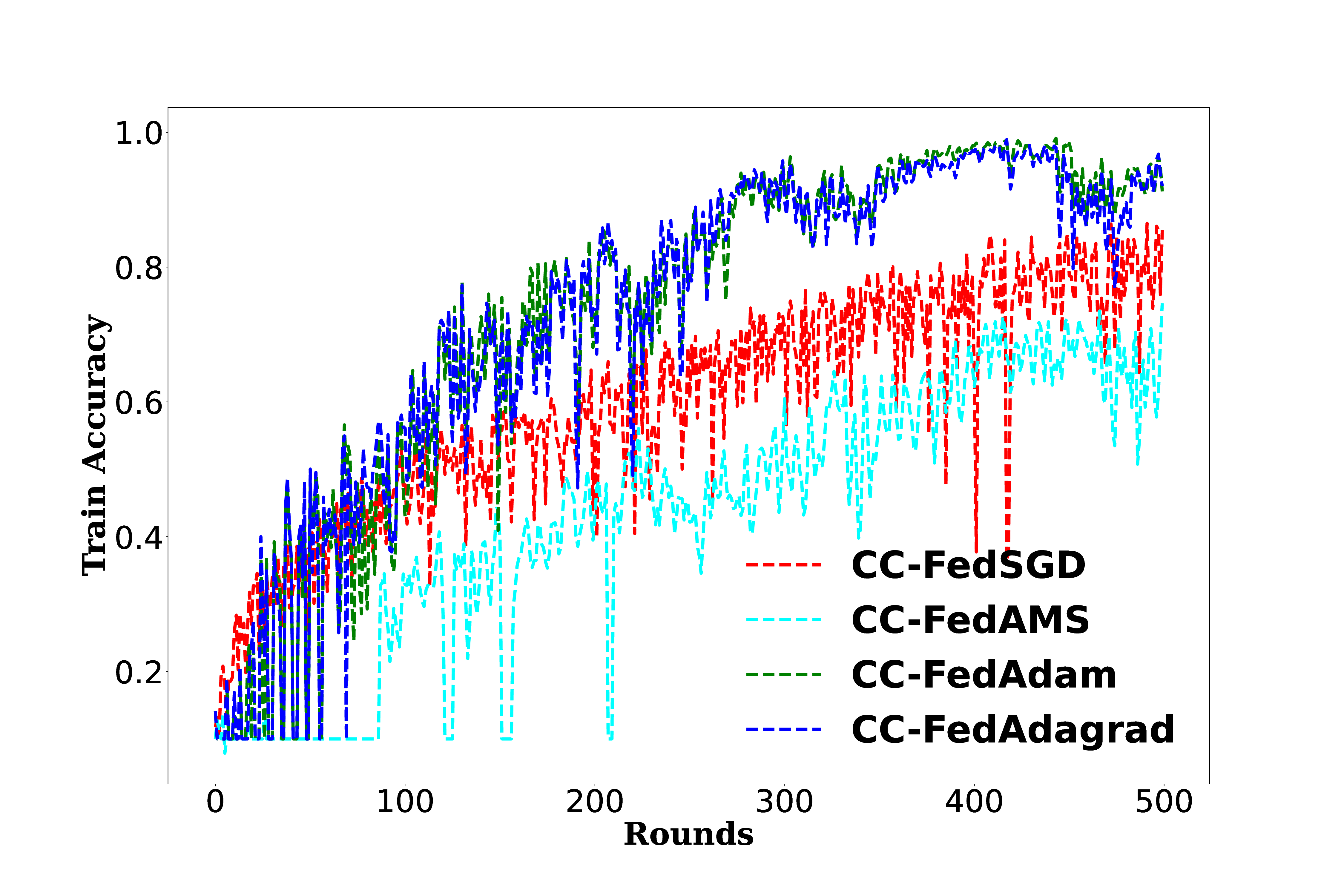}
\label{subfig:cifar10_resnet_delay_5_randomness_1_niid_05_train_appen}
}
\subfigure[Testing Curves]{
\hspace{0pt}
\includegraphics[width=.35\textwidth]{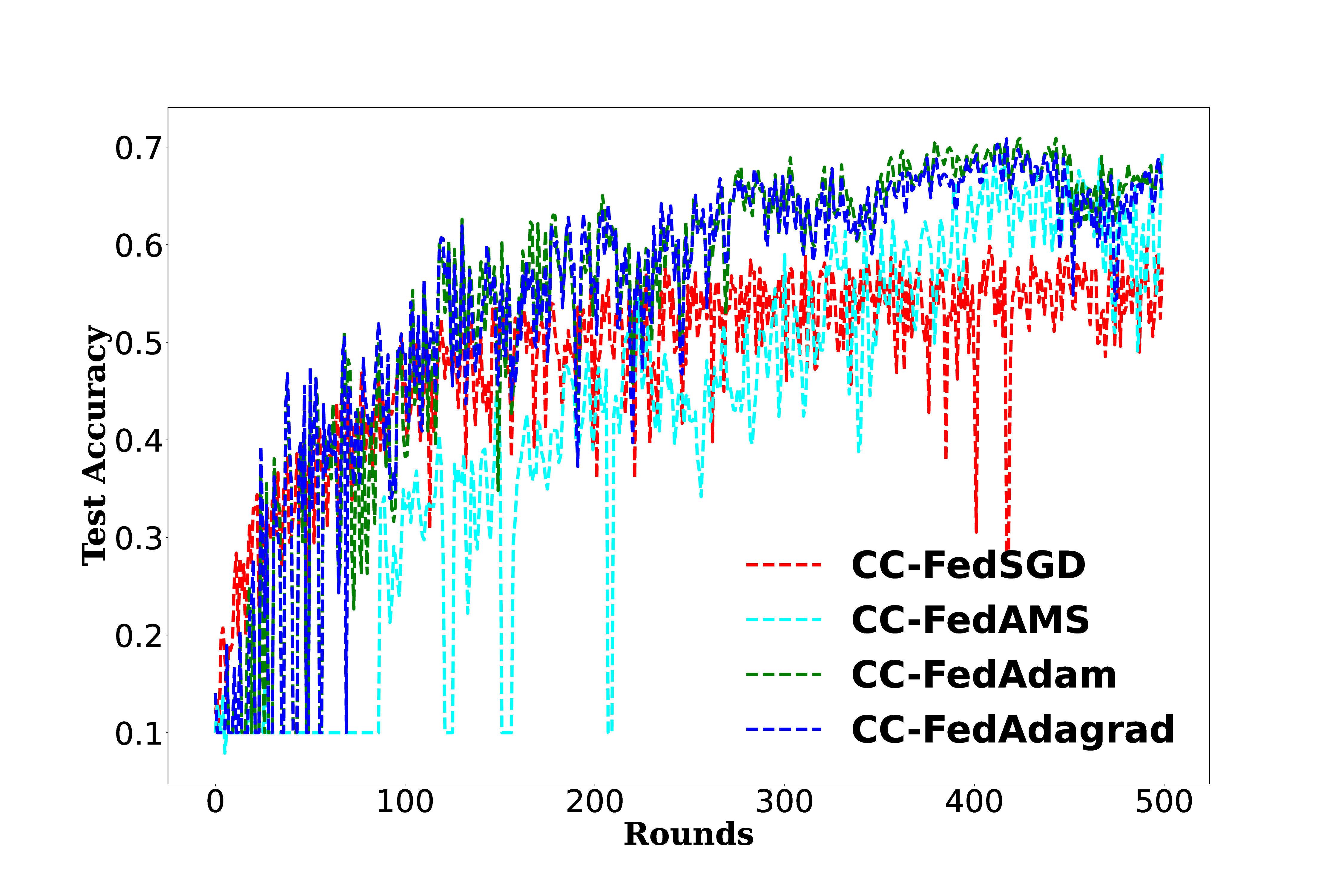}
\label{subfig:cifar10_resnet_delay_5_randomness_1_niid_05_test_appen}
}
\vspace*{-6pt}
\caption{Training and testing curves for various CC-Federated Adaptive Optimizers (ResNet on CIFAR-10) with $R=1$.}
\label{fig:cifar10_resnet_delay_5_randomness_1_niid_05_result_appen}
\end{figure*}

Figure \ref{fig:cifar10_resnet_adaptive_delay_5_niid_03_result_appen}, extra experimental results with $\alpha=0.3$.

\begin{figure*}[h]
\vspace*{-12pt}
\centering
\subfigure[Training Curves]{
\hspace{0pt}
\includegraphics[width=.35\textwidth]{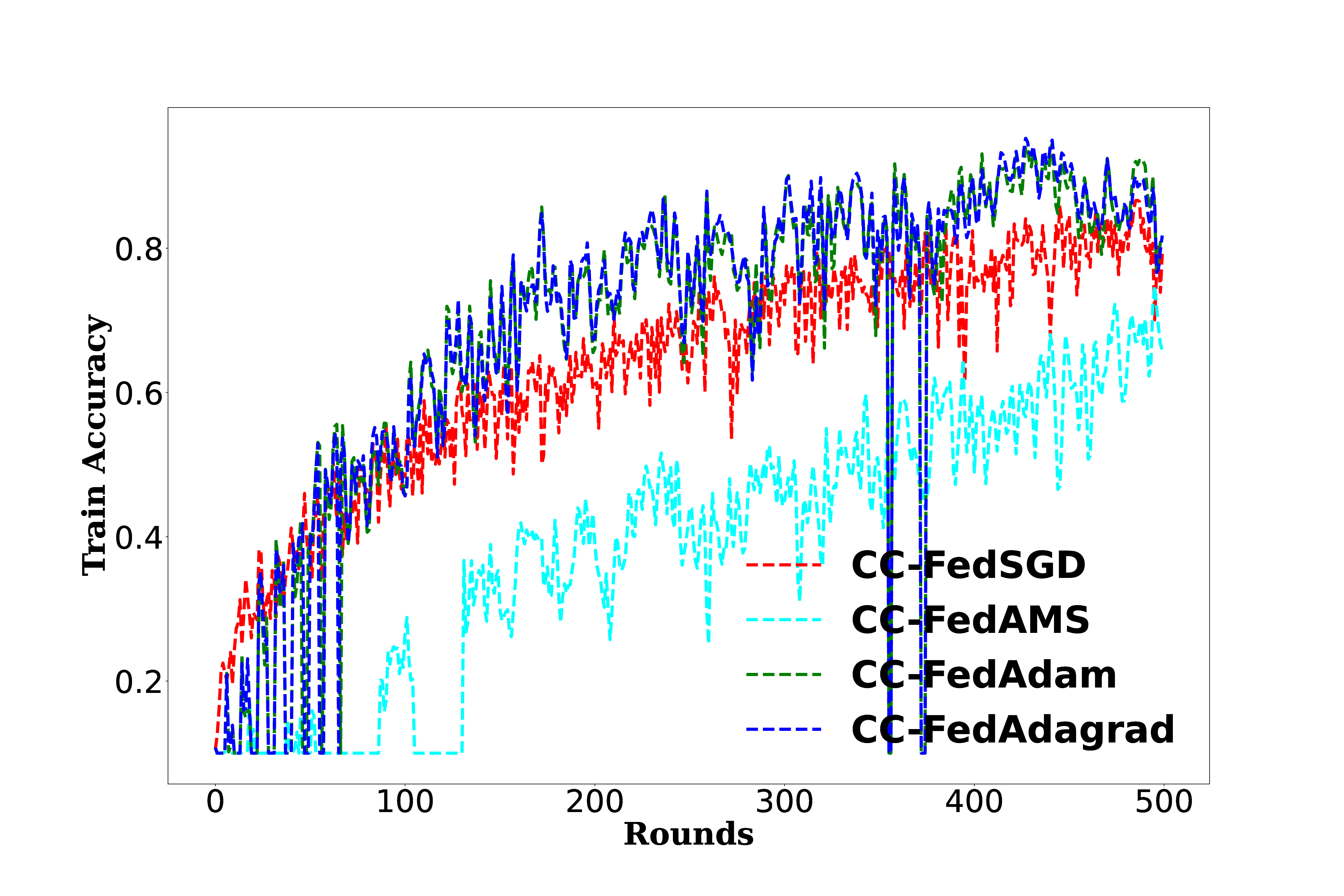}
\label{subfig:cifar10_resnet_adaptive_delay_5_niid_03_train_appen}
}
\subfigure[Testing Curves]{
\hspace{0pt}
\includegraphics[width=.35\textwidth]{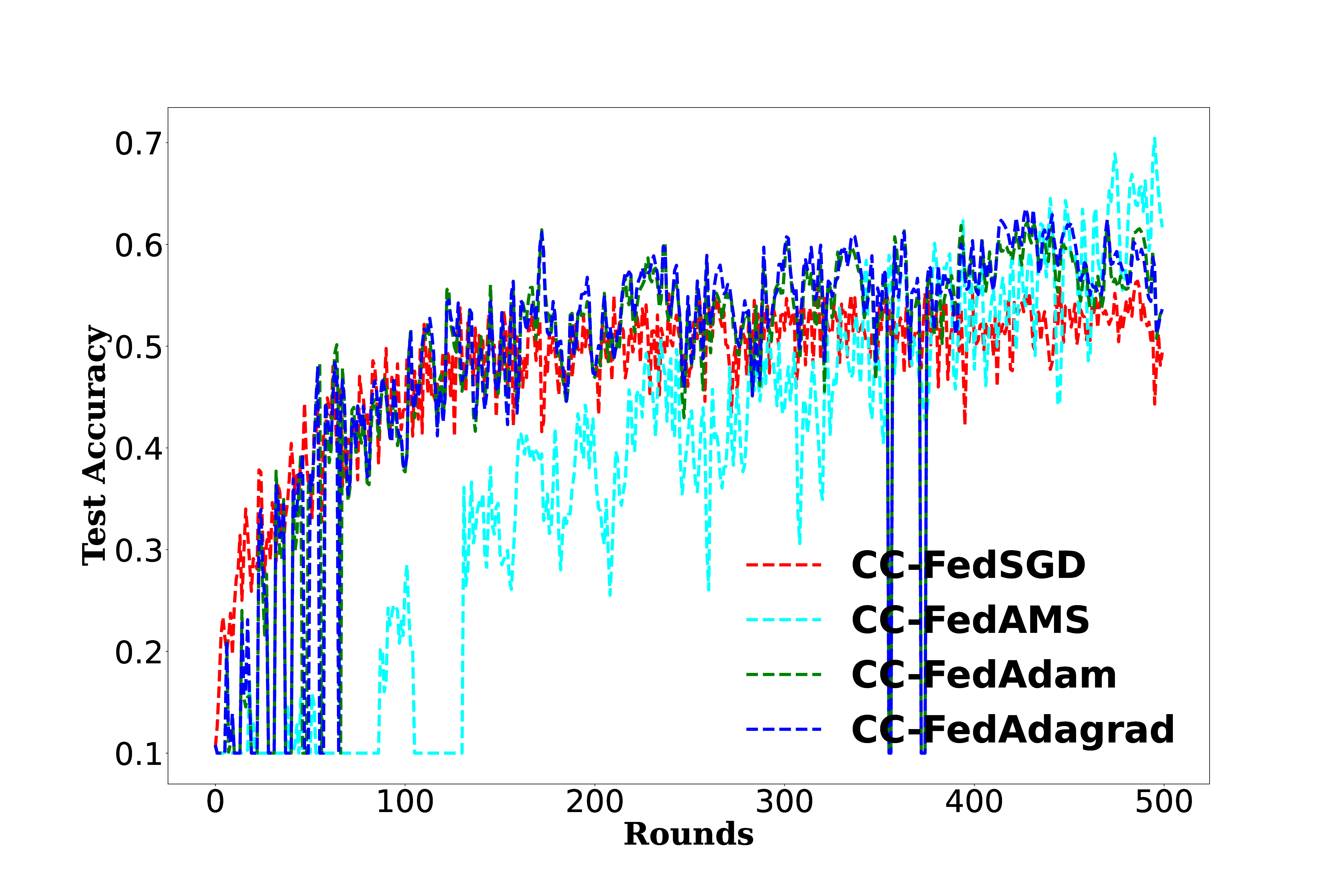}
\label{subfig:cifar10_resnet_adaptive_delay_5_niid_03_test_appen}
}
\vspace*{-6pt}
\caption{Training and testing curves for various CC-Federated Adaptive Optimizers (ResNet on CIFAR-10) with $\alpha=0.3$.}
\label{fig:cifar10_resnet_adaptive_delay_5_niid_03_result_appen}
\end{figure*}

\end{document}